\documentclass[twoside,11pt]{article}

\usepackage{include/jmlr2e}

\usepackage{microtype}
\usepackage{graphicx}
\usepackage{amsfonts}
\usepackage{amssymb}
\usepackage{amsmath}
\usepackage{mathtools}
\usepackage{subfigure}
\usepackage{bm}
\usepackage{natbib} %
\usepackage{booktabs} %
\usepackage{multirow,array}
\usepackage{hyperref}
\usepackage{comment}
\usepackage{algorithmic}
\usepackage{algorithm}
\usepackage{xcolor}
\definecolor{myblue}{HTML}{0000B5}
\definecolor{crimson}{HTML}{B30000}
\newcommand{\bb}[1]{\textcolor{myblue}{#1}}
\newcommand{\cc}[1]{\textcolor{crimson}{#1}}

\DeclareMathOperator*{\argmax}{argmax}

\renewcommand{\cite}{\citep}

\usepackage{lastpage}
\jmlrheading{21}{2020}{1-\pageref{LastPage}}{1/20; Revised 8/20}{8/20}{20-081}{Tabish Rashid, Mikayel Samvelyan, Christian Schroeder de Witt, Gregory Farquhar, Jakob Foerster, and Shimon Whiteson}

\ShortHeadings{Monotonic Value Function Factorisation for Deep Multi-Agent Reinforcement Learning}{Rashid, Samvelyan, Schroeder de Witt, Farquhar, Foerster, Whiteson}
\firstpageno{1}

\begin{document}

\title{Monotonic Value Function Factorisation for Deep Multi-Agent Reinforcement Learning}

\author{\name Tabish Rashid$^*$ \email tabish.rashid@cs.ox.ac.uk \\
       \addr University of Oxford
       \AND
       \name Mikayel Samvelyan$^*$ \email mikayel@samvelyan.com \\
       \addr Russian-Armenian University
       \AND
       \name Christian Schroeder de Witt \email cs@robots.ox.ac.uk \\
       \addr University of Oxford
       \AND
       \name Gregory Farquhar \email gregory.farquhar@cs.ox.ac.uk \\
       \addr University of Oxford
        \AND
       \name Jakob Foerster \email jnf@fb.com\\
       \addr Facebook AI Research
        \AND
       \name Shimon Whiteson \email shimon.whiteson@cs.ox.ac.uk \\
       \addr University of Oxford
   }

\newcommand{\customfootnotetext}[2]{{%
		\renewcommand{\thefootnote}{#1}%
		\footnotetext[0]{#2}}}%
	
\customfootnotetext{$^*$}{Equal contribution.}

\editor{George Konidaris}

\maketitle
\begin{abstract}
\label{sec:abstract}

In many real-world settings, a team of agents must coordinate its behaviour  
while acting in a decentralised fashion. At the same time, it is often possible to 
train the agents in a centralised fashion
where global state information is available and communication constraints are lifted. 
Learning joint action-values conditioned on extra state information is 
an attractive way to exploit centralised learning, but the best strategy for 
then extracting decentralised policies is unclear.
Our solution is QMIX, a novel value-based method that can train decentralised policies in a centralised end-to-end fashion. 
QMIX employs a mixing network that estimates joint action-values as a monotonic combination of per-agent values.
We structurally enforce that the joint-action value is monotonic in the 
per-agent values, through the use of non-negative weights in the mixing network, which 
guarantees consistency between the 
centralised and decentralised policies.
To evaluate the performance of QMIX, we propose the StarCraft Multi-Agent Challenge (SMAC) as a new benchmark for deep multi-agent reinforcement learning.
We evaluate QMIX on a challenging set of SMAC scenarios and show that it significantly outperforms existing multi-agent reinforcement learning methods.

\end{abstract}


\section{Introduction}
\label{sec:intro}

Reinforcement learning (RL) holds considerable promise to help address a variety of cooperative multi-agent problems, such as coordination of robot swarms \cite{huttenrauch_guided_2017} and autonomous cars \cite{cao_overview_2012}. 

In many such settings, partial observability and/or communication constraints necessitate the learning of \textit{decentralised policies}, which condition only on the local action-observation history of each agent. Decentralised policies also naturally attenuate the problem that joint action spaces grow exponentially with the number of agents, often rendering the application of traditional single-agent RL methods impractical.

Fortunately, decentralised policies can often be learned in a centralised fashion in a simulated or laboratory setting. This often grants access to additional state information, otherwise hidden from agents, and removes inter-agent communication constraints. 
The paradigm of \textit{centralised training with decentralised execution} \cite{oliehoek_optimal_2008,kraemer_multi-agent_2016} has recently attracted attention in the RL community \cite{jorge_learning_2016,foerster_counterfactual_2017}. 
However, many challenges surrounding how to best exploit centralised training remain open.

One of these challenges is how to represent and use the action-value function 
that many RL methods learn.  On the one hand, properly capturing the effects of 
the agents' actions requires a centralised action-value function $Q_{tot}$ that 
conditions on the global state and the joint action.  On the other hand, such a 
function is difficult to learn when there are many agents and, even if it can 
be learned, offers no obvious way to extract decentralised policies that allow 
each agent to select only an individual action based on an individual 
observation.

The simplest option is to forgo a centralised action-value function and let each agent $a$ learn an individual action-value function $Q_a$ independently, as in \emph{independent Q-learning} (IQL) \cite{tan_multi-agent_1993}.  However, this approach cannot explicitly represent interactions between the agents and may not converge, as each agent's learning is confounded by the learning and exploration of others.

At the other extreme, we can learn a fully centralised action-value function $Q_{tot}$ and then use it to guide the optimisation of decentralised policies in an actor-critic framework, an approach taken by \emph{counterfactual multi-agent} (COMA) policy gradients \cite{foerster_counterfactual_2017}, as well as work by \citet{gupta_cooperative_2017}. However, this requires on-policy learning, which can be sample-inefficient, and training the fully centralised critic becomes impractical when there are more than a handful of agents.

In between these two extremes, we can learn a centralised but factored $Q_{tot}$, an approach taken by \emph{value decomposition networks} (VDN) \cite{sunehag_value-decomposition_2017}. By representing $Q_{tot}$ as a sum of individual value functions $Q_a$ that condition only on individual observations and actions, a decentralised policy arises simply from each agent selecting actions greedily with respect to its $Q_a$. However, VDN severely limits the complexity of centralised action-value functions that can be represented and ignores any extra state information available during training.

In this paper, we propose a new approach called QMIX which, like VDN, lies between the extremes of IQL and COMA, but can represent a much richer class of action-value functions. Key to our method is the insight that the full factorisation of VDN is not necessary to extract decentralised policies.  Instead, we only need to ensure that a global $\argmax$ performed on $Q_{tot}$ yields the same result as a set of individual $\argmax$ operations performed on each $Q_a$.  To this end, it suffices to enforce a monotonicity constraint on the relationship between $Q_{tot}$ and each $Q_a$:
\begin{equation}
\label{eq:monotonicity_constraint}
\frac{\partial Q_{tot}}{\partial Q_a}  \geq 0,~ \forall a.
\end{equation}

QMIX consists of \textit{agent networks} representing each $Q_a$, and a 
\emph{mixing network} that combines them into $Q_{tot}$, not as a simple sum as 
in VDN, but in a complex nonlinear way that ensures consistency between the 
centralised and decentralised policies.
At the same time, it enforces the 
constraint of \eqref{eq:monotonicity_constraint} by restricting the mixing 
network to have positive weights.
We use \emph{hypernetworks} \cite{ha_hypernetworks_2016} to condition the weights of the mixing network on the state, which is observed only during training.
As a result, QMIX can represent complex centralised action-value 
functions with a factored representation that scales well in the number of 
agents and allows decentralised policies to be easily extracted via inexpensive 
individual argmax operations.

To evaluate QMIX, as well as the growing number of other algorithms recently proposed for multi-agent RL \cite{foerster_counterfactual_2017, sunehag_value-decomposition_2017}, we introduce the StarCraft Multi-Agent Challenge (SMAC)\footnote{Code is available at \url{https://github.com/oxwhirl/smac}.}.  In single-agent RL, standard environments such as the Arcade Learning Environment \cite{bellemare13arcade} and MuJoCo \cite{Plappert2019multigoal} have facilitated rapid progress.  While some multi-agent testbeds have emerged, such as Poker \cite{HeinrichS16}, Pong \cite{tampuu_multiagent_2015}, Keepaway Soccer \cite{stone2005keepaway}, or simple gridworld-like environments \cite{lowe_multi-agent_2017, leibo_multi-agent_2017, yang2018mean, zheng2017magent}, there are currently no challenging standard testbeds for centralised training with decentralised execution with the exception of the recently introduced Hanabi Challenge \citep{bard2020Hanabi} which focusses on a setting with less than 5 agents.

\begin{figure}[t!]
	\centering
	\subfigure[3 Stalkers vs 5 Zealots]{
		\includegraphics[width=0.4\columnwidth]{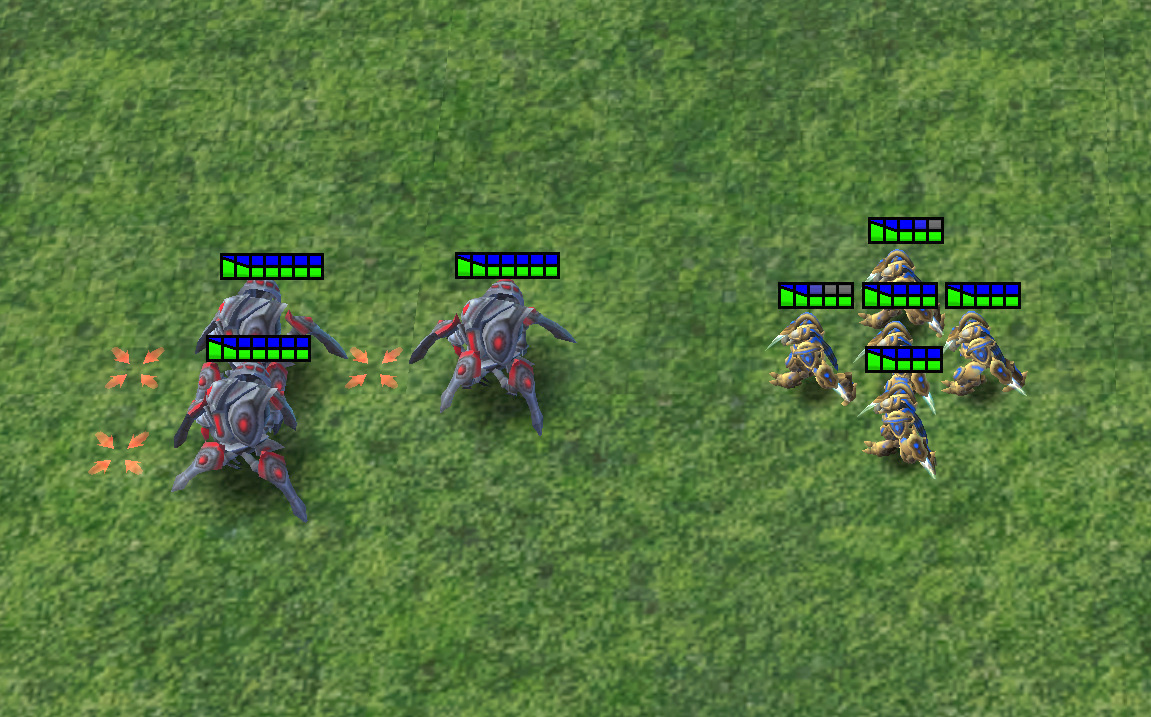}}
	\subfigure[2 Collosi vs 64 Zerglings]{
		\includegraphics[width=0.4\columnwidth]{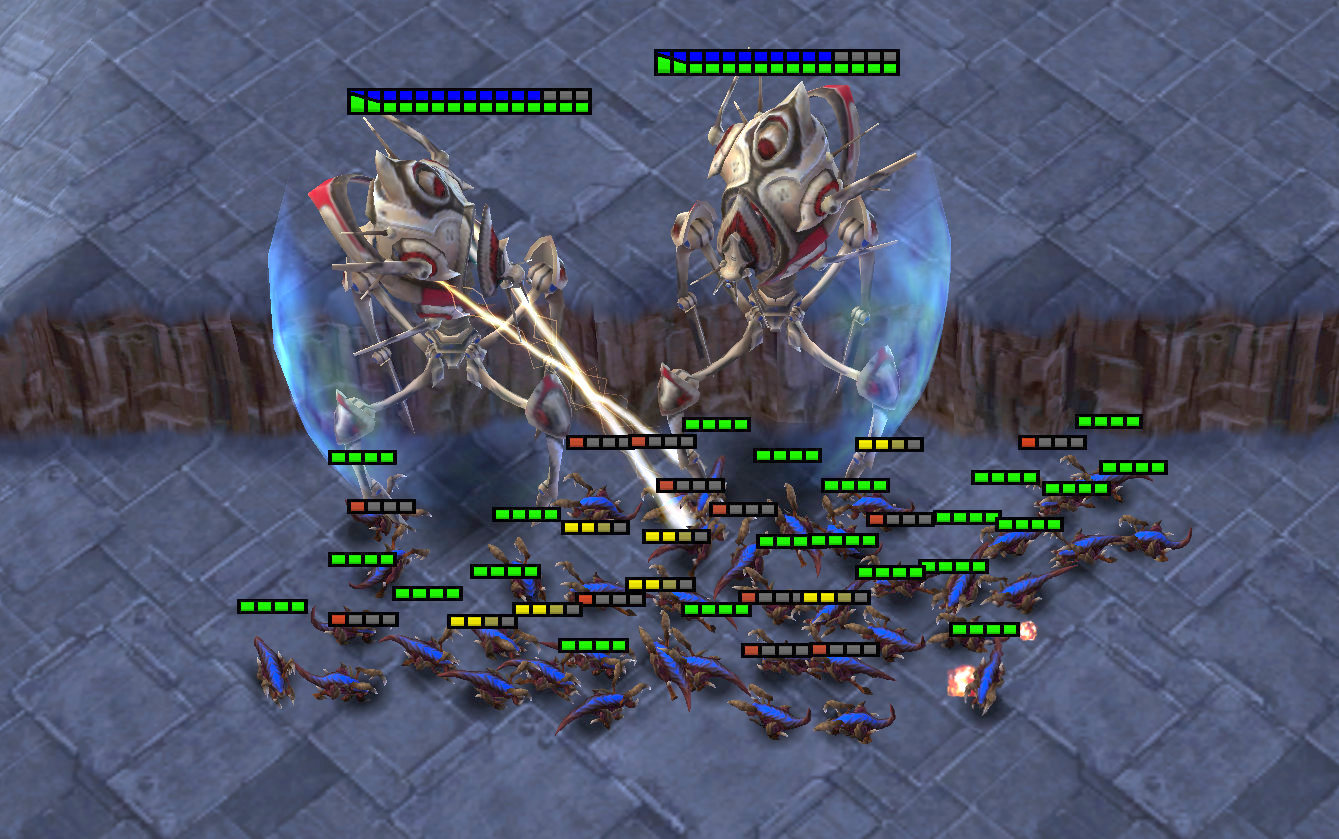}}
	\caption{\textit{Decentralised unit micromanagement} in StarCraft II, where each learning agent controls an individual unit. The goal is to coordinate behaviour across agents to defeat all enemy units.}
	\label{fig:starcraft_screenshots}
\end{figure}

SMAC fills this gap. It is built on the popular real-time strategy game StarCraft II and makes use of the SC2LE environment \cite{vinyals_starcraft_2017}.
Instead of tackling the full game of StarCraft with centralised control, it focuses on decentralised micromanagement challenges (Figure~\ref{fig:starcraft_screenshots}).
In these challenges, each of our units is controlled by an independent, learning agent that has to act based only on local observations, while the opponent's units are controlled by the hand-coded built-in StarCraft II AI.
SMAC offers a diverse set of scenarios that challenge algorithms to handle high-dimensional inputs and partial observability, and to learn coordinated 
behaviour even when restricted to fully decentralised execution.
In contrast to the diverse set of scenarios included in SMAC, the Hanabi Challenge focusses on a single task involving between 2-5 agents designed to test agents' ability to reason about the actions of the other agents (an ability in humans referred to as \textit{theory of mind} \citep{rabinowitz2018machine}).

To further facilitate research in this field, we also open-source PyMARL,
a learning framework that can serve as a starting point for other researchers and includes implementations of several key multi-agent RL algorithms. PyMARL is modular, extensible, built on PyTorch, and serves as a template for dealing with some of the unique challenges of deep multi-agent RL in practice. 
We also offer a set of guidelines for best practices in evaluations using our benchmark, including the reporting of standardised performance metrics, sample efficiency, and computational requirements (see Section~\ref{sec:setting}).

Our experiments on SMAC show that QMIX outperforms IQL, VDN, and COMA, both in terms of absolute performance and learning speed.
In particular, our method shows considerable performance gains on the harder tasks in SMAC, and tasks with heterogeneous agents. 
Moreover, our analysis and ablations show both the necessity of 
conditioning on the state information and a flexible multi-layer network for mixing of agent 
$Q$-values in order to achieve consistent performance across tasks.

Since its introduction by \citet{rashid2018qmix}, QMIX has become an important value-based algorithm in discrete action environments. 
It has inspired a number of extensions and follow-up work (Section \ref{sec:rw_qmix}) and is a prominent point of comparison in any work that utilises SMAC (Section \ref{sec:rw_smac}). 
Recently it is 
becoming increasingly common (and arguably necessary) to evaluate deep multi-agent RL algorithms on non-gridworld environments.
This has been driven in part by the introduction of PyMARL and SMAC, which provide an open-source codebase and a standardised testbed for evaluating and comparing deep multi-agent RL algorithms.\footnote{
	This work was originally published as a conference paper \citep{rashid2018qmix}.
	The main additions in this article are:
	\begin{itemize}
		\item Introducing SMAC, a benchmark for the cooperative multi-agent RL setting of centralised training and decentralised execution (Section \ref{sec:setting}).
		\item Releasing PyMARL, a framework for running and developing multi-agent RL algorithms (Section \ref{sec:pymarl}).
		\item Experimental results comparing IQL, VDN, COMA, and QTRAN on the SMAC benchmark (Section \ref{sec:results}).
		\item Further analysis and ablation experiments to investigate why QMIX outperforms VDN (Sections \ref{sec:rnd_matrix} and \ref{sec:analysis}).
		\item Comprehensive literature review of classical as well as novel cooperative multi-agent RL approaches (Section \ref{sec:related}).
	\end{itemize}
}

\section{Related Work}
\label{sec:related}

Recent work in multi-agent RL has started moving from 
tabular methods \cite{yang_multiagent_2004, busoniu_comprehensive_2008} to deep learning methods that can tackle
high-dimensional state and action spaces \cite{tampuu_multiagent_2015,foerster_counterfactual_2017,peng_multiagent_2017, de_witt_multi-agent_2018}.
In this paper, we 
focus on the fully-cooperative setting in which all agents must maximise a joint reward signal in the paradigm of centralised training and decentralised execution.  

\subsection{Independent Learners}
A natural approach to producing decentraliseable agents in a multi-agent system 
is to directly learn decentralised value functions or policies. 
\emph{Independent Q-learning} \cite{tan_multi-agent_1993} trains independent
action-value functions for each agent using $Q$-learning \cite{watkins_learning_1989}. \citet{tampuu_multiagent_2015} extend this approach to 
deep neural networks using DQN \cite{mnih_human-level_2015}.
While trivially achieving decentralisation, these approaches are prone to instability arising from the non-stationarity of the environment induced by simultaneously learning and exploring agents. 
\citet{foerster_stabilising_2017} addresses the issue of non-stationarity when using an experience replay with independent learners to some extent.
\citet{lauer2000algorithm} ignore updates which decrease $Q$-value estimates in order to not prematurely underestimate an action's $Q$-value due to the exploratory actions of other agents. In a tabular setting, they prove this converges to the optimal policy in deterministic environments, provided there is a unique optimal joint-action.
\citet{matignon2007hysteretic} introduce \textit{Hysteretic $Q$-learning} which instead uses a smaller learning rate for decreasing $Q$-value estimates, which is slightly more robust to stochasticity and the presence of multiple optimal joint actions. 
\citet{omidshafiei_deep_2017} utilises Hysteretic $Q$-learning in a multi-task deep RL setting. 
\citet{panait2008theoretical} introduce \textit{Leniency} which ignores updates that decrease $Q$-value estimates with a probability that is decreasing during training. 
\citet{wei2016lenient} show that Leniency outperforms Hysteretic $Q$-learning in cooperative stochastic games, and \citet{palmer2018lenient} extend Leniency to the deep RL setting and show benefits over Hysteretic and fully Independent $Q$-learners on deterministic and stochastic gridworlds with two agents. 
\citet{palmer_negative_2019} maintain
a learned interval whose lower bound is approximately the minimum cumulative reward received when all agents are coordinating. They use this interval when decreasing $Q$-value estimates to distinguish between miscoordination between the agents and the stochasticity of the environment, which Hysteretic and Lenient learners do not distinguish between. 
\citet{lu2019decentralised} utilise Distributional RL \citep{bellemare2017distributional} in combination with Hysteretic $Q$-learning in order to also distinguish between miscoordination and stochasticity, and argue their approach is more stable and robust to hyperparameters than the above methods. 
All of these approaches do not utilise extra state information available during a centralised training regime, nor do they attempt to learn joint action-value functions.

\subsection{Centralised Execution}
In settings where decentralised execution is not mandatory, 
the centralised learning of joint action-value function naturally handles the
coordination problems and avoids the non-stationarity problem. Nonetheless, a centralised action-value function is impractical to scale since the joint action space grows exponentially in the number of agents.
Classical approaches to scalable centralised learning include 
\textit{coordination graphs} \cite{guestrin_multiagent_2002}, which exploit conditional independencies between agents by decomposing a global reward function 
into a sum of agent-local terms.
This can significantly reduce the computation required in order to find the maximum joint action (depending on the exact structure of the coordination graph supplied) through message passing or variable elimination algorithms.
However, specifying an appropriate coordination graph (i.e. not fully-connected) can require significant domain knowledge.
\textit{Sparse cooperative Q-learning} \cite{kok_collaborative_2006} is a tabular $Q$-learning 
algorithm that learns to coordinate the actions of a group of cooperative 
agents only in the states in which such coordination is necessary, encoding
those dependencies in a coordination graph. 
Both methods require the dependencies between agents to be pre-supplied, whereas we do not require such prior knowledge. 
\citet{castellini2019representational} investigate the representational capacity of coordination graphs in a deep RL setting using one-step matrix games.
\citet{bohmer_deep_2019} also extend coordination graphs to the deep RL setting, making use of parameter sharing and limiting the graphs to contain only pairwise edges, in order to consider more complex scenarios. 
\citet{chen_factorized_2018} similarly factor the joint $Q$-function into pairwise interaction terms, performing maximisation using coordinate ascent instead of message-passing.
QMIX (and VDN) correspond to the case of a degenerate fully disconnected coordination graph,  thus enabling fully decentralised execution.

DIAL
\cite{foerster_learning_2016} utilises an end-to-end differentiable architecture that allows for a learned inter-agent communication to emerge via backpropagation.
CommNet 
\cite{sukhbaatar_learning_2016} use a centralised network architecture to exchange information between agents. 
BicNet 
\cite{peng_multiagent_2017} utilise bidirectional RNNs for inter-agent communication in an actor-critic setting and  
additionally requires estimating individual agent rewards.
MAGNet \citep{malysheva_deep_2018} allow agents to centrally learn a shared graph structure that encodes the relevance of individual environment entities to each agent. Based on this graph, agents can communicate and coordinate their actions during centralised execution.
\citet{zhao_learning_2019} show that the performance of decentralised policies can be improved upon through a centralised communication scheme that uses a central information-filtering neural network in between timesteps.
Clustered Deep Q-Networks (CDQN) \citep{pageaud_multiagent_2019} use a hierarchical approach in order to scale learning to more agents and larger joint action spaces. Low-level agents are clustered into groups, each of which is managed by a higher-level agent within a centralised execution setting. All low-level agents within a cluster execute the same action selected by a vote, which alleviates some of the non-stationarity of the independent learning setting.
In contrast to these approaches, QMIX does not require any form of inter-agent communication during decentralised execution.

\subsection{Centralised Training and Decentralised Execution}
A number of methods have developed hybrid approaches that exploit the centralised training opportunity for training fully decentralised policies. 
\citet{gupta_cooperative_2017} present a centralised actor-critic algorithm with per-agent critics, which scales easily with the number of agents.
Similarly, \citet{lowe_multi-agent_2017} present MADDPG, which learns a 
centralised critic for each agent and apply this to competitive games with 
continuous action spaces.
In contrast to \citet{gupta_cooperative_2017}, MADDPG takes advantage of the centralised training paradigm by incorporating information from the other agents into the critics, at the expense of scalability due to the increased input size.
COMA \citep{foerster_counterfactual_2017} instead uses a single centralised critic to train 
decentralised actors in the fully-cooperative setting, estimating a counterfactual advantage function for each 
agent in order to address multi-agent credit assignment. 
\citet{iqbal_actor-attention-critic_2018} devise an actor-critic algorithm with per-agent critics that share an attention mechanism in order to improve scalability compared to a single centralised critic.
Their approach allows for the agents to have differing action spaces as well allowing for individual rewards. The authors do not compare against QMIX and do not evaluate on environments of similar complexity as StarCraft. 
\citet{de_witt_multi-agent_2018} construct a multi-agent actor-critic algorithm with a hierarchical actor policy that is able to use common knowledge between agents for coordination. 
Learning Individual Intrinsic Reward (LIIR) \citep{du_liir:_2019} learns an intrinsic reward function for each agent to supplement the team reward. The intrinsic reward function is trained such that it maximises the team reward function. As a bi-level optimisation process, LIIR is significantly less computationally efficient than QMIX and has not been demonstrated to scale to large numbers of agents.
These approaches are all actor-critic variants which use the policy-gradient theorem, and thus are prone to get stuck sub-optimal local minima. Additionally, all but MADDPG are on-policy algorithms which can have poor sample efficiency compared to off-policy algorithms such as QMIX.

\citet{lin2019cesma} train a centralised action-value function $Q$ which has access to observations of all agent. They then train the decentralised agents via supervised learning to mimic the actions of the centralised policy. Distilling a centralised policy into decentralised agent policies can be problematic since the centralised policy conditions on more information than any of the agents' policies. QMIX instead learns a factored joint $Q$-function that can be easily decentralised due to its architecture.  
\citet{sunehag_value-decomposition_2017} propose \emph{value decomposition 
	networks} (VDN), which allow for centralised value function learning with 
decentralised execution. Their algorithm decomposes a central 
action-value function into a sum of individual agent terms. 
VDN does not make 
use of additional state information during training and can represent only a
limited class of centralised action-value functions.
\citet{son_qtran:_2019} introduce QTRAN which learns a centralised, unrestricted, joint $Q$-function as well as a VDN-factored joint $Q$-function that is decentralisable. 
Since the unrestricted $Q$-value cannot be maximised efficiently, the VDN-factored $Q$ is used to produce the (approximate) maximum joint-action. 
They show that solving a linear optimisation problem involving all joint actions, in which the VDN-factored $Q$-value matches the unrestricted $Q$-value for the approximated maximum joint action and over-estimates for every other action, results in decentralisable policies with the correct argmax.
This allows QTRAN to represent a much larger class of joint $Q$-function than QMIX, as well as produce decentralisable policies for them.
However, exactly solving the linear optimisation problem is prohibitively expensive. Thus, the authors instead optimise a soft approximation using $L_2$ penalties via stochastic gradient descent.
In practice, optimising this loss is difficult which results in poor performance for QTRAN on complex environments such as SMAC.

\subsection{Extensions of QMIX}
\label{sec:rw_qmix}

\citet{mahajan_maven:_2019} show that the representational constraints of QMIX can prohibit it from learning an optimal policy. 
Importantly, they show that this is exacerbated by increased $\epsilon$-greedy exploration.
They introduce MAVEN, which conditions QMIX agents on a shared latent space whose value is chosen at the beginning of an episode by a hierarchical policy. Note that this requires access to the initial state and communication at the first timestep.
A mutual information loss is added to encourage diversity of trajectories across the shared latent space, which allows for relatively-greedy agents (once conditioned on the latent variable) to still achieve committed exploration during training similar to Bootstrapped DQN \citep{osband2016deep}.
This results in significant performance gains on some scenarios in SMAC \citep{samvelyan2019starcraft}.

Action Semantics Networks (ASN) \cite{wang_action_2019} extends QMIX with a novel agent network architecture that separates each agent's $Q$-value estimation of actions that influence other agents from those actions that do not.
The $Q$-values for actions which influence another agent are computed using the relevant parts of the current agent's observation (e.g. the relative position of another agent).
Thus, ASN requires 
intimate a priori knowledge about an environment's action semantics and the structure of the agent's observations, whereas QMIX does not.

\citet{wang_learning_2019} extend QMIX to a setting which allows for communication between agents.
Their approach allows agents to condition their $Q$-values on messages from other agents, which take the form of a real-valued vector. 
An entropy loss and a mutual information loss incentivise the learning of succinct and expressive messages.
\citet{zhang2019efficient} also allow for each agent's $Q$-values to condition on vector-valued messages from other agents. However, they only utilise the messages from other agents if their gap between the best and second-best action is small enough (interpreted as uncertainty in the best choice of action). Additionally, they minimise the variance of the messages across actions to limit noisy or uninformative messages.

SMIX($\lambda$) \cite{yao_smixlambda:_2019} replaces the 1-step $Q$-learning target of QMIX with a SARSA($\lambda$) target.
They incorrectly claim that the $Q$-learning update rule is responsible for the representational limitations of QMIX, instead of the non-negative weights in the mixing network which enforces monotonicity.
They also incorrectly claim that their method can represent a larger class of joint action-value functions than QMIX. 
Since they also use non-negative weights in their mixing network, SMIX($\lambda$) can represent \textbf{exactly} the same class of value functions as QMIX. 

\citet{liu2019value} investigate transfer learning algorithms in a multi-agent setting using QMIX as a base. \citet{yang2019hierarchical} use DIAYN \citep{eysenbach2018diversity} to learn lower-level skills that a higher level QMIX agent uses. \citet{fu2019deep} extend QMIX to allow for action spaces consisting of both discrete and continuous actions.

The applicability of QMIX in real-world application domains has also been considered such as robot swarm coordination \cite{huttenrauch_deep_2018} and stratospheric geoengineering \cite{de_witt_stratospheric_2019}. 

QMIX has been extended from discrete to continuous action spaces \cite[COMIX]{de_witt_deep_2020}. COMIX uses the cross-entropy method to approximate the otherwise intractable greedy per-agent action maximisation step. \citet{de_witt_deep_2020} also present experiments with FacMADDPG, i.e., MADDPG with a factored critic, which show that QMIX-like network factorisations also perform well in actor-critic settings.
COMIX has been found to outperform previous state-of-the-art MADDPG in continuous robotic control suite Multi-Agent Mujoco, thereby illustrating its versatility beyond SMAC.

\subsection{PyMARL and SMAC}
\label{sec:rw_smac}
A number of papers have established unit micromanagement in StarCraft as a benchmark for deep multi-agent RL. 
\citet{usunier_episodic_2016} present an algorithm using a centralised \textit{greedy MDP} and first-order optimisation which they evaluate on Starcraft: BroodWar \citep{synnaeve_torchcraft_2016}. \citet{peng_multiagent_2017} also evaluate their methods on StarCraft. However, neither requires decentralised execution. 
Similar to our setup in SMAC is the work of \citet{foerster_stabilising_2017}, who evaluate replay stabilisation methods for IQL on combat scenarios with up to five agents. \citet{foerster_counterfactual_2017} also uses this setting.

In this paper, we construct unit micromanagement tasks using the \textit{StarCraft II Learning Environment} (SC2LE) \citep{vinyals_starcraft_2017} as opposed to 
StarCraft, since it is actively supported by the game developers and offers a more stable testing environment. 
The full game of StarCraft II has already been used as an RL environment \cite{vinyals_starcraft_2017}, and DeepMind's AlphaStar \cite{vinyals2019grandmaster} has recently shown an impressive level of play on a StarCraft II matchup using a centralised controller. By contrast, SMAC introduces strict decentralisation and local partial observability to turn the StarCraft II game engine into a new set of decentralised cooperative multi-agent problems. 

Since their introduction, SMAC and PyMARL have been used by a number of papers to evaluate and compare multi-agent RL methods \citep{de_witt_multi-agent_2018,wang_few_2019,du_liir:_2019,wang_learning_2019,zhang2019efficient,mahajan_maven:_2019,yao_smixlambda:_2019,bohmer_deep_2019,wang_action_2019}.

\subsection{Hypernetworks and Monotonic Networks}
QMIX relies on a neural network to transform the centralised state into the weights of another neural network, in a manner reminiscent of 
\emph{hypernetworks} \citep{ha_hypernetworks_2016}. This second neural network 
is constrained to be monotonic with respect to its inputs by keeping its 
weights positive. \citet{Dugas_2009} also investigate such functional restrictions 
for neural networks.

\section{Background}
\label{sec:background}

A fully cooperative multi-agent sequential decision-making task can be described as a \emph{decentralised partially observable Markov decision process} (Dec-POMDP) \cite{oliehoek_concise_2016} consisting of a tuple $G=\left\langle S,U,P,r,Z,O,n,\gamma\right\rangle$. 
$s \in S$ describes the true state of the environment.
At each time step, each agent $a \in A \equiv \{1,...,n\}$ chooses an action $u^a\in U$, forming a joint action $\mathbf{u}\in\mathbf{U}\equiv U^n$. 
This causes a transition on the environment according to the state transition function $P(s'|s,\mathbf{u}):S\times\mathbf{U}\times S\rightarrow [0,1]$. 
All agents share the same reward function $r(s,\mathbf{u}):S\times\mathbf{U}\rightarrow\mathbb{R}$ and $\gamma\in[0,1)$ is a discount factor. 

We consider a \textit{partially observable} scenario in which each agent draws individual observations $z\in Z$ according to observation function $O(s,a):S\times A\rightarrow Z$. 
Each agent has an action-observation history $\tau^a\in T\equiv(Z\times U)^*$, on which it conditions a stochastic policy $\pi^a(u^a|\tau^a):T\times U\rightarrow [0,1]$. The joint policy $\pi$ has a joint \textit{action-value function}: $Q^\pi(s_t, \mathbf{u}_t)=\mathbb{E}_{s_{t+1:\infty},\mathbf{u}_{t+1:\infty}} \left[R_t|s_t,\mathbf{u}_t\right]$, where $R_t=\sum^{\infty}_{i=0}\gamma^ir_{t+i}$ is the \textit{discounted return}.

We operate in the framework of centralised training with decentralised 
execution, i.e. the
learning algorithm has access to all local action-observation histories 
$\boldsymbol{\tau}$ and global state $s$, but each agent's learnt policy can 
condition only on its own action-observation history $\tau^a$.

\subsection{Deep $Q$-Learning}

Deep $Q$-learning represents the action-value function with a deep neural network parameterised by $\theta$. \textit{Deep Q-networks} (DQNs) \cite{mnih_human-level_2015} use a  \textit{replay memory} to store the transition tuple $\left\langle s,u,r,s'\right\rangle$, where the state  $s'$  is observed after taking the action $u$ in state $s$ and receiving reward $r$. $\theta$ is learnt by sampling batches of $b$ transitions from the replay memory and minimising the squared \textit{TD error}:
\begin{equation}\label{eq:dqn}
\mathcal{L}(\theta)=\sum\limits_{i=1}^{b}\left[\left(y_i^{\text{DQN}}-Q(s,u;\theta)\right)^2\right],
\end{equation} 
where $y^{\text{DQN}}=r+\gamma\max_{u'} Q(s',u';\theta^-)$. $\theta^-$ are the parameters of a \textit{target network} that are periodically copied from $\theta$ and kept constant for a number of iterations.  

\subsection{Deep Recurrent $Q$-Learning}

In partially observable settings, agents can benefit from conditioning on their entire action-observation history. \citet{hausknecht_deep_2015} propose \textit{deep recurrent Q-networks} (DRQN) that make use of recurrent neural networks. Typically, gated architectures such as LSTM \cite{hochreiter_long_1997} or GRU \cite{chung_empirical_2014} are used to facilitate learning over longer timescales.

\subsection{Independent $Q$-Learning}

Perhaps the most commonly applied method in multi-agent learning is \textit{independent Q-learning} (IQL) \cite{tan_multi-agent_1993}, which decomposes a multi-agent problem into a collection of simultaneous single-agent problems that share the same environment. This approach does not address the non-stationarity introduced due to the changing policies of the learning agents, and thus, unlike $Q$-learning, has no convergence guarantees even in the limit of infinite exploration. In practice, nevertheless, IQL commonly serves as a surprisingly strong baseline even in mixed and competitive games \cite{tampuu_multiagent_2015, leibo_multi-agent_2017}.

\subsection{Value Decomposition Networks}

By contrast, \textit{value decomposition networks} (VDNs) \citep{sunehag_value-decomposition_2017} aim to learn a joint action-value function $Q_{tot}(\boldsymbol{\tau},\mathbf{u})$, where $ \bm{\tau} \in \mathbf{T} \equiv \mathcal{T}^n $ is a joint action-observation history and $ \mathbf{u} $ is a joint action.  It represents $Q_{tot}$ as a sum of individual value functions $Q_a (\tau^a, u^a;\theta^a)$, one for each agent $a$, that condition only on individual action-observation histories:
\begin{equation}\label{eq:vdn}
Q_{tot}(\boldsymbol{\tau}, \mathbf{u}) = \sum_{i=1}^n Q_i (\tau^i, u^i;\theta^i).
\end{equation}
Strictly speaking, each $Q_a$ is a \textit{utility function} \cite{guestrin_multiagent_2002} and not a value function since by itself it does not estimate an expected return.  However, for terminological simplicity we refer to both $Q_{tot}$ and $Q_a$ as value functions.

The loss function for VDN is equivalent to \eqref{eq:dqn}, where $Q$ is replaced by $Q_{tot}$.  An advantage of this representation is that a decentralised policy arises simply from each agent performing greedy action selection with respect to its $Q_a$. 

\section{QMIX}

\label{sec:methods}

In this section, we propose a new approach called QMIX which, like VDN, lies 
between the extremes of IQL and centralised $Q$-learning. However, QMIX can represent a much richer class of action-value functions than VDN. 

Key to our method is the insight that the full factorisation of VDN is not 
necessary in order to extract decentralised policies that are fully 
consistent with their centralised counterpart.
As long as the environment is not adversarial, there exists a deterministic optimal policy conditioned 
on the full action-observation history.
Hence, we only need to establish consistency between the deterministic greedy 
decentralised policies and the deterministic greedy centralised policy based 
on the optimal joint action-value function.
When the greedy decentralised policies are determined by an $\argmax$ over the 
$Q_a$, 
consistency holds if a 
global $\argmax$ performed on $Q_{tot}$ yields the same result as a set of 
individual $\argmax$ operations performed on each $Q_a$:
\begin{equation}
\label{eq:argmax_constist}
\argmax_{\mathbf{u}}Q_{tot}(\boldsymbol{\tau}, \mathbf{u}) = 
\begin{pmatrix}
\argmax_{u^1}Q_1(\tau^1, u^1)   \\
\vdots \\
\argmax_{u^n}Q_n(\tau^n, u^n) \\
\end{pmatrix}.
\end{equation}
This allows each agent $a$ to participate in a decentralised execution solely 
by choosing greedy actions with respect to its $ Q_a $.
As a side effect, if \eqref{eq:argmax_constist} is satisfied, then taking the 
$\argmax$ of $Q_{tot}$, required by off-policy learning updates, is trivially 
tractable without an exhaustive evaluation of $Q_{tot}$ for the exponentially 
many joint actions.

VDN's representation is sufficient to satisfy \eqref{eq:argmax_constist}. 
However, QMIX is based on the observation that this representation can be 
generalised to the larger family of monotonic functions that also satisfy 
\eqref{eq:argmax_constist}.
Monotonicity in this context is defined as a constraint on the relationship 
between $Q_{tot}$ and each $Q_a$:
\begin{equation}
\label{eq:deriv-constr}
\frac{\partial Q_{tot}}{\partial Q_a}  \geq 0,~ \forall a \in A,
\end{equation}
which is sufficient
to satisfy 
\eqref{eq:argmax_constist}, as the following theorem shows.
	\begin{theorem} 
		If $\forall a \in A \equiv\{1, 2, ..., n\}$, $\frac{\partial Q_{tot}}{\partial Q_a}  \geq 0$  then
		\begin{equation}
		 \nonumber
		\argmax_{\mathbf{u}}Q_{tot}(\boldsymbol{\tau}, \mathbf{u}) = 
		\begin{pmatrix}
		\argmax_{u^1}Q_1(\tau^1, u^1)   \\
		\vdots \\
		\argmax_{u^n}Q_n(\tau^n, u^n) \\
		\end{pmatrix}.
		\end{equation}
	\end{theorem}
	\begin{proof}
	Proof is provided in Appendix \ref{appendix:monotonicity}.
	\end{proof}

\begin{figure*}[t]
	\centering
	\includegraphics[width=0.5\textwidth]{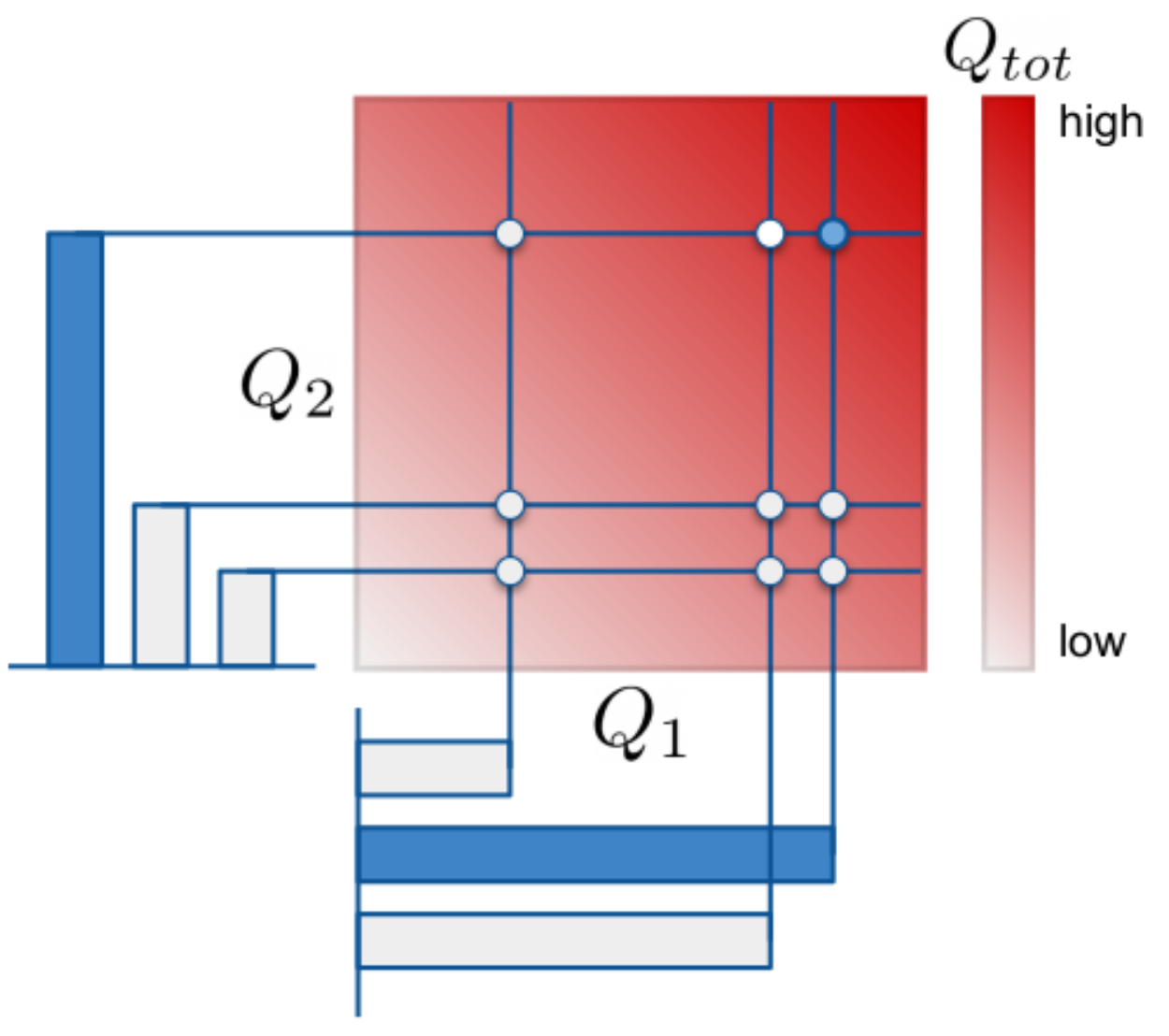}
	\caption{The discrete per-agent action-value scores $Q_a$ are fed into the monotonic 
		function $Q_{tot}(Q_1, Q_2)$. The maximum $Q_a$ for each agent is shown in blue, which corresponds to the maximum $Q_{tot}$ also shown in blue. The constraint \eqref{eq:argmax_constist} is 
		satisfied due to the monotonicity of $Q_{tot}$. 
	}
	\label{fig:monotonic}
\end{figure*}

Figure \ref{fig:monotonic} illustrates the relationship between $Q_{tot}$ and individual $Q_a$ functions, and how monotonicity leads to a decentralisable $\argmax$ in an example with two 
agents with three possible actions.
Each agent $a$ produces scores $Q_a(u_a)$ for each discrete action that it can 
take (columns with the local $\argmax$ in blue).
These are the inputs to a continuous monotonic mixing function 
$Q_{tot}(Q_1,...,Q_N)$, the output of which is represented by the monotonically 
increasing heatmap.
The intersections of the blue lines indicate the estimated $Q_{tot}$ for each 
of the discrete joint actions.
Due to the monotonicity of $Q_{tot}$, these joint-action value estimates 
maintain the ordering corresponding to each agent's $Q_a$, for that agent's 
actions when the other agents' actions remain fixed.
The result is that the global greedy joint action of $Q_{tot}$, indicated by 
the blue dot, corresponds to the set of decentralised greedy actions. 

We now describe how QMIX enforces \eqref{eq:deriv-constr} in practice, by 
representing $Q_{tot}$ using an architecture consisting of \textit{agent 
networks}, a \textit{mixing network}, and a set of \emph{hypernetworks} 
\cite{ha_hypernetworks_2016}.
Figure \ref{fig:QMIX} illustrates the overall setup.

\begin{figure*}[t]
	\centering
	\includegraphics[width=\textwidth]{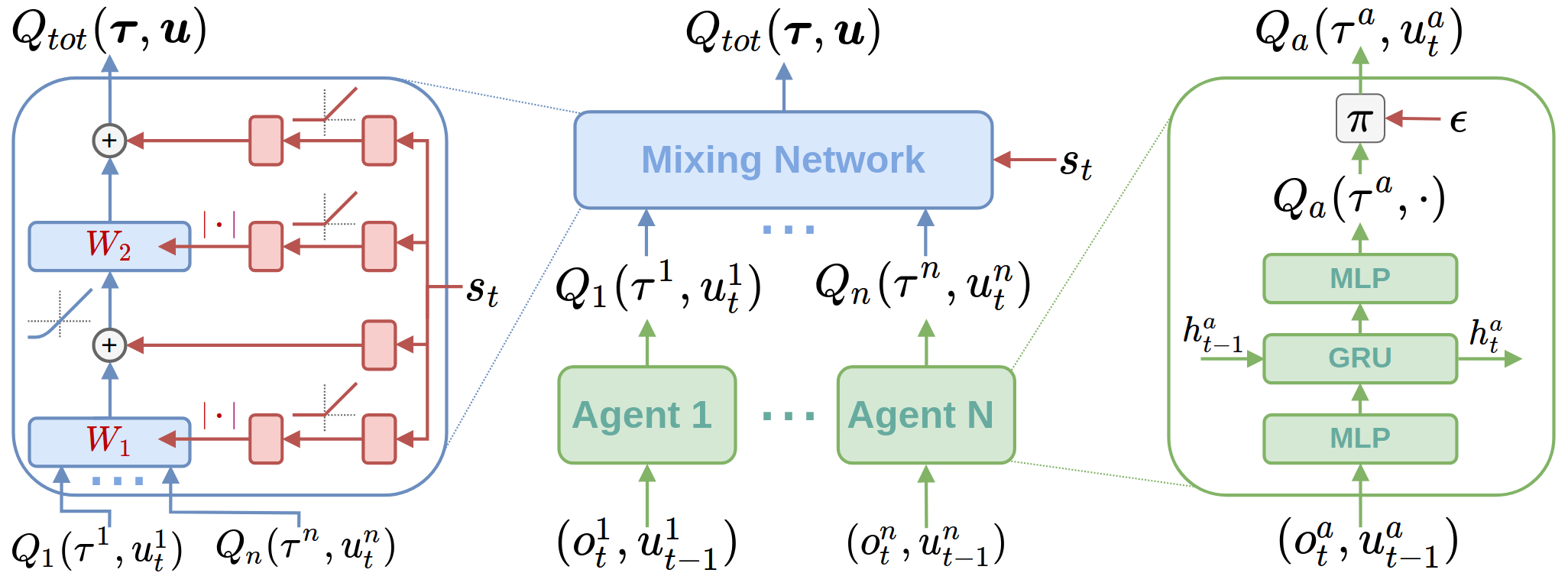}
	\text{~~~~~~~~~~~~~}(a) \hfill (b) \hfill (c) 
	\text{~~~~~~~~~~}
	\caption{(a) Mixing network structure. In red are the hypernetworks that 
	produce the weights and biases for mixing network layers shown in blue. (b) 
	The overall QMIX architecture. (c) Agent network structure. Best viewed in 
	colour.}
	\label{fig:QMIX}
\end{figure*}

For each agent $a$, there is one agent network that represents its individual value function $Q_a (\tau^a, u^a)$. We represent agent networks as DRQNs that receive the current individual observation $o^a_t$ and the last action $u^a_{t-1}$ as input at each time step, as shown in Figure \ref{fig:QMIX}c. We include the last actions since they are part of the action-observation history $\tau^a$ on which the decentralised policy can condition. Due to the use of stochastic policies during training, it is necessary to provide the actual action that was executed in the environment. If weights are shared across the agent networks in order to speed up  learning, agent IDs are included as part of the observations to allow for heterogeneous policies. %

The mixing network is a feed-forward neural network that takes the agent 
network outputs as input and mixes them monotonically, producing the values of 
$Q_{tot}$, as shown in Figure \ref{fig:QMIX}a. To enforce the monotonicity 
constraint of \eqref{eq:deriv-constr}, the weights (but not the biases) of the 
mixing network are restricted to be non-negative. This allows the mixing 
network to approximate any monotonic function arbitrarily closely in the limit 
of infinite width \citep{Dugas_2009}.

The weights of each layer of the mixing network are produced by separate 
hypernetworks. Each hypernetwork takes the state $s$ as input and generates the 
weights of one layer of the mixing network. 
Each hypernetwork consists of two fully-connected layers with a ReLU nonlinearity, followed by an absolute activation function, to ensure 
that the mixing network weights are non-negative. The output of the 
hypernetwork is then a vector, which is reshaped into a matrix of appropriate 
size. 
The biases are produced in the same manner but are not restricted to 
being non-negative. 
The first bias is produced by a hypernetwork with a single linear layer, and the final bias is produced by a two-layer hypernetwork with a 
ReLU nonlinearity. Figure \ref{fig:QMIX}a illustrates the mixing network and 
the hypernetworks.

The state is used by the hypernetworks rather than being passed directly into the mixing network because $Q_{tot}$ is allowed to depend on the extra state 
information in non-monotonic ways. Thus, it would be overly constraining to pass some function of $s$ through the monotonic network alongside the 
per-agent values.
Instead, the use of hypernetworks makes it possible to condition the 
weights of the monotonic network on $s$ in an arbitrary way, thus 
integrating the full state $s$ into the joint action-value estimates as 
flexibly as possible.
The choice of nonlinearity for the mixing network is also an important consideration due to its interaction with the non-negative weights.
A ReLU is not a good choice since a negative input to the mixing network is likely to remain negative (depending on the biases), which would then be zeroed out by the ReLU leading to no gradients for all agent networks. 
It is for this reason that we use an ELU.

QMIX is trained end-to-end to minimise the following loss:
\begin{equation}\label{eq:qmix_loss}
\mathcal{L}(\theta)=\sum\limits_{i=1}^b\left[\left(y_i^{tot} - Q_{tot}(\boldsymbol{\tau}, \mathbf{u}, s; \theta) \right)^2\right],
\end{equation} 
where $b$ is the batch size of transitions sampled from the replay buffer, the DQN target is given by $y^{tot} = r+\gamma\max_{\mathbf{u}'} Q_{tot}(\boldsymbol{\tau}', \mathbf{u}', s'; \theta^-)$, and $\theta^-$ are the parameters of a target network as in DQN. %
\eqref{eq:qmix_loss} is analogous to the standard DQN loss of \eqref{eq:dqn}. Since \eqref{eq:argmax_constist} holds, we can perform the maximisation of $Q_{tot}$ in time linear in the number of agents (as opposed to scaling exponentially in the worst case). %

The mixing network relies on centralised training, 
although it may be discarded after training to allow fully decentralised 
execution of the learned joint policy.
In a setting with limited communication, the core algorithmic mechanism of QMIX 
can still be applied: the mixing network can be viewed as a manager that 
coordinates learning, but requires only the communication of low-dimensional 
action-values from, and gradients to, the decentralised agents.
In a setting without communication constraints, we can further exploit the 
centralised training setting by, e.g., sharing parameters between agents for more 
efficient learning.

\begin{algorithm}[h!]
	\caption{QMIX}
	\label{alg:qmix}
	\begin{algorithmic}[1]
		\STATE \mbox{Initialise $\theta$, the parameters of mixing network, agent networks and hypernetwork.}
		\STATE Set the learning rate $\alpha$ and replay buffer $\mathcal{D} = \left\{ \right\}$
		\STATE $\text{step} = 0, \theta^- = \theta$  
		\WHILE{$\text{step} < \text{step}_{max}$}
		\STATE{$t = 0, s_0=\text{initial state}$}
		\WHILE{$s_t \neq terminal$ and $t < \text{episode limit}$}
		\FOR{each agent $a$}
		\STATE $\tau^a_t = \tau^a_{t-1} \cup \{(o_t, u_{t-1})\}$
		\STATE $\epsilon = \text{epsilon-schedule}(step)$
		\STATE {$u^a_t=
		\begin{cases}
			\argmax_{u^a_t}Q(\tau^a_t, u^a_t)~~~~\text{with probability }1 - \epsilon\\
			randint(1, |U|)~~~~~~~~~~\text{with probability }\epsilon\\
		\end{cases}$}
		\ENDFOR
		\STATE Get reward $r_t$ and next state $s_{t+1}$
		\STATE $\mathcal{D} = \mathcal{D} \cup \left\{(s_t, \mathbf{u}_t, r_t, s_{t+1} )\right\}$
		\STATE{$t=t+1, \text{step}=\text{step}+1$}
		\ENDWHILE
		\IF {$|\mathcal{D}| > \text{batch-size}$}
		\STATE {b $\leftarrow$ random batch of episodes from $\mathcal{{D}}$}
		\FOR {each timestep $t$ in each episode in batch $b$}
		\STATE {$Q_{tot} = \textit{Mixing-network~}(Q_1(\tau^1_t,u_t^1),\dots,Q_n(\tau^n_t,u_t^n); \textit{Hypernetwork}(s_t; \theta))$}
		\STATE {Calculate target $Q_{tot}$ using $\textit{Mixing-network}$ with $\textit{Hypernetwork}(s_t; \theta^-))$}
		\ENDFOR
		\STATE {$\Delta Q_{tot} = y^{tot} - Q_{tot}$ // Eq (\ref{eq:qmix_loss})} 
		\STATE {$\Delta \theta = \nabla_\theta(\Delta Q_{tot})^2$}
		\STATE {$\theta = \theta - \alpha \Delta \theta$}
		\ENDIF
		\IF {update-interval steps have passed}
		\STATE {$\theta^- = \theta$} 
		\ENDIF
		\ENDWHILE
	\end{algorithmic}
\end{algorithm}

In Algorithm \ref{alg:qmix} we outline the pseudocode for the particular implementation of QMIX we use for all of our experiments. 
The choice to gather rollouts from an entire episode (line 17) before executing a single gradient descent step (line 24), as well as using per-agent $\epsilon$-greedy action selection (line 10), is not a requirement for QMIX but is an implementation detail. However, the action selection based solely on agent network's $Q$-values (line 10), as well as calculation of $Q_{tot}$ and its target using the mixing network and hypernetworks (lines 19-20) are essential features of QMIX. The update of per-agent action-value functions and hypernetworks using the DQN-style loss with respect to $Q_{tot}$ and joint reward (lines 22-24) is also central to our method.

\subsection{Representational Complexity}

The value function class representable with QMIX includes any value function that can be factored into a nonlinear monotonic combination of the agents' individual value functions in the fully observable setting. 

This follows since the mixing network is a universal function approximator of monotonic functions \cite{Dugas_2009}, and hence can represent any value function that factors into a nonlinear monotonic combination of the agent's individual value functions. Additionally, we require that the agent's individual value functions order the values of the actions appropriately. By this we mean that $Q_a$ is such that $Q_a(s_t, u^a) > Q_a(s_t, u'^a) \iff Q_{tot}(s_t, (\mathbf{u}^{-a}, u^a)) > Q_{tot}(s_t, (\mathbf{u}^{-a}, u'^a))$, i.e., they can represent a function that respects the ordering of the agent's actions in the joint-action value function. Since the agents' networks are universal function approximators \cite{pinkus1999approximation}, they can represent such a $Q_a$. Hence QMIX is able to represent any value function that factors into a nonlinear monotonic combination of the agent's individual value functions. 

In a Dec-POMDP, QMIX cannot necessarily represent the true optimal $Q$-function. 
This is because each agent's observations are no longer the full state, and thus they might not be able to distinguish the true state given their local observations. If the agent's value function ordering is then wrong, i.e., $Q_a(\tau^a, u) > Q_a(\tau^a, u')$ when $Q_{tot}(s_t, (\mathbf{u}^{-a}, u)) < Q_{tot}(s_t, (\mathbf{u}^{-a}, u'))$, then the mixing network would be unable to correctly represent $Q_{tot}$ given the monotonicity constraints. 

QMIX thus expands upon the linear monotonic value functions that are representable by VDN. Table \ref{table_matrix_examples}a gives an example of a monotonic value function for the simple case of a two-agent matrix game. 
Note that VDN is unable to represent this simple monotonic value function.
Table \ref{table_matrix_examples_soln}a provides the optimal $Q_{tot}$ values for VDN, minimising a squared loss. Section \ref{sec:two_step_game} additionally provides results for the values that VDN learns in a function approximation setting.

\begin{table}[h]
    \centering
    \setlength{\extrarowheight}{3pt}
    \begin{tabular}{cc|*{2}{>{\centering\arraybackslash}p{.025\linewidth}|}}
        & \multicolumn{1}{c}{} & \multicolumn{2}{c}{Agent $2$} \\
        & \multicolumn{1}{c}{} & \multicolumn{1}{c}{$A$}  & \multicolumn{1}{c}{$B$} \\ \cline{3-4} 
        \multirow{2}{*}{\rotatebox[origin=c]{90}{Agent $1$}}  & $A$ & 0 & 1 \\ \cline{3-4}
        & $ B $ & 1 & 8  \\\cline{3-4}
        & \multicolumn{1}{c}{} & \multicolumn{2}{c}{(a)} \\
    \end{tabular}~~~~~~~
    \begin{tabular}{cc|*{2}{>{\centering\arraybackslash}p{.025\linewidth}|}}
        & \multicolumn{1}{c}{} & \multicolumn{2}{c}{Agent $2$} \\
        & \multicolumn{1}{c}{} & \multicolumn{1}{c}{$A$}  & \multicolumn{1}{c}{$B$} \\ \cline{3-4}
        \multirow{2}{*}{\rotatebox[origin=c]{90}{Agent $1$}}  & $A$ & 2 & 1 \\ \cline{3-4}
        & $ B $ & 1 & 8  \\\cline{3-4}
        & \multicolumn{1}{c}{} & \multicolumn{2}{c}{(b)} \\
    \end{tabular}
    \caption{(a) An example of a monotonic payoff matrix, (b) a nonmonotonic payoff matrix.}
    \label{table_matrix_examples}
\end{table}

\begin{table}[h]
    \centering
    \setlength{\extrarowheight}{3pt}
    \begin{tabular}{cc|*{2}{>{\centering\arraybackslash}p{.05\linewidth}|}}
        & \multicolumn{1}{c}{} & \multicolumn{2}{c}{Agent $2$} \\
        & \multicolumn{1}{c}{} & \multicolumn{1}{c}{$A$}  & \multicolumn{1}{c}{$B$} \\ \cline{3-4} 
        \multirow{2}{*}{\rotatebox[origin=c]{90}{Agent $1$}}  & $A$ & -1.5 & 2.5 \\ \cline{3-4}
        & $ B $ & 2.5 & 6.5  \\\cline{3-4}
        & \multicolumn{1}{c}{} & \multicolumn{2}{c}{(a)} \\
    \end{tabular}~~~~~~~
    \begin{tabular}{cc|*{2}{>{\centering\arraybackslash}p{.05\linewidth}|}}
        & \multicolumn{1}{c}{} & \multicolumn{2}{c}{Agent $2$} \\
        & \multicolumn{1}{c}{} & \multicolumn{1}{c}{$A$}  & \multicolumn{1}{c}{$B$} \\ \cline{3-4}
        \multirow{2}{*}{\rotatebox[origin=c]{90}{Agent $1$}}  & $A$ & 4/3 & 4/3 \\ \cline{3-4}
        & $ B $ & 4/3 & 8  \\\cline{3-4}
        & \multicolumn{1}{c}{} & \multicolumn{2}{c}{(b)} \\
    \end{tabular}
    \caption{(a) Optimal $Q_{tot}$ values for VDN on the example in Table \ref{table_matrix_examples}a. (b) Optimal $Q_{tot}$ values for QMIX on the example in Table \ref{table_matrix_examples}b.}
    \label{table_matrix_examples_soln}
\end{table}

However, the constraint in \eqref{eq:deriv-constr} prevents QMIX from representing value functions that do not factorise in such a manner. A simple example of such a value function for a two-agent matrix game is given in Table \ref{table_matrix_examples}b. Intuitively, any value function for which an agent's best action depends on the actions of the other agents \emph{at the same time step} will not factorise appropriately, and hence cannot be perfectly represented by QMIX. 
Table \ref{table_matrix_examples_soln}b provides the optimal $Q_{tot}$ values for QMIX when minimising a squared loss.
For the example in Table \ref{table_matrix_examples}b, QMIX still recovers the optimal policy and learns the correct maximum over the $Q$-values.
Formally, the class of value functions that cannot be represented by QMIX are called \emph{nonmonotonic} \citep{mahajan_maven:_2019}.

Although QMIX cannot perfectly represent all value functions, its increased representational capacity allows it to learn $Q$-value estimates that are more accurate for computing the target values to regress onto during $Q$-learning targets (bootstrapping). 
In particular, we show in Section \ref{sec:two_step_game} that the increased representational capacity allows for QMIX to learn the optimal policy in an environment with no per-timestep coordination, compared to VDN which still fails to learns the optimal policy.
In Section \ref{sec:rnd_matrix} we show that QMIX learns significantly more accurate maxima over $Q_{tot}$ in randomly generated matrix games, demonstrating that it learns better $Q$-value estimates for bootstrapping.   
Finally, we show in Section \ref{sec:results} that QMIX significantly outperforms VDN on the challenging SMAC benchmark.

\section{Two-Step Game}
\label{sec:two_step_game}

To illustrate the effects of representational complexity of VDN and QMIX, we devise a simple two-step cooperative matrix game for two agents. 

At the first step, Agent $1$ chooses which of the two matrix games to play in the next timestep. For the first time step, the actions of Agent $2$ have no effect. In the second step, both agents choose an action and receive a global reward according to the payoff matrices depicted in Table \ref{tab:2step_game}.

We train VDN and QMIX on this task for $5000$ episodes and examine the final learned value functions in the limit of full exploration ($\epsilon=1$). Full exploration ensures that each method is guaranteed to eventually explore all available game states, such that the representational capacity of the state-action value function approximation remains the only limitation.
The full details of the architecture and hyperparameters used, as well as additional results are provided in Appendix \ref{sec:2step_arch}.

\begin{table}[h!]
	\setlength{\extrarowheight}{3pt}
	\centering
	\begin{tabular}{cc|*{2}{>{\centering\arraybackslash}p{.05\linewidth}|}}
		& \multicolumn{1}{c}{} & \multicolumn{2}{c}{\bb{Agent $2$}} \\
		& \multicolumn{1}{c}{} & \multicolumn{1}{c}{\bb{$A$}}  & \multicolumn{1}{c}{\bb{$B$}} \\ 
		\cline{3-4}
        \multirow{2}{*}{\rotatebox[origin=c]{90}{\cc{Agent $1$}}} & \cc{$A$} & 7 & 7 \\ \cline{3-4}
        & \cc{$B$} & 7 & 7  \\\cline{3-4}
        & \multicolumn{1}{c}{}  & \multicolumn{2}{c}{State $2$A} \\
    \end {tabular}~~~~~~~
    \begin{tabular}{cc|*{2}{>{\centering\arraybackslash}p{.05\linewidth}|}}
    	& \multicolumn{1}{c}{} & \multicolumn{2}{c}{\bb{Agent $2$}} \\
        & \multicolumn{1}{c}{} & \multicolumn{1}{c}{\bb{$A$}}  & \multicolumn{1}{c}{\bb{$B$}} \\ 
        \cline{3-4}
		\multirow{2}{*}{\rotatebox[origin=c]{90}{\cc{Agent $1$}}} & \cc{$A$} & 0 & 1 \\ \cline{3-4}
		& \cc{$B$} & 1 & 8  \\\cline{3-4}
		& \multicolumn{1}{c}{} & \multicolumn{2}{c}{State $2$B} \\
	\end{tabular}
    \caption{Payoff matrices of the two-step game after the Agent 1 chose the first action. Action A takes the agents to State $2$A and action B takes them to State $2$B.}
    \label{tab:2step_game}
\end{table}

\begin{table}[h]
	\setlength{\extrarowheight}{3pt}
	\centering
	(a)
	\begin{tabular}{c|*{2}{>{\centering\arraybackslash}p{.08\linewidth}|}}
		\multicolumn{1}{c}{} & \multicolumn{2}{c}{State $1$} \\
		\multicolumn{1}{c}{} & \multicolumn{1}{c}{\bb{$A$}}  & \multicolumn{1}{c}{\bb{$B$}} \\ \cline{2-3}
		\cc{$A$} & 6.94 & 6.94 \\ \cline{2-3}
		\cc{$B$} & 6.35 & 6.36  \\\cline{2-3}
	\end{tabular}~
	\begin{tabular}{|*{2}{>{\centering\arraybackslash}p{.08\linewidth}|}}
		\multicolumn{2}{c}{State $2$A} \\
		\multicolumn{1}{c}{\bb{$A$}}  & \multicolumn{1}{c}{\bb{$B$}} \\ \cline{1-2}
		6.99 & 7.02 \\\cline{1-2}
		6.99 & 7.02  \\\cline{1-2}
	\end{tabular}~
	\begin{tabular}{|*{2}{>{\centering\arraybackslash}p{.08\linewidth}|}}
		\multicolumn{2}{c}{State $2$B} \\
		\multicolumn{1}{c}{\bb{$A$}}  & \multicolumn{1}{c}{\bb{$B$}} \\\cline{1-2}
		\text{-1.87} & 2.31 \\\cline{1-2}
		2.33 & 6.51  \\\cline{1-2}
	\end{tabular}\\\bigskip

	(b)
	\begin{tabular}{c|*{2}{>{\centering\arraybackslash}p{.08\linewidth}|}}
		\multicolumn{1}{c}{} & \multicolumn{1}{c}{\bb{$A$}}  & \multicolumn{1}{c}{\bb{$B$}} \\ \cline{2-3}
		\cc{$A$} & 6.93 & 6.93  \\ \cline{2-3}
		\cc{$B$} & 7.92 & 7.92  \\\cline{2-3}
	\end{tabular}~
	\begin{tabular}{|*{2}{>{\centering\arraybackslash}p{.08\linewidth}|}}
		\multicolumn{1}{c}{\bb{$A$}}  & \multicolumn{1}{c}{\bb{$B$}} \\ \cline{1-2}
		7.00 & 7.00 \\ \cline{1-2}
		7.00 & 7.00  \\\cline{1-2}
	\end{tabular}~
	\begin{tabular}{|*{2}{>{\centering\arraybackslash}p{.08\linewidth}|}}
		\multicolumn{1}{c}{\bb{$A$}}  & \multicolumn{1}{c}{\bb{$B$}} \\\cline{1-2}
		0.00 & 1.00 \\\cline{1-2}
		1.00 & 8.00 \\\cline{1-2}
	\end{tabular}\\
    \caption{$Q_{tot}$ on the two-step game for (a) VDN and (b) QMIX.}
    \label{qmix_2step_game_main}
\end{table}

Table \ref{qmix_2step_game_main}, which shows the learned values for $Q_{tot}$, demonstrates that QMIX's higher representational capacity allows it to accurately represent the joint-action value function whereas VDN cannot. This directly translates into VDN learning the suboptimal strategy of selecting Action A at the first step and receiving a reward of 7, whereas QMIX recovers the optimal strategy from its learnt joint-action values and receives a reward of 8. 

The simple example presented in this section demonstrates the importance of accurate $Q$-value estimates, especially for the purpose of bootstrapping.
In Section \ref{sec:rnd_matrix} we provide further evidence that QMIX's increased representational capacity allows it to learn more accurate $Q$-value, that directly lead to more accurate bootstrapped estimates.

\section{Random Matrix Games}
\label{sec:rnd_matrix}

To demonstrate that QMIX learns more accurate $Q$-values than VDN for the purpose of bootstrapping, we compare 
the learnt values of $\max_{\mathbf{u}} Q_{tot}(s, \mathbf{u})$
on randomly generated matrix games.
Single-step matrix games allow us to compare the learnt maximum $Q$-values across both methods in isolation.
We focus on $\max_{\mathbf{u}} Q_{tot}(s, \mathbf{u})$ because it is the only learnt quantity involved in the bootstrapping for 1-step $Q$-learning.

The maximum return for each matrix is $10$, and all other values are drawn uniformly from $[0,10)$.
Each individual seed has different values for the payoff matrix, but within each run they remain fixed.
We set $\epsilon=1$ to ensure all actions are sampled and trained on equally.

\begin{figure*}[h!]
	\centering
	\includegraphics[width=0.32\textwidth]{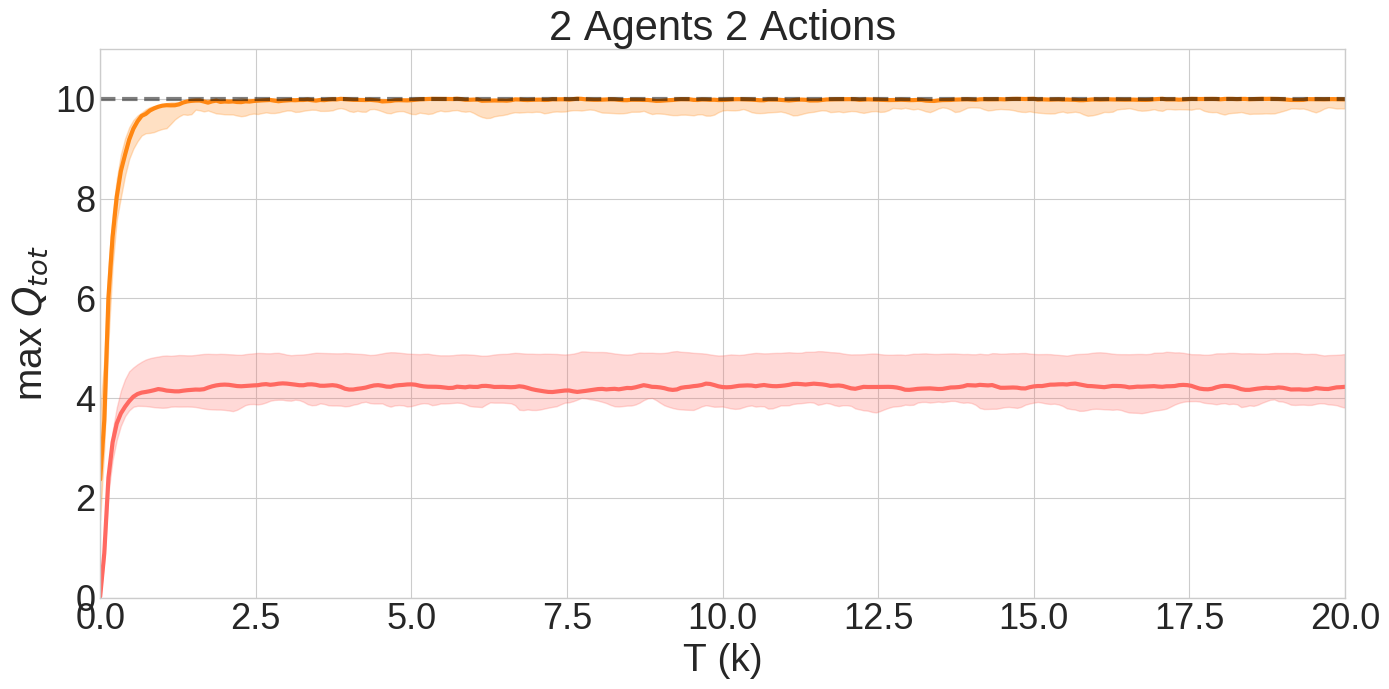}
	\includegraphics[width=0.32\textwidth]{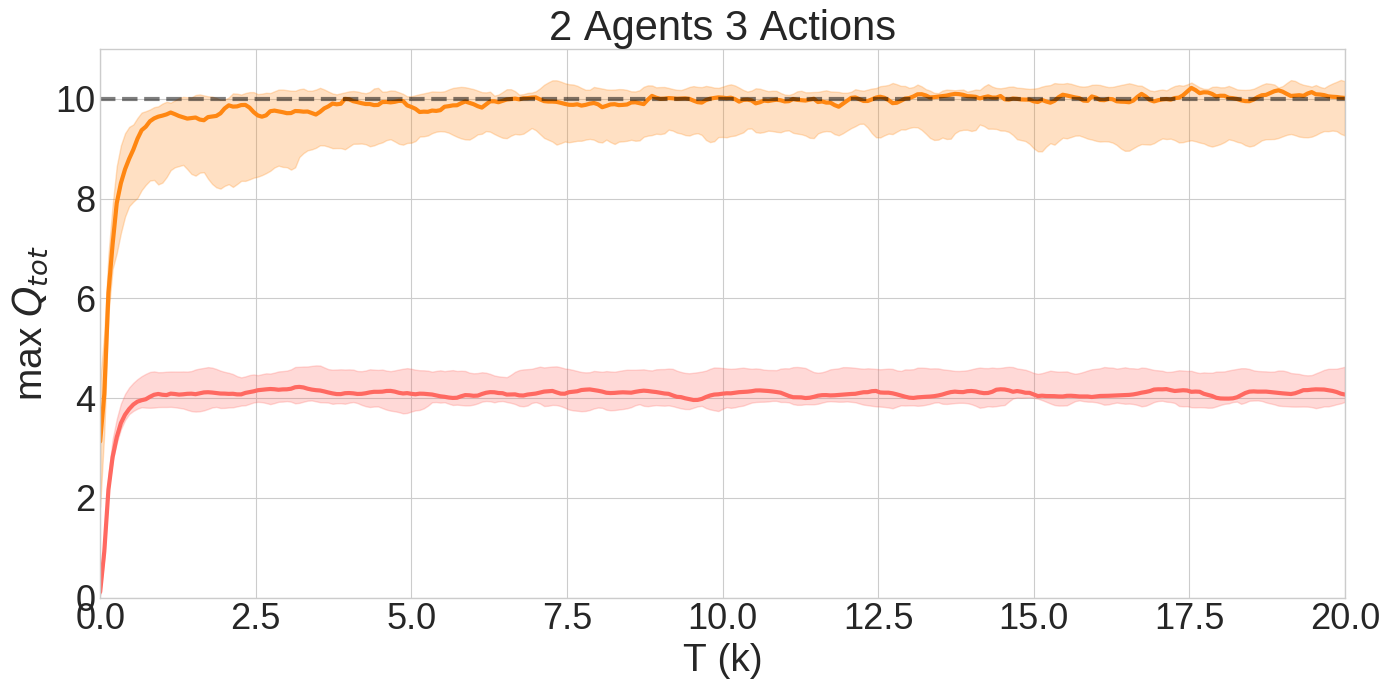}
	\includegraphics[width=0.32\textwidth]{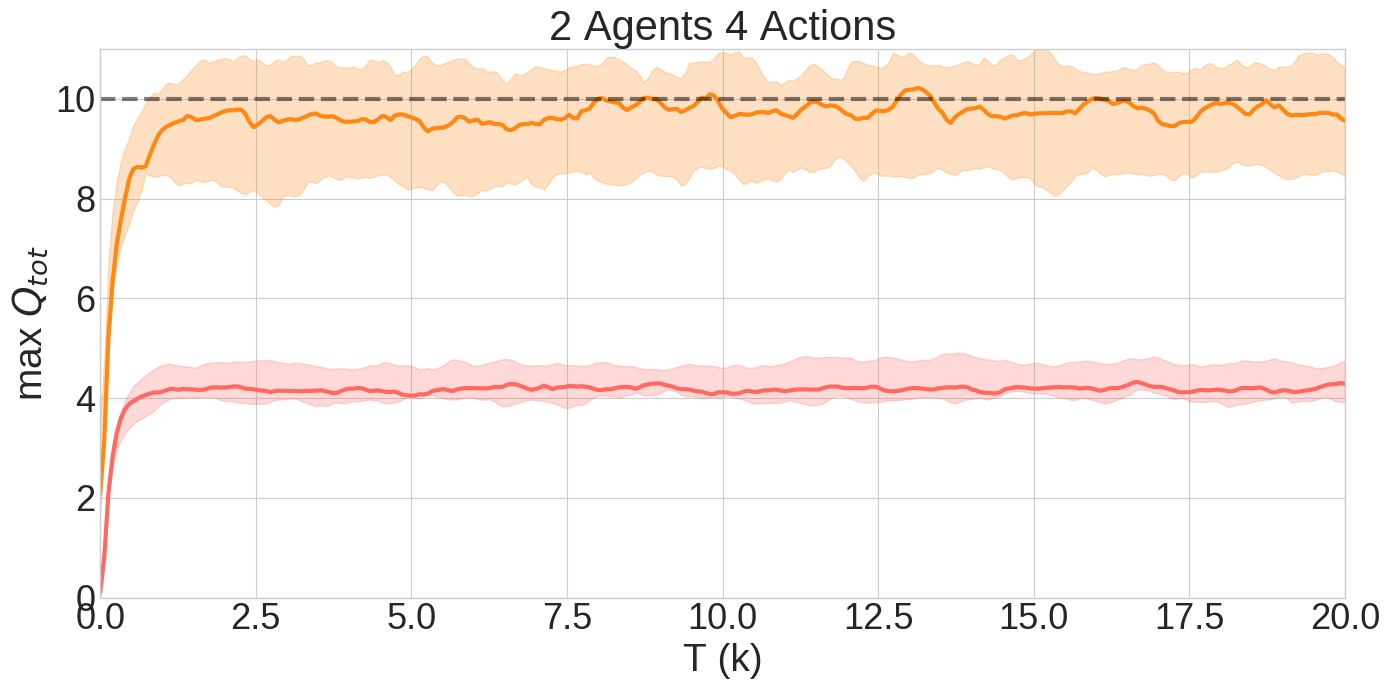}
	\includegraphics[width=0.32\textwidth]{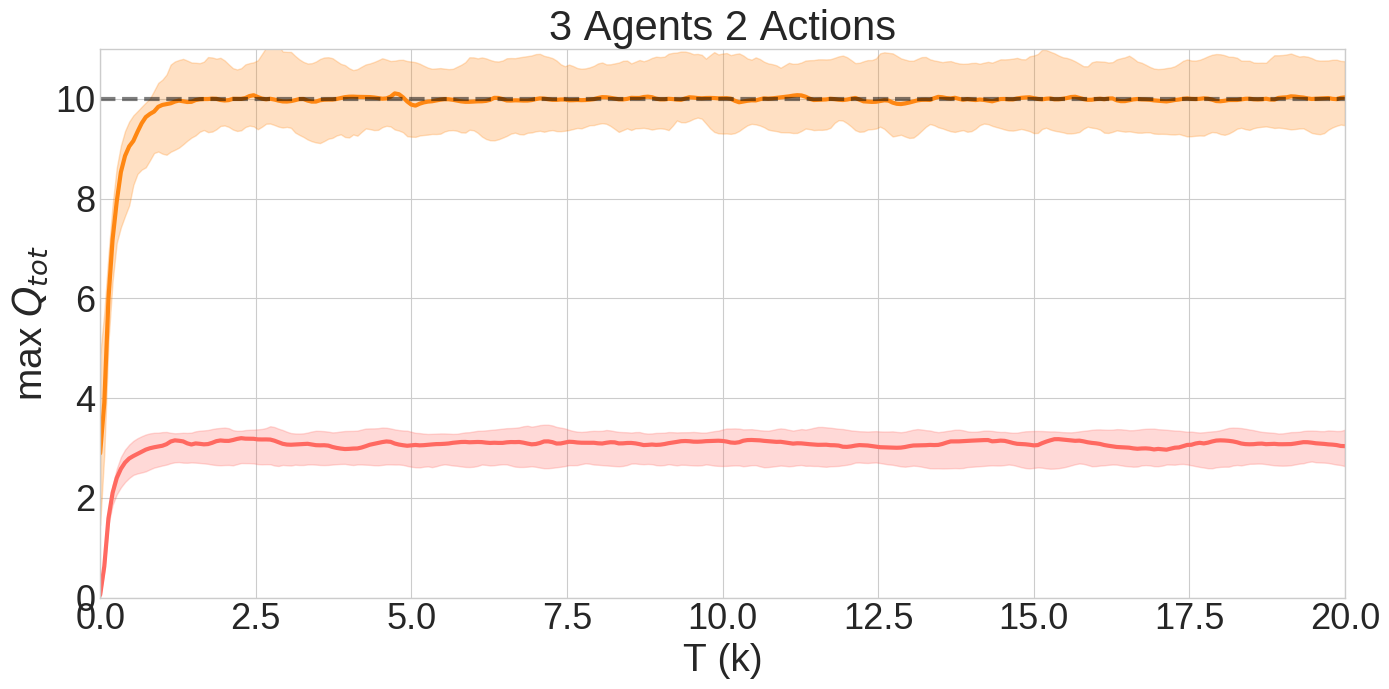}
	\includegraphics[width=0.32\textwidth]{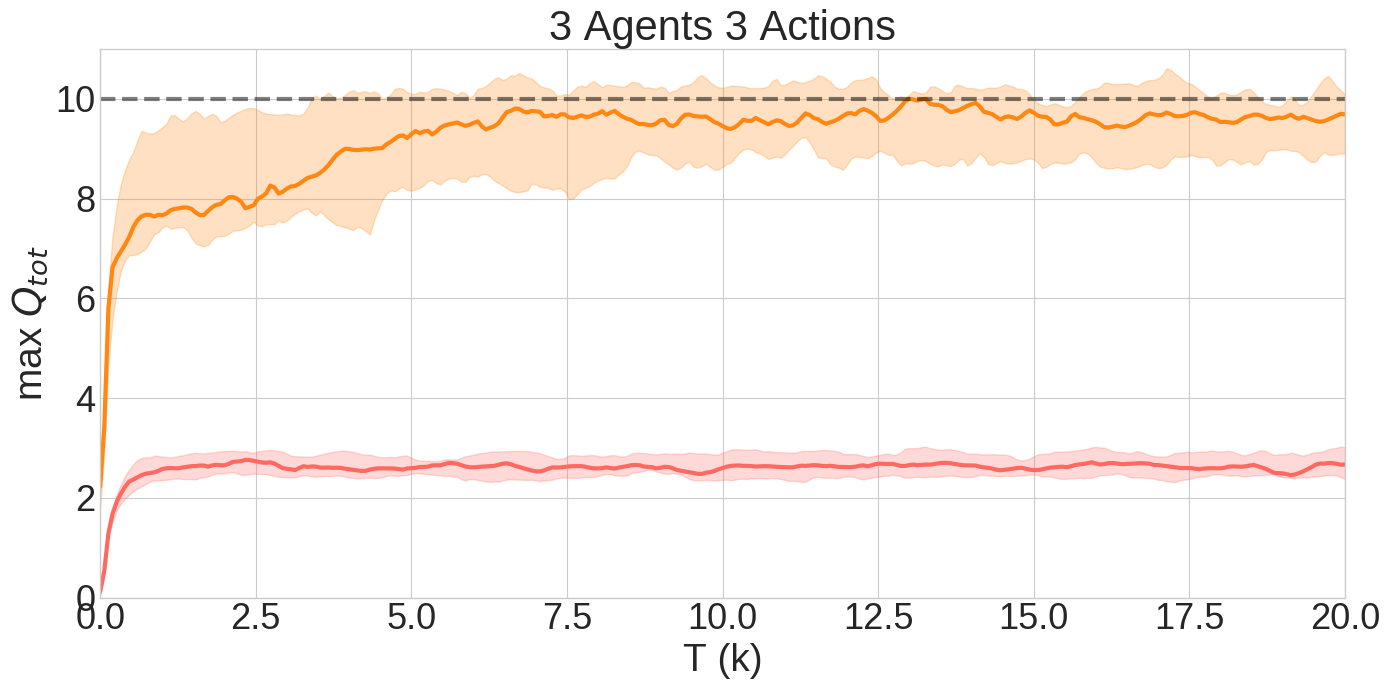}
	\includegraphics[width=0.32\textwidth]{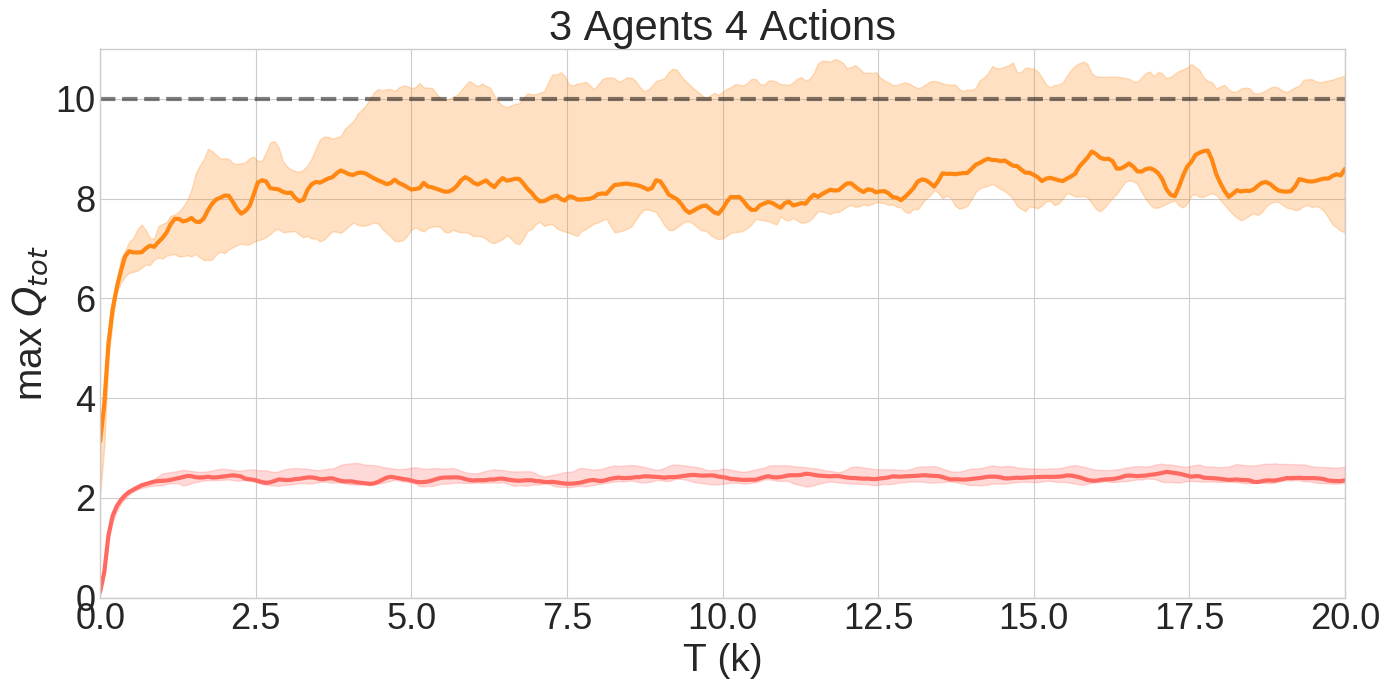}
	\includegraphics[width=0.32\textwidth]{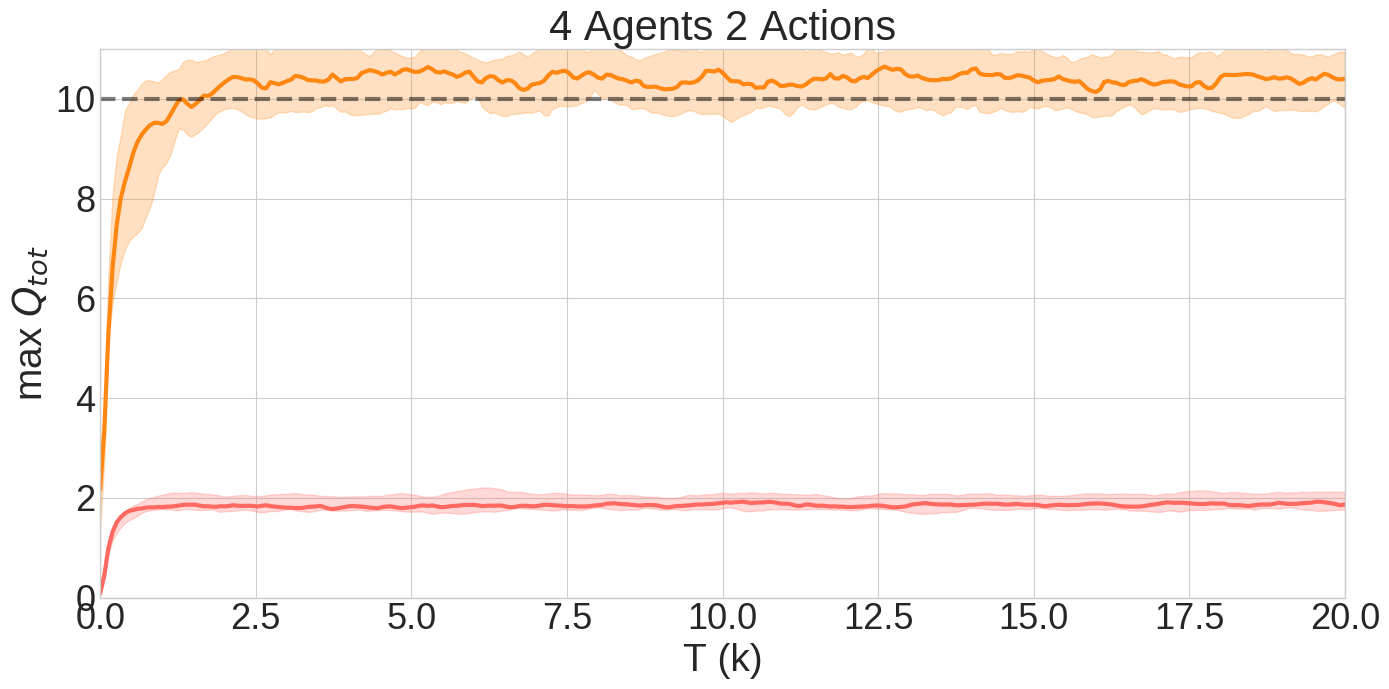}
	\includegraphics[width=0.32\textwidth]{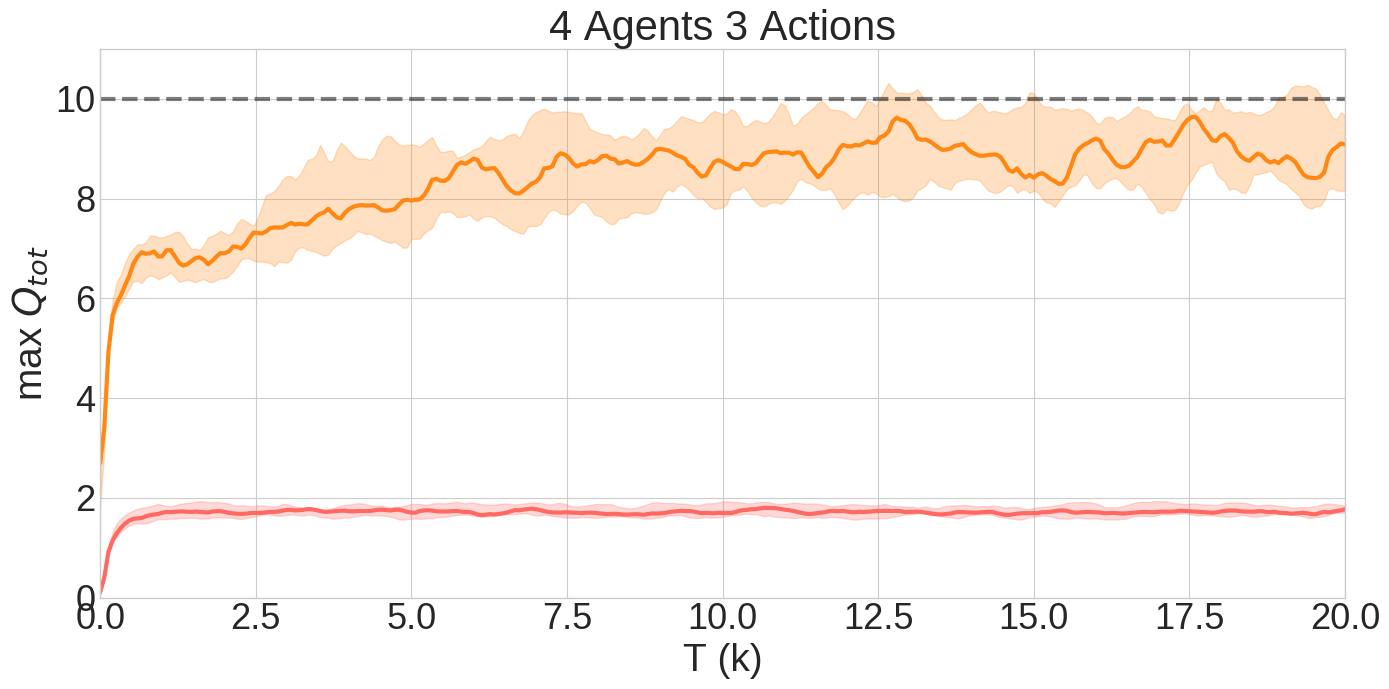}
	\includegraphics[width=0.32\textwidth]{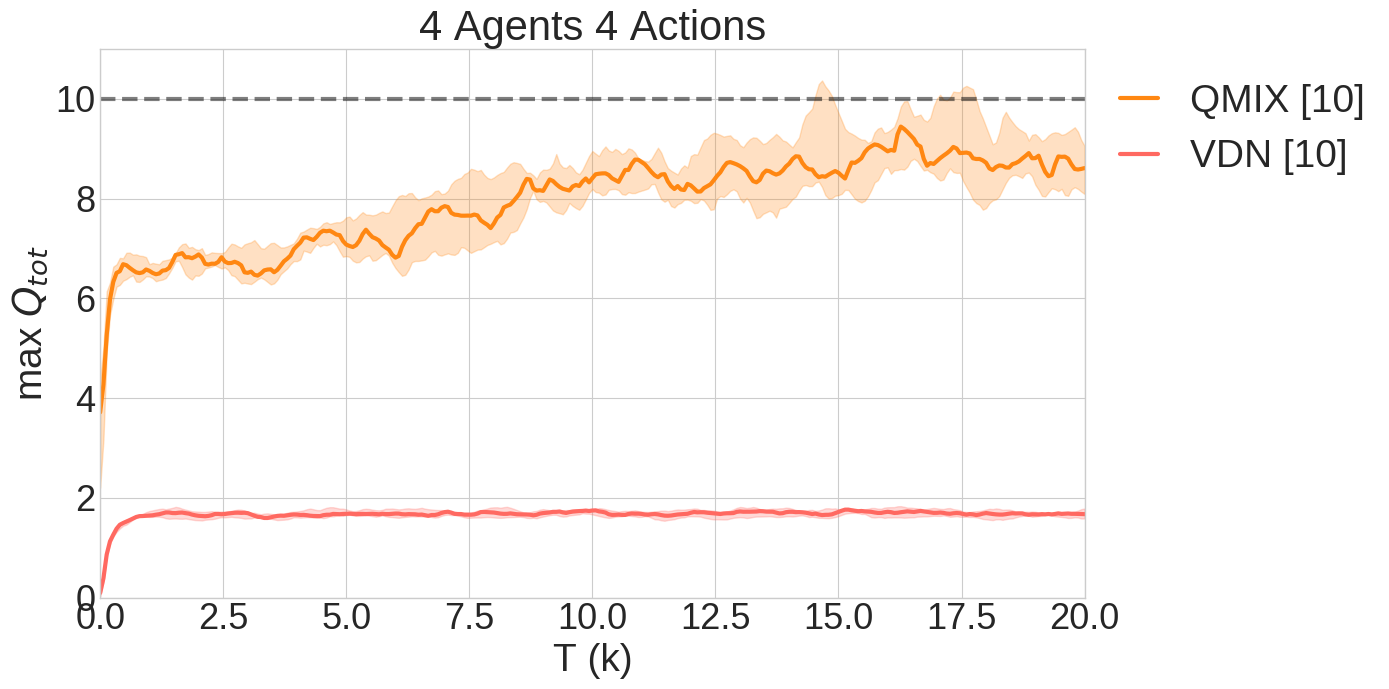}
	\caption{The median $\max_{\mathbf{u}} Q_{tot}(s, \mathbf{u})$ for $\{2,3,4\}$ agents with $\{2,3,4\}$ actions across 10 runs for VDN and QMIX. 25\%-75\% quartile is shown shaded. The dashed line at 10 indicates the correct value.}
	\label{fig:rnd_matrix}
\end{figure*}

Figure \ref{fig:rnd_matrix} shows the results for a varying number of agents and actions. 
We can see that QMIX learns significantly more accurate maxima than VDN due to its larger representational capacity.  

\section{The StarCraft Multi-Agent Challenge}
\label{sec:setting}

In this section, we describe the StarCraft Multi-Agent Challenge (SMAC) to which we apply QMIX and a number of other methods.
SMAC is based on the popular real-time strategy (RTS) game StarCraft II.
In a regular full game of StarCraft II, one or more humans compete against each other or against a built-in game AI to gather resources, construct buildings, and build armies of units to defeat their opponents.

Akin to most RTSs, StarCraft has two main gameplay components: macromanagement and micromanagement. \emph{Macromanagement} refers to high-level strategic considerations, such as economy and resource management.
\emph{Micromanagement} (micro), on the other hand, refers to fine-grained control of individual units.

StarCraft has been used as a research platform for AI, and more recently, RL. Typically, the game is framed as a competitive problem: an agent takes the role of a human player, making macromanagement decisions and performing micromanagement as a puppeteer that issues orders to individual units from a centralised controller.

In order to build a rich multi-agent testbed, we instead focus solely on micromanagement.
Micro is a vital aspect of StarCraft gameplay with a high skill ceiling, and is practiced in isolation by amateur and professional players.
For SMAC, we leverage the natural multi-agent structure of micromanagement by proposing a modified version of the problem designed specifically for decentralised control.
In particular, we require that each unit be controlled by an independent agent that conditions only on local observations restricted to a limited field of view centred on that unit (see Figure \ref{fig:obs}).
Groups of these agents must be trained to solve challenging combat scenarios, battling an opposing army under the centralised control of the game's built-in scripted AI.

\begin{figure}[t!]
	\centering
	\includegraphics[width=.6\linewidth]{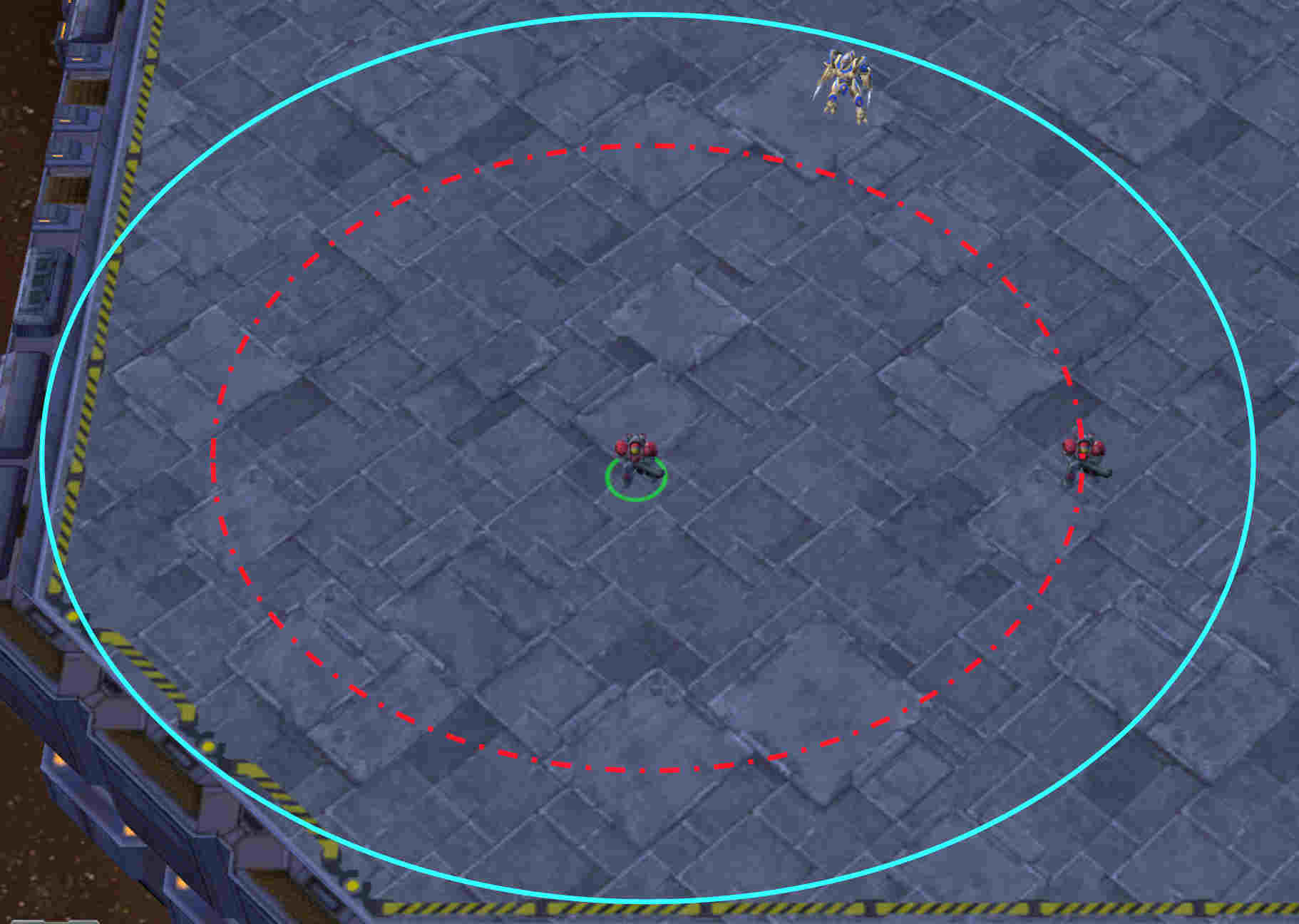}
	\caption{The cyan and red circles respectively border the sight and shooting range of the agent.}
	\label{fig:obs}
\end{figure}

Proper micro of units during battles maximises the damage dealt to enemy units while minimising damage received, and requires a range of skills.
For example, one important technique is \textit{focus fire}, i.e., ordering units to jointly attack and kill enemy units one after another. When focusing fire, it is important to avoid \textit{overkill}: inflicting more damage to units than is necessary to kill them.
Other common micro techniques include: assembling units into formations based on their armour types, making enemy units give chase while maintaining enough distance so that little or no damage is incurred (\textit{kiting},  Figure~\ref{fig:starcraft_screenshots}a), coordinating the positioning of units to attack from different directions or taking advantage of the terrain to defeat the enemy.

\begin{figure}[t!]
	\centering
	\subfigure[\texttt{MMM2}]{
		\includegraphics[width=0.4\columnwidth]{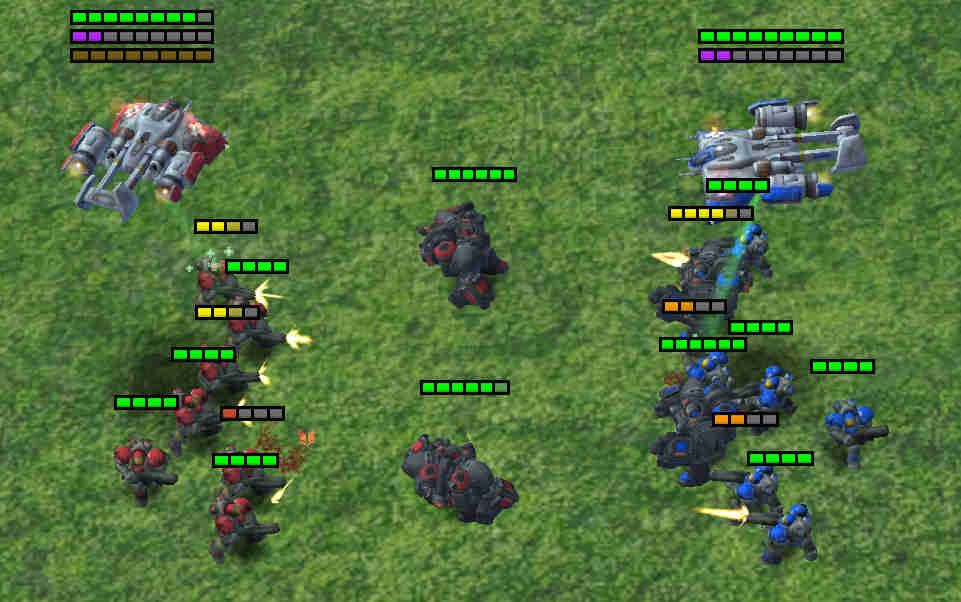}}
	\subfigure[\texttt{corridor}]{
		\includegraphics[width=0.4\columnwidth]{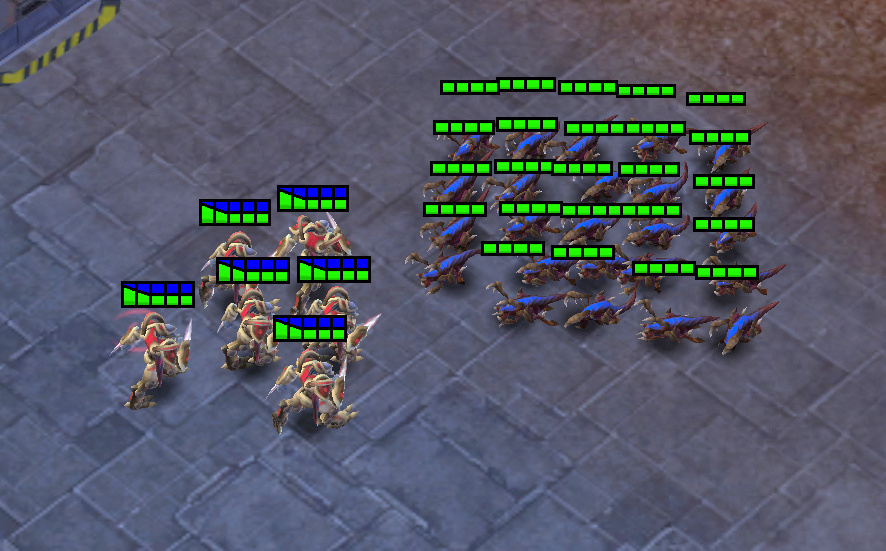}}
	\caption{\label{fig:SC2maps_2}Screenshots of two SMAC scenarios.}
\end{figure}

SMAC thus provides a convenient environment for evaluating the effectiveness of MARL algorithms. The simulated StarCraft II environment and carefully designed scenarios require learning rich cooperative behaviours under partial observability, which is a challenging task. The simulated environment also provides an additional state information during training, such as information on all the units on the entire map. This is crucial for facilitating algorithms to take full advantage of the centralised training regime and assessing all aspects of MARL methods.

SMAC features the following characteristics that are common in many real-word multi-agent systems:
\begin{itemize}
	\setlength{\itemsep}{0pt}
	\setlength{\parskip}{2pt}
	\item Partial observability: Like, e.g., 
	self-driving cars and autonomous drones, agents in SMAC receive only 
	limited information about the environment.
	\item Large number of agents: the scenarios in SMAC include up to 27 learning agents.
	\item Diversity: many scenarios feature heterogeneous units with resulting 
	diversity in optimal strategies.
	\item Long-term planning: defeating the enemy in SMAC often requires the agents to perform a long sequence of actions.
	\item High-dimensional observation spaces: the inclusion of diverse units and the large size of the map yield an immense state space.
	\item Large per-agent action space: the agents can perform up to 70 actions 
	at each time step.
	\item Coordinated teamwork: micromanagement of units requires the individual agents to execute strictly coordinated actions, which is essential to many MARL problems.
	\item Stochasticity: the behaviour of the enemy differs across individual runs. Also, the amount of time that the agents must wait until being able to shoot again is stochastic.
\end{itemize}
SMAC is one of the first MARL benchmarks that includes all of these features, 
making it useful for evaluating the effectiveness of methods for many aspects 
of learning decentralised multi-agent control.
We hope that it will become a standard benchmark for measuring the progress and a grand challenge for pushing the boundaries of MARL.
\subsection{Scenarios}

SMAC consists of a set of StarCraft II micro scenarios which aim to evaluate how well independent agents are able to learn coordination to solve complex tasks. 
These scenarios are carefully designed to necessitate the learning of one or more micro techniques to defeat the enemy.
Each scenario is a confrontation between two armies of units.
The initial position, number, and type of units in each army varies from scenario to scenario, as does the presence or absence of elevated or impassable terrain. Figures \ref{fig:starcraft_screenshots}  and \ref{fig:SC2maps_2} include screenshots of several SMAC micro scenarios.

The first army is controlled by the learned allied agents.
The second army consists of enemy units controlled by the built-in game AI, which uses carefully handcrafted non-learned heuristics.
At the beginning of each episode, the game AI instructs its units to attack the allied agents using its scripted strategies.
An episode ends when all units of either army have died or when a pre-specified time limit is reached (in which case the game is counted as a defeat for the allied agents).
The goal is to maximise the win rate, i.e., the ratio of games won to games played.

The complete list of challenges is presented in Table~\ref{tab:scenario}. More specifics on the SMAC scenarios and environment settings can be found in Appendices \ref{appendix:SMAC_scenarios} and \ref{appendix:SMAC_evn_setting} respectively.

\begin{table*}
	\scalebox{.9}{
		\begin{tabular}{ccc}
			\toprule
			Name & Ally Units & Enemy Units \\
			\midrule
			\texttt{2s3z} &  2 Stalkers \& 3 Zealots &  2 Stalkers \& 3 Zealots \\
			\texttt{3s5z} &  3 Stalkers \&  5 Zealots &  3 Stalkers \&  5 Zealots \\
			\texttt{1c3s5z} &  1 Colossus, 3 Stalkers \&  5 Zealots &  1 Colossus, 3 Stalkers \&  5 Zealots \\
			\hline
			\texttt{5m\_vs\_6m} & 5 Marines & 6 Marines \\
			\texttt{10m\_vs\_11m} & 10 Marines & 11 Marines \\
			\texttt{27m\_vs\_30m} & 27 Marines & 30 Marines \\
			\texttt{3s5z\_vs\_3s6z} & 3 Stalkers \& 5 Zealots & 3 Stalkers \& 6 Zealots \\
			\texttt{MMM2} &  1 Medivac, 2 Marauders \& 7 Marines
			&  1 Medivac, 3 Marauders \& 8 Marines  \\
			\hline
			\texttt{2s\_vs\_1sc}& 2 Stalkers  & 1 Spine Crawler \\
			\texttt{3s\_vs\_5z} & 3 Stalkers & 5 Zealots \\
			\texttt{6h\_vs\_8z} & 6 Hydralisks  & 8 Zealots \\
			\texttt{bane\_vs\_bane} & 20 Zerglings \& 4 Banelings  & 20 Zerglings \& 4 Banelings \\		
			\texttt{2c\_vs\_64zg}& 2 Colossi  & 64 Zerglings \\		
			\texttt{corridor} & 6 Zealots  & 24 Zerglings \\
			\bottomrule
	\end{tabular}}
	\caption{SMAC challenges. Note that list of SMAC scenarios has been updated from the earlier version. All scenarios, however, are still available in the repository.}
	\label{tab:scenario}
\end{table*}

\subsection{State and Observations}\label{section:state_and_obs}

At each timestep, agents receive local observations drawn within their field of view. This  encompasses information about the map within a circular area around each unit and with a radius equal to the \textit{sight range} (Figure \ref{fig:obs}). The sight range makes the environment partially observable from the standpoint of each agent. Agents can only observe other agents if they are both alive and located within the sight range. Hence, there is no way for agents to distinguish between teammates that are far away from those that are dead.

The feature vector observed by each agent contains the following attributes for both allied and enemy units within the sight range: \texttt{distance}, \texttt{relative x}, \texttt{relative y}, \texttt{health}, \texttt{shield}, and \texttt{unit\_type}. Shields serve as an additional source of protection that needs to be removed before any damage can be done to the health of units.
All Protoss units have shields, which can regenerate if no new damage is dealt.
In addition, agents have access to the last actions of allied units that are in the field of view. Lastly, agents can observe the terrain features surrounding them, in particular, the values of eight points at a fixed radius indicating height and walkability.

The global state, which is only available to agents during centralised training, contains information about all units on the map. Specifically, the state vector includes the coordinates of all agents relative to the centre of the map, together with unit features present in the observations. Additionally, the state stores the \texttt{energy} of Medivacs and \texttt{cooldown} of the rest of the allied units, which represents the  minimum delay between attacks. Finally, the last actions of all agents are attached to the central state.

All features, both in the state as well as in the observations of individual agents, are normalised by their maximum values. The sight range is set to nine for all agents. The feature vectors for both local observations and global state have a fixed size, where information about the ally/enemy is represented as a sequence ordered by the unit ID number.

\subsection{Action Space}

The discrete set of actions that agents are allowed to take consists of \texttt{move[direction]}\footnote{Four directions: north, south, east, or west.}, \texttt{attack[enemy\_id]}, \texttt{stop} and \texttt{no-op}.\footnote{Dead agents can only take \texttt{no-op} action while live agents cannot.}
As healer units, Medivacs must use \texttt{heal[agent\_id]} actions instead of \texttt{attack[enemy\_id]}. The maximum number of actions an agent can take ranges between 7 and 70, depending on the scenario.

To ensure decentralisation of the task, agents can use the \texttt{attack[enemy\_id]} action only on enemies in their \textit{shooting range} (Figure \ref{fig:obs}).
This additionally constrains the ability of the units to use the built-in \emph{attack-move} macro-actions on enemies that are far away. We set the shooting range to six for all agents. Having a larger sight range than a shooting range forces agents to use move commands before starting to fire.

To ensure that agents only execute valid actions, a mask is provided indicating which actions are valid and invalid at each timestep.

\subsection{Rewards}

The overall goal is to maximise the win rate for each battle scenario.
The default setting is to use the \textit{shaped reward}, which produces a reward based on the hit-point damage dealt and enemy units killed, together with a special bonus for winning the battle.
The exact values and scales for each of these events can be configured using a range of flags. 
To produce fair comparisions we encourage using this default reward function for all scenarios.
We also provide another \textit{sparse reward} option, in which the reward is +1 for winning and -1 for losing an episode.

\subsection{Evaluation Methodology}
To ensure the fairness of the challenge and comparability of results, performance should be evaluated under standardised conditions. 
One should not undertake any changes to the environment used for evaluating the policies. 
This includes the observation and state spaces, action space, the game mechanics, and settings of the environment (e.g., frame-skipping rate).
One should not modify the StarCraft II map files in any way or change the difficulty of the game AI. Episode limits of each scenario should also remain unchanged. %

SMAC restricts the execution of the trained models to be decentralised, i.e., during testing each agent must base its policy solely on its own action-observation history and cannot use the global state or the observations of other agents. %
It is, however, acceptable to train the decentralised policies in centralised fashion. Specifically, agents can exchange individual observations, model parameters and gradients during training as well as make use of the global state. %

\subsection{Evaluation Metrics}

Our main evaluation metric is the mean win percentage of evaluation episodes as a function of environment steps observed, over the course of training. 
Such progress can be estimated by periodically running a fixed number of evaluation episodes (in practice, 32) with any exploratory behaviours disabled. 
Each experiment is repeated using a number of independent training runs and the resulting plots include the median performance as well as the 25-75\% percentiles. 
We use five independent runs for this purpose in order to strike a balance between statistical significance and the computational requirements.
We recommend using the median instead of the mean in order to avoid the effect of any outliers. 
We report the number of independent runs, as well as environment steps used in training. %

We also include the computational resources used, as well as the wall clock time for running each experiment. SMAC provides functionality for saving StarCraft II replays, which can be viewed using a freely available client. The resulting videos can be used to comment on interesting behaviours observed. Each independent run takes between 8 to 16 hours, depending on the exact scenario, using Nvidia Geforce GTX 1080 Ti graphics cards.

\section{PyMARL}
\label{sec:pymarl}

To make it easier to develop algorithms for SMAC, we have also open-sourced our 
software engineering framework PyMARL\footnote{PyMARL is available at 
\url{https://github.com/oxwhirl/pymarl}.}.
Compared to many other frameworks for deep reinforcement learning  
\citep{stooke2019rlpyt,stable-baselines,hoffman2020acme}, PyMARL is designed 
with a focus on the multi-agent setting.
It is also intended for fast, easy development and experimentation for 
researchers, relying on fewer abstractions than e.g., RLLib 
\cite{liang2018rllib}.
To maintain this lightweight development experience, PyMARL does not support 
distributed training.

PyMARL's codebase is organized in a modular fashion in order to enable the rapid development of new algorithms, as well as provide implementations of current deep MARL algorithms to benchmark against. It is built on top of PyTorch to facilitate the fast execution and training of deep neural networks, and take advantage of the rich ecosystem built around it. PyMARL's modularity makes it easy to extend, and components can be readily isolated for testing purposes.

Since the implementation and development of deep MARL algorithms come with a number of additional challenges beyond those posed by single-agent deep RL, it is crucial to have simple and understandable code. In order to improve the readability of code and simplify the handling of data between components, PyMARL encapsulates all data stored in the buffer within an easy to use data structure. This encapsulation provides a cleaner interface for the necessary handling of data in deep MARL algorithms, whilst not obstructing the manipulation of the underlying PyTorch Tensors. In addition, PyMARL aims to maximise the batching of data when performing inference or learning so as to provide significant speed-ups over more naive implementations. 

PyMARL features implementations of the following algorithms: QMIX, QTRAN \cite{son_qtran:_2019}, COMA \cite{foerster_counterfactual_2017}, VDN \cite{sunehag_value-decomposition_2017}, and IQL \cite{tan_multi-agent_1993} as baselines.

\section{SMAC Results}
\label{sec:results}

In this section, we present the results of our experiments of QMIX and other existing multi-agent RL methods using the new SMAC benchmark.

The evaluation procedure is similar to the one in \cite{rashid2018qmix}. The training is paused after every $10000$ timesteps during which $32$ test episodes are run with agents performing action selection greedily in a decentralised fashion. The percentage of episodes where the agents defeat all enemy units within the permitted time limit is referred to as the \textit{test win rate}.

The architectures and training details are presented in Appendix~\ref{sec:smac_setup}.
A table of the results is included in Appendix~\ref{sec:table_results}.
Data from the individual runs are available at the SMAC repository (\url{https://github.com/oxwhirl/smac}).

\begin{figure*}[h!]
	\centering
	\includegraphics[width=0.45\textwidth]{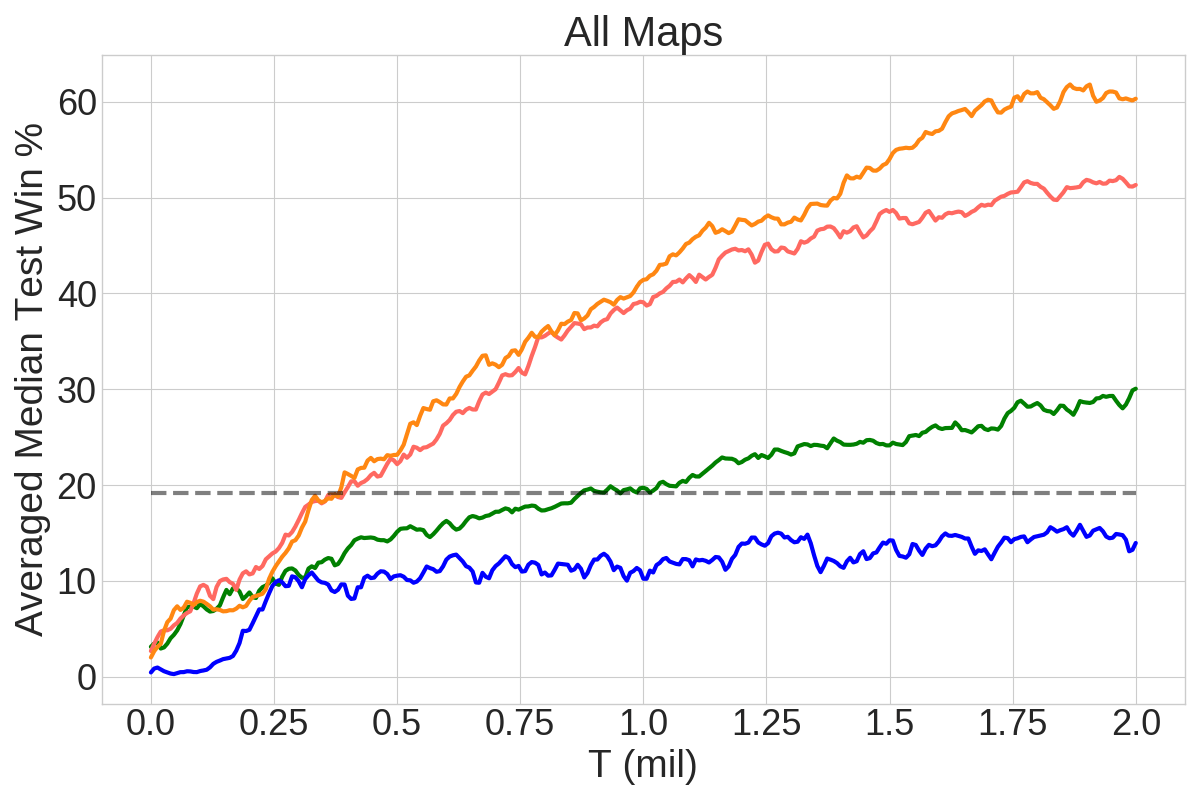}
	\includegraphics[width=0.45\textwidth]{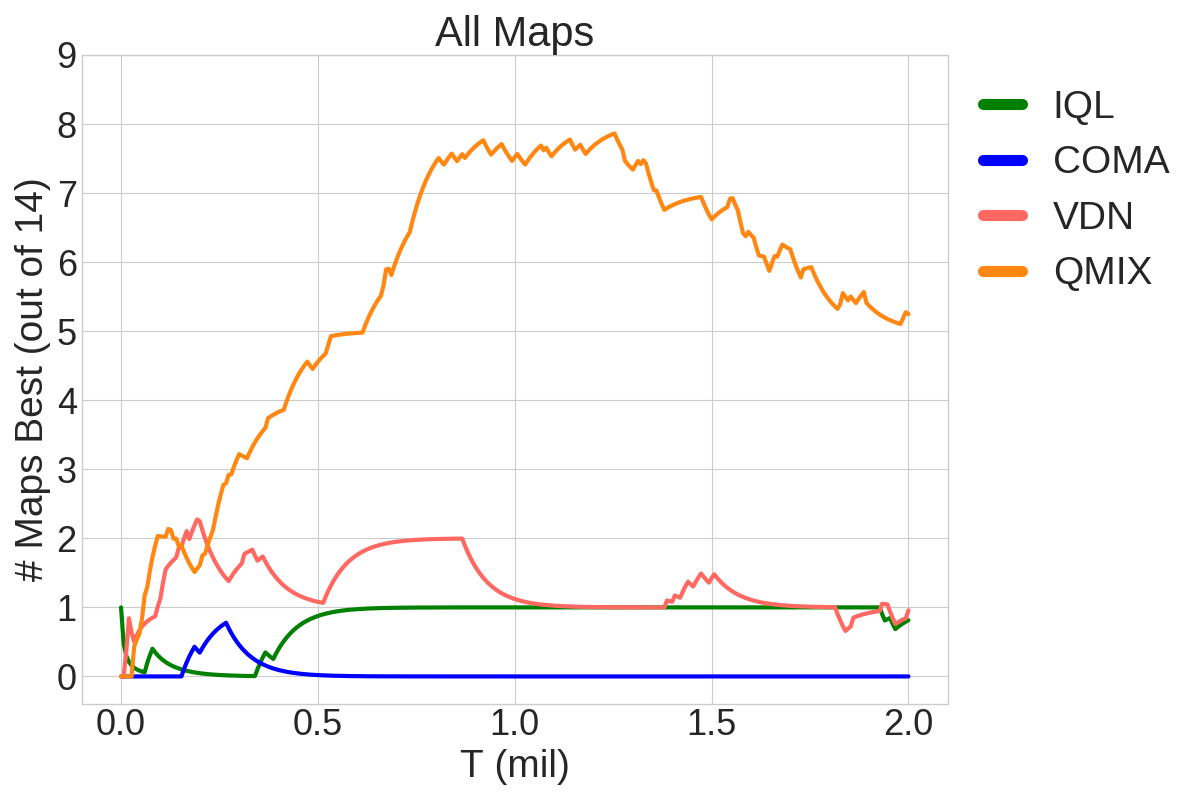}
	\caption{Left: The median test win \%, averaged across all 14 scenarios. Heuristic's performance is shown as a dotted line. Right: The number of scenarios in which the algorithm's median test win \% is the highest by at least $1/32$ (smoothed).}
	\label{fig:median_test_win}
\end{figure*}

Figure \ref{fig:median_test_win} plots the median test win percentage averaged across all scenarios in order to compare the algorithms across the entire SMAC suite. 
We also plot the performance of a simple heuristic AI that selects the closest enemy unit (ignoring partial observability) and attacks it with the entire team until it is dead, upon which the next closest enemy unit is selected. This is a basic form of \textit{focus-firing}, which is a crucial tactic for achieving good performance in micromanagement scenarios.
The relatively poor performance of the heuristic AI shows that the suite of SMAC scenarios requires more complex behaviour than naively focus-firing the closest enemy, making it an interesting and challenging benchmark.

Overall QMIX achieves the highest test win percentage and is the best performer on up to eight scenarios during training. 
Additionally, IQL, VDN, and QMIX all significantly outperform COMA, demonstrating the sample efficiency of off-policy value-based methods over on-policy policy gradient methods. 

\begin{figure*}[h!]
	\centering
	\includegraphics[width=0.45\textwidth]{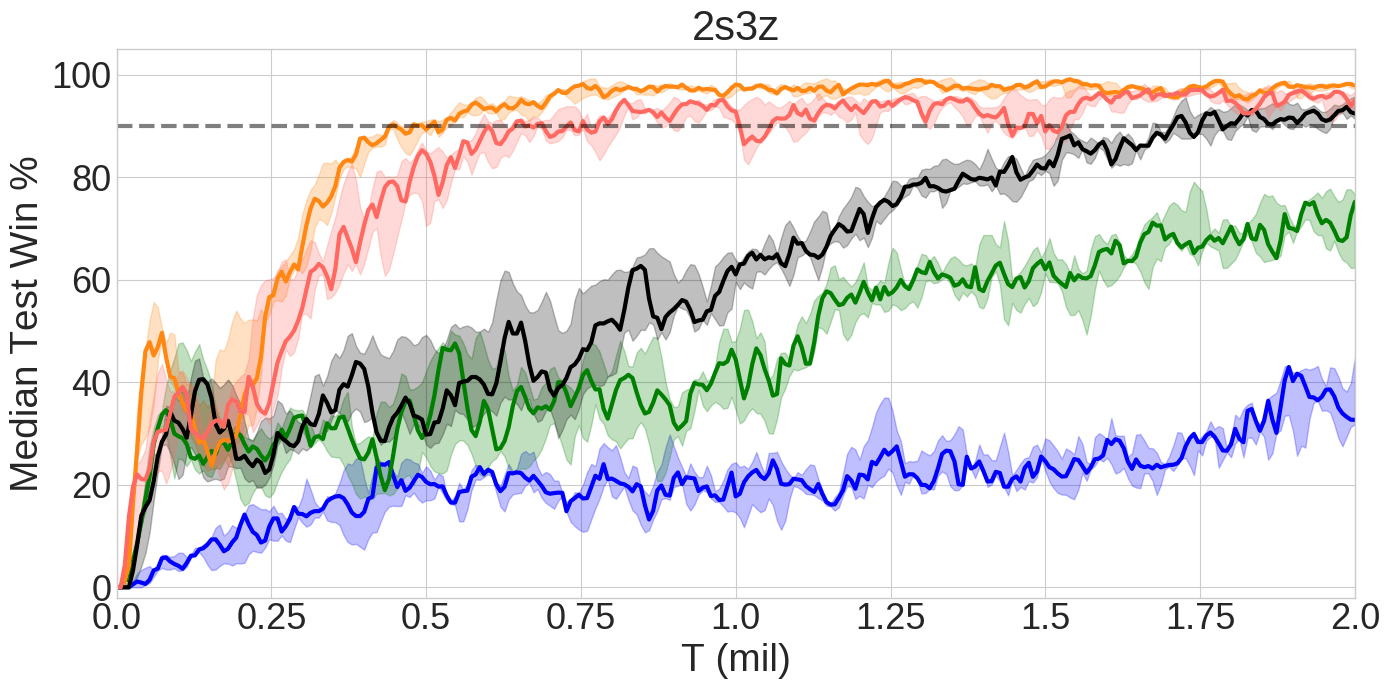}
	\includegraphics[width=0.45\textwidth]{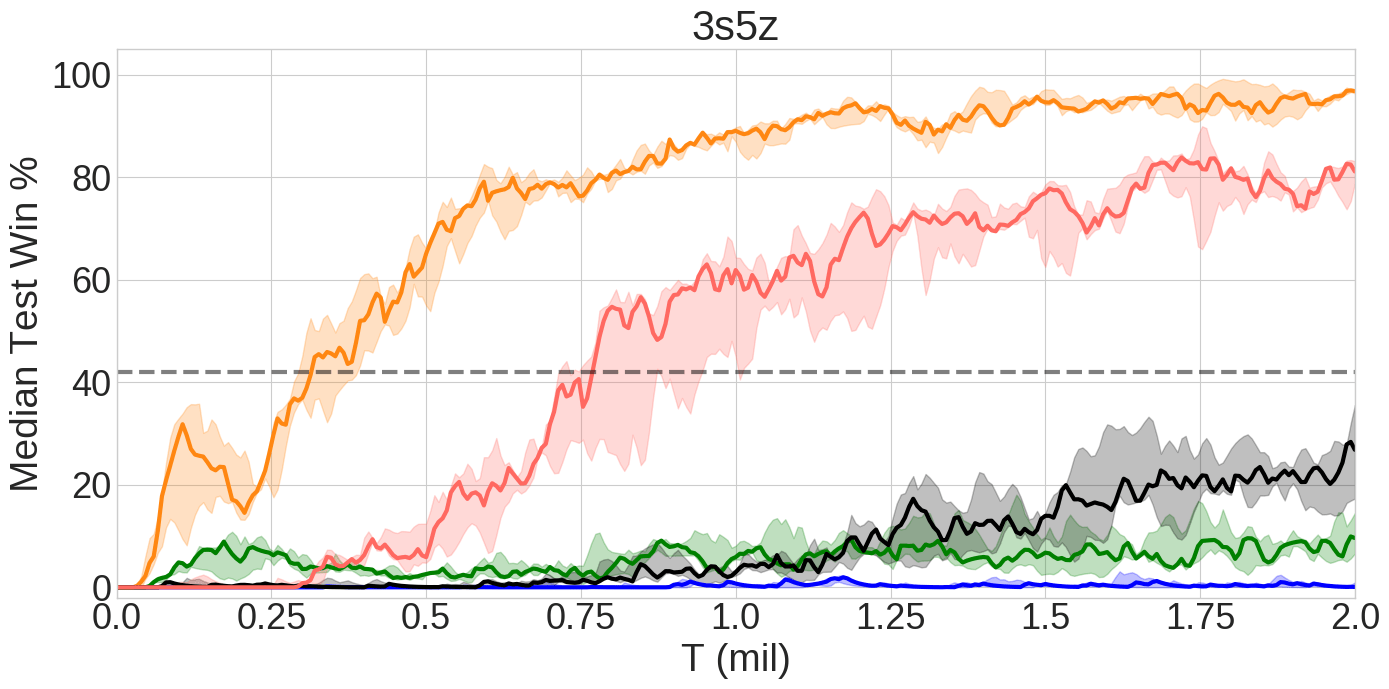}
	\includegraphics[width=0.45\textwidth]{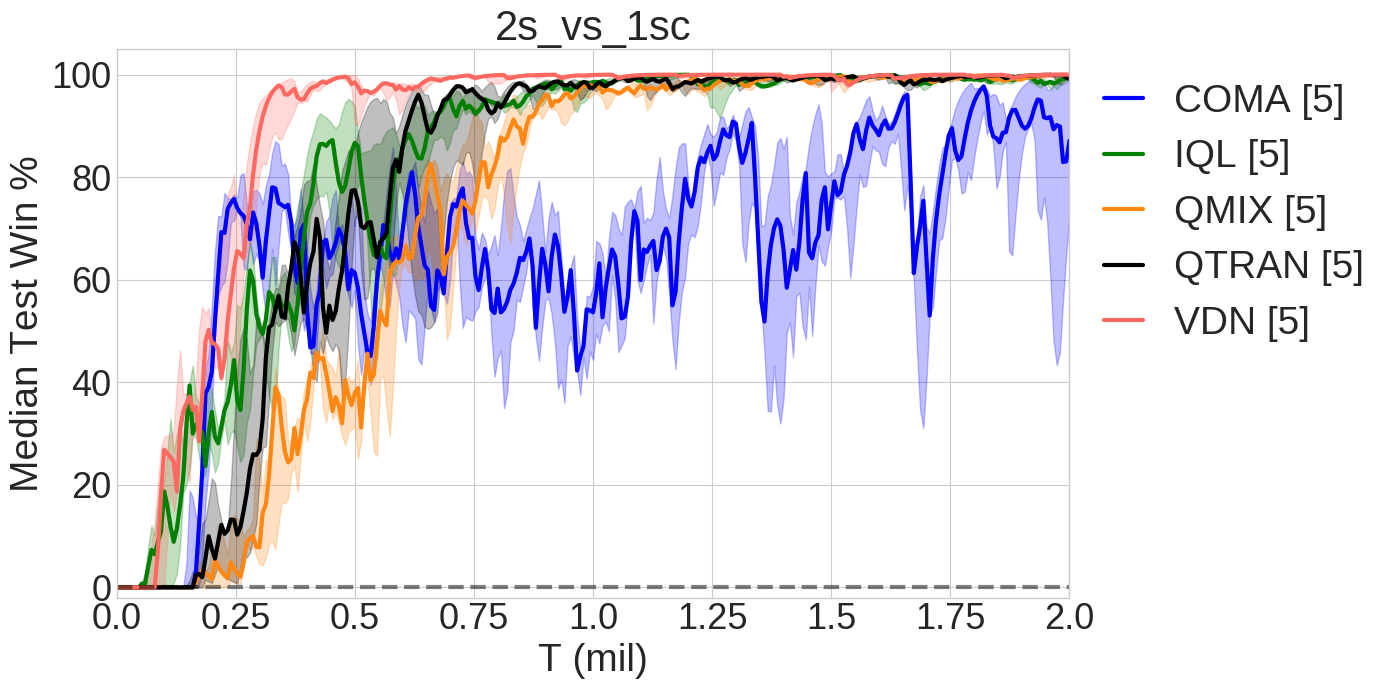}
	\caption{Three scenarios including QTRAN.}
	\label{fig:qtran_results}
\end{figure*}

We also compare to QTRAN on 3 scenarios in Figure \ref{fig:qtran_results}. 
We can see that QTRAN fails to achieve good performance on \textit{3s5z} and takes far longer to reach the performance of VDN and QMIX on \textit{2s3z}. 
Since it barely beats IQL on relatively easy scenarios, we do not perform a more comprehensive benchmarking of QTRAN.
In preliminary experiments, we found the QTRAN-Base algorithm slightly more performant and more stable than QTRAN-Alt. 
For more details on the hyperparameters and architectures considered, please see the Appendix \ref{sec:smac_setup}. 

Based on the overall performances of all algorithms, we broadly group the scenarios into three categories: 
\textit{Easy},
\textit{Hard}, and
\textit{Super-Hard}.

\begin{figure*}[h!]
	\centering
	\includegraphics[width=0.45\textwidth]{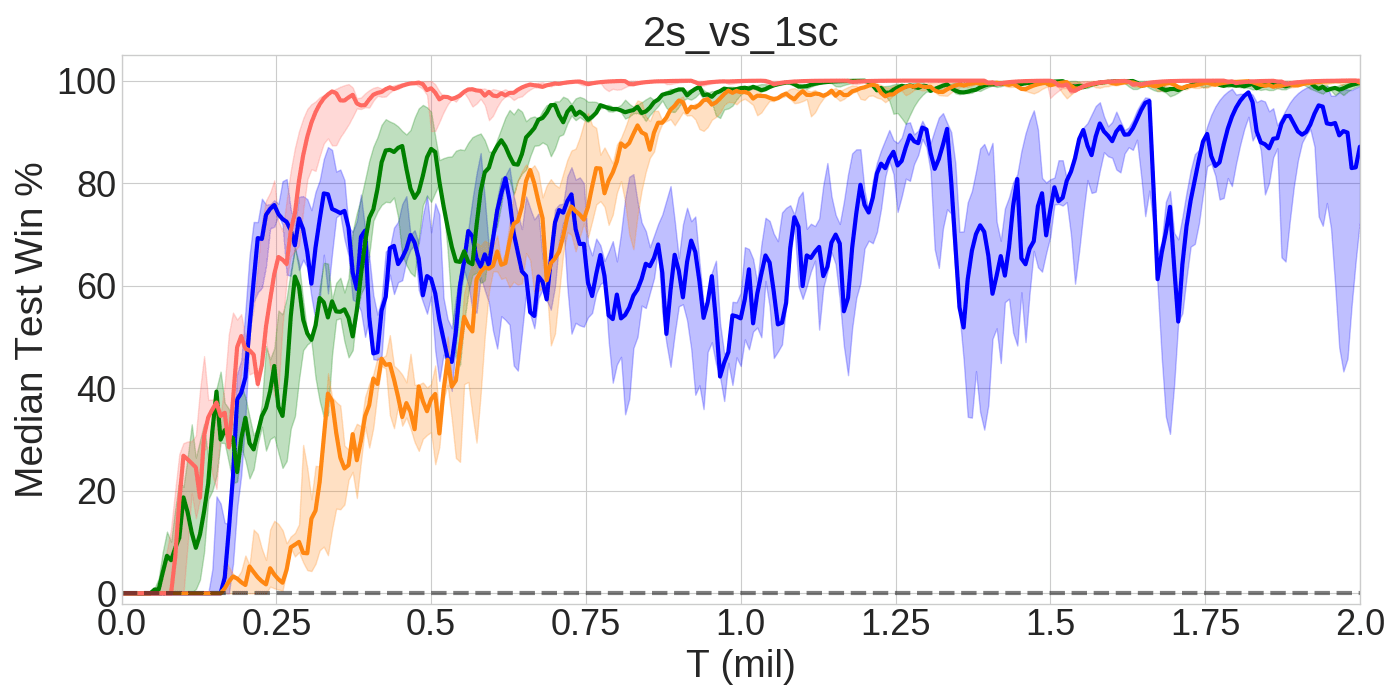}
	\includegraphics[width=0.45\textwidth]{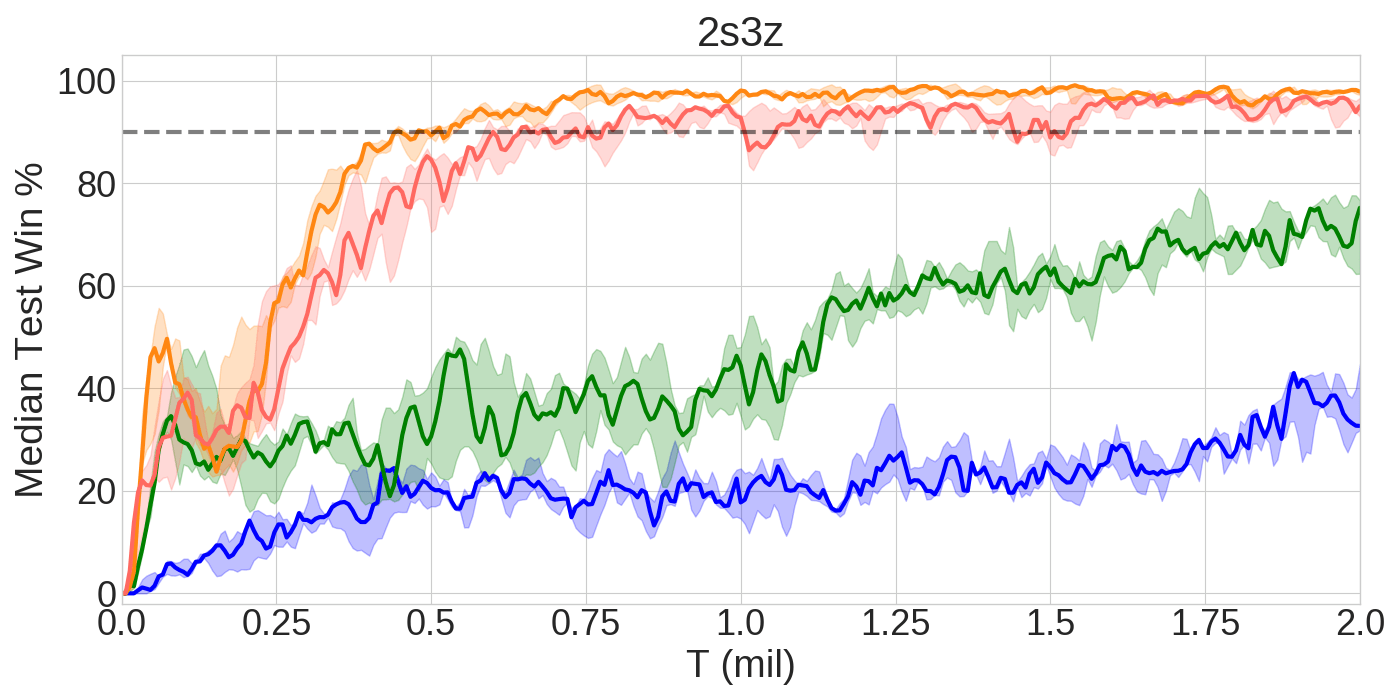}
	\includegraphics[width=0.45\textwidth]{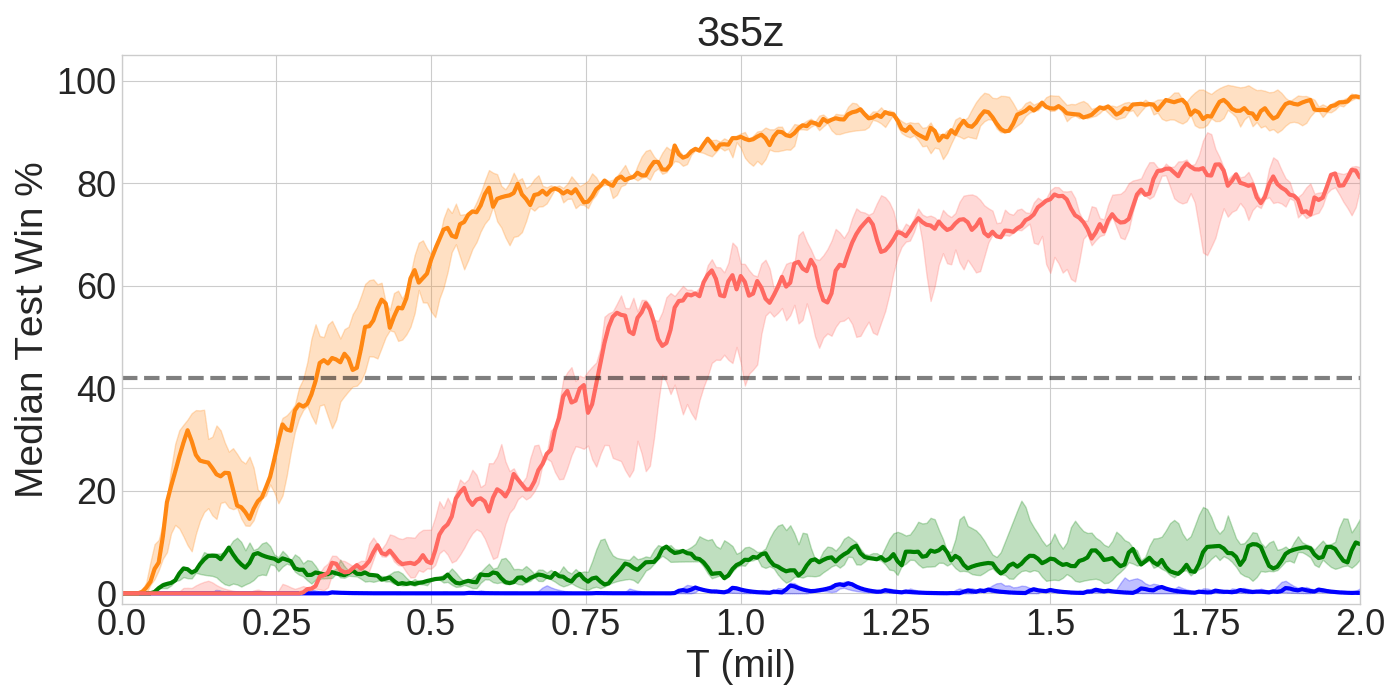}
	\includegraphics[width=0.45\textwidth]{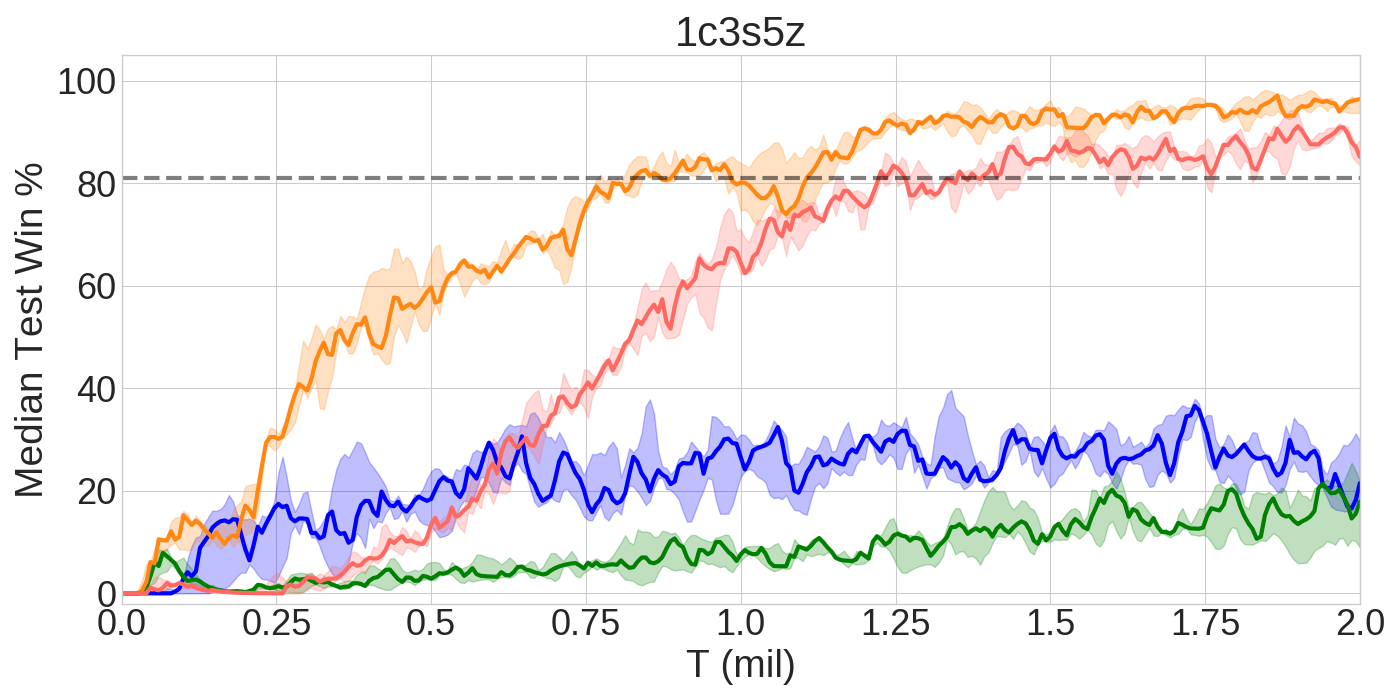}
	\includegraphics[width=0.45\textwidth]{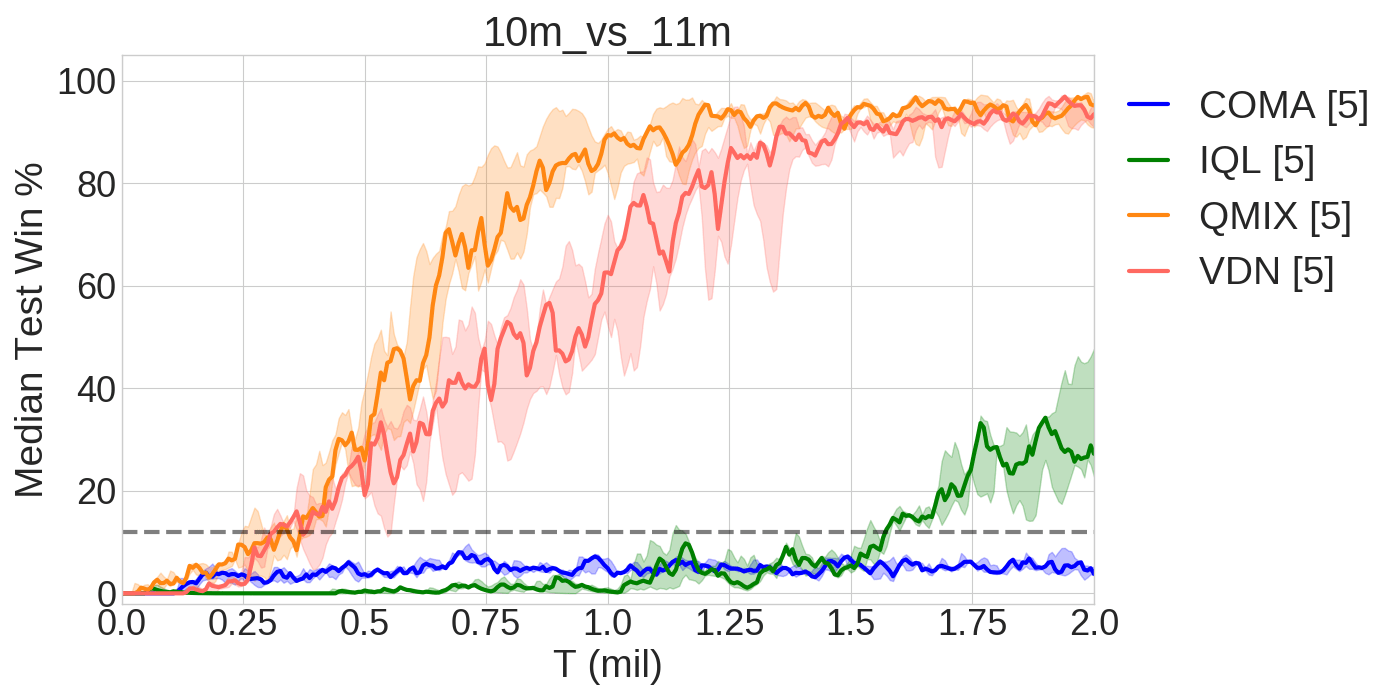}
	\caption{Easy scenarios. The heuristic AI's performance shown as a dotted black line. 
	}
	\label{fig:easy_map_results}
\end{figure*}

Figure \ref{fig:easy_map_results} shows that IQL and COMA struggle even on the \emph{Easy} scenarios, performing poorly on four of the five scenarios in this category. This shows the advantage of learning a centralised but factored centralised $Q_{tot}$. 
Even though QMIX exceeds $95\%$ test win rate on all of five \emph{Easy} scenarios, they serve an important role in the benchmark as sanity checks when implementing and testing new algorithms. 

\begin{figure*}[h!]
	\centering
	\includegraphics[width=0.45\textwidth]{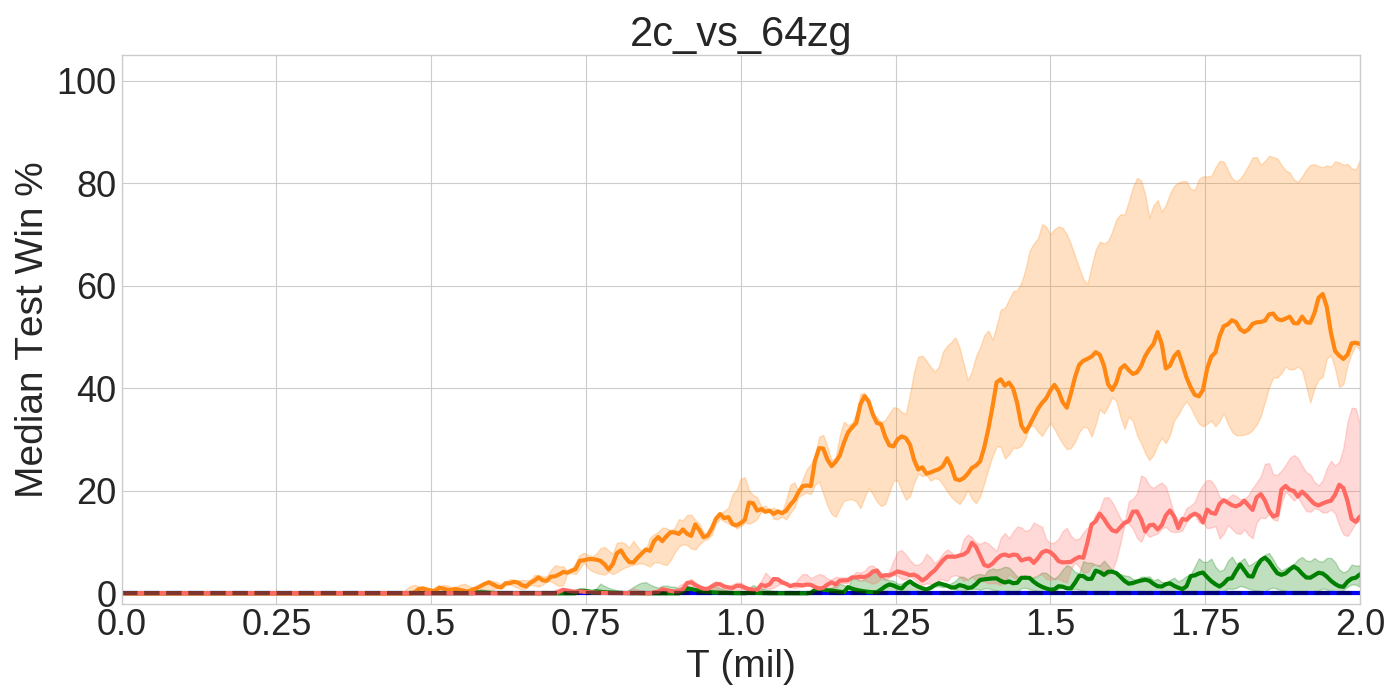}
	\includegraphics[width=0.45\textwidth]{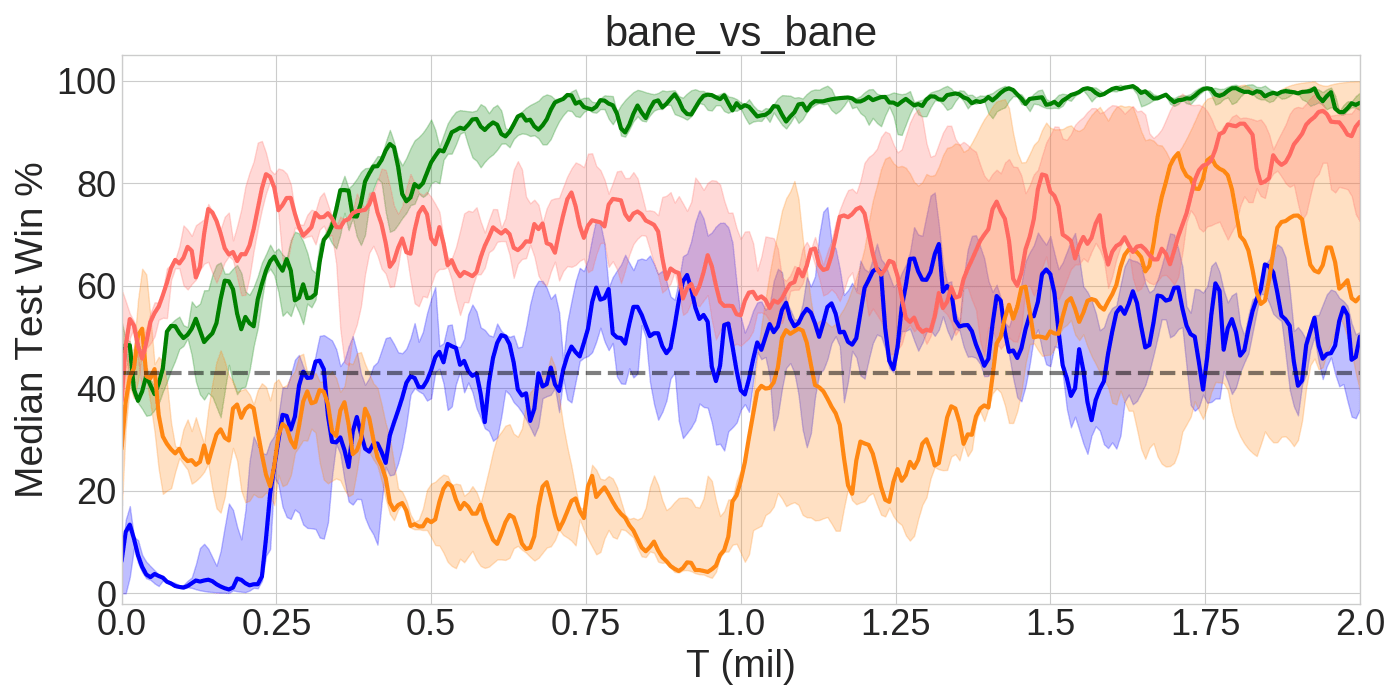}
	\\
	\includegraphics[width=0.45\textwidth]{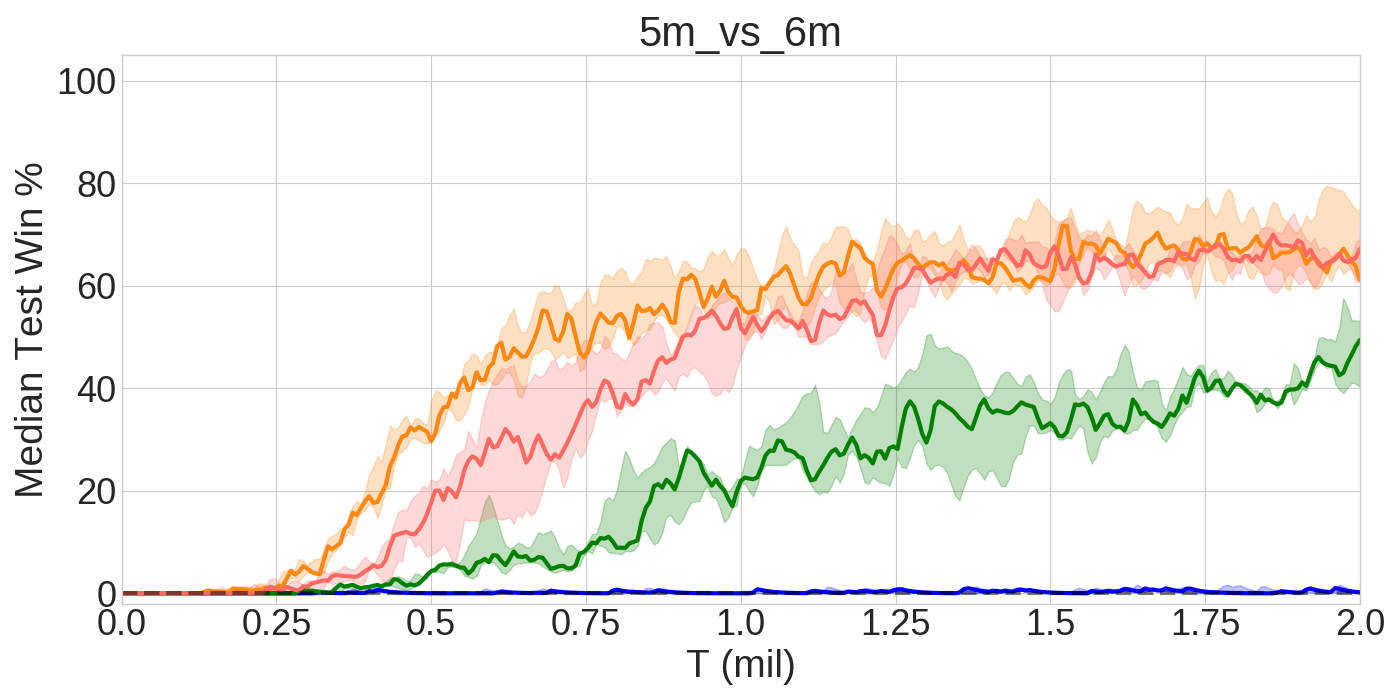}
	\includegraphics[width=0.45\textwidth]{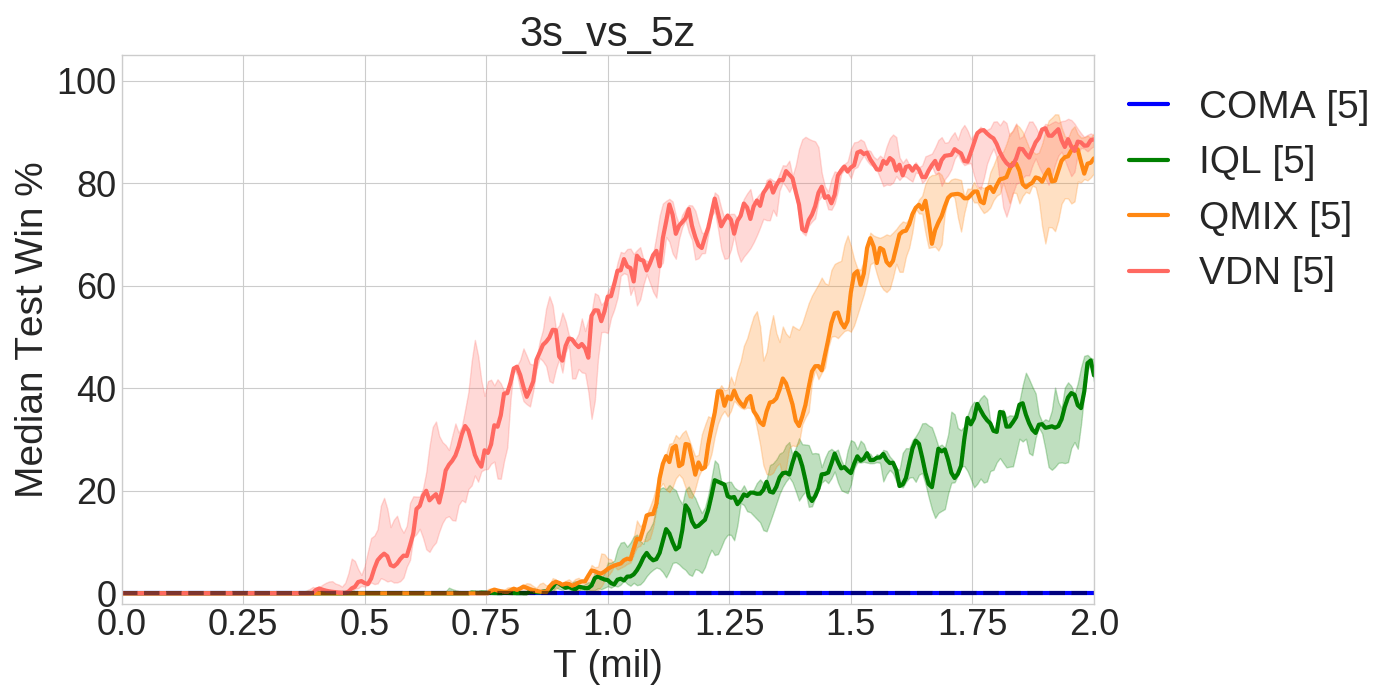}
	\caption{Hard scenarios. The heuristic AI's performance shown as a dotted black line.}
	\label{fig:hard_map_results}
\end{figure*}

The \textit{Hard} scenarios in Figure \ref{fig:hard_map_results} each present their own unique problems.
\textit{2c\_vs\_64zg} only contains 2 allied agents, but 64 enemy units (the largest in the SMAC benchmark) making the action space of the agents much larger than the other scenarios.
\textit{bane\_vs\_bane} contains a large number of allied and enemy units, but the results show that IQL easily finds a winning strategy whereas all other methods struggle and exhibit large variance.
\textit{5m\_vs\_6m} is an asymmetric scenario that requires precise control to win consistently, and in which the best performers (QMIX and VDN) have plateaued in performance.
Finally, \textit{3s\_vs\_5z} requires the three allied stalkers to \textit{kite} the enemy zealots for the majority of the episode (at least 100 timesteps), which leads to a delayed reward problem.

\begin{figure*}[h!]
	\centering
	\includegraphics[width=0.45\textwidth]{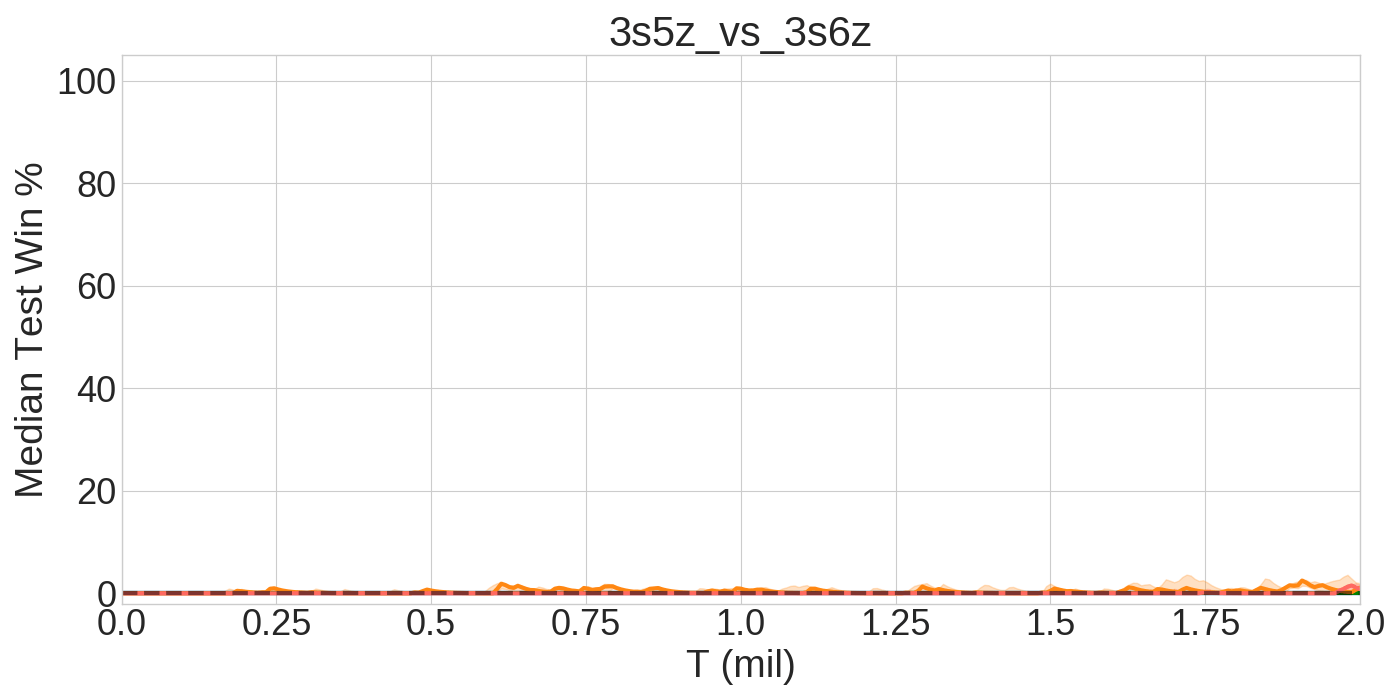}
	\includegraphics[width=0.45\textwidth]{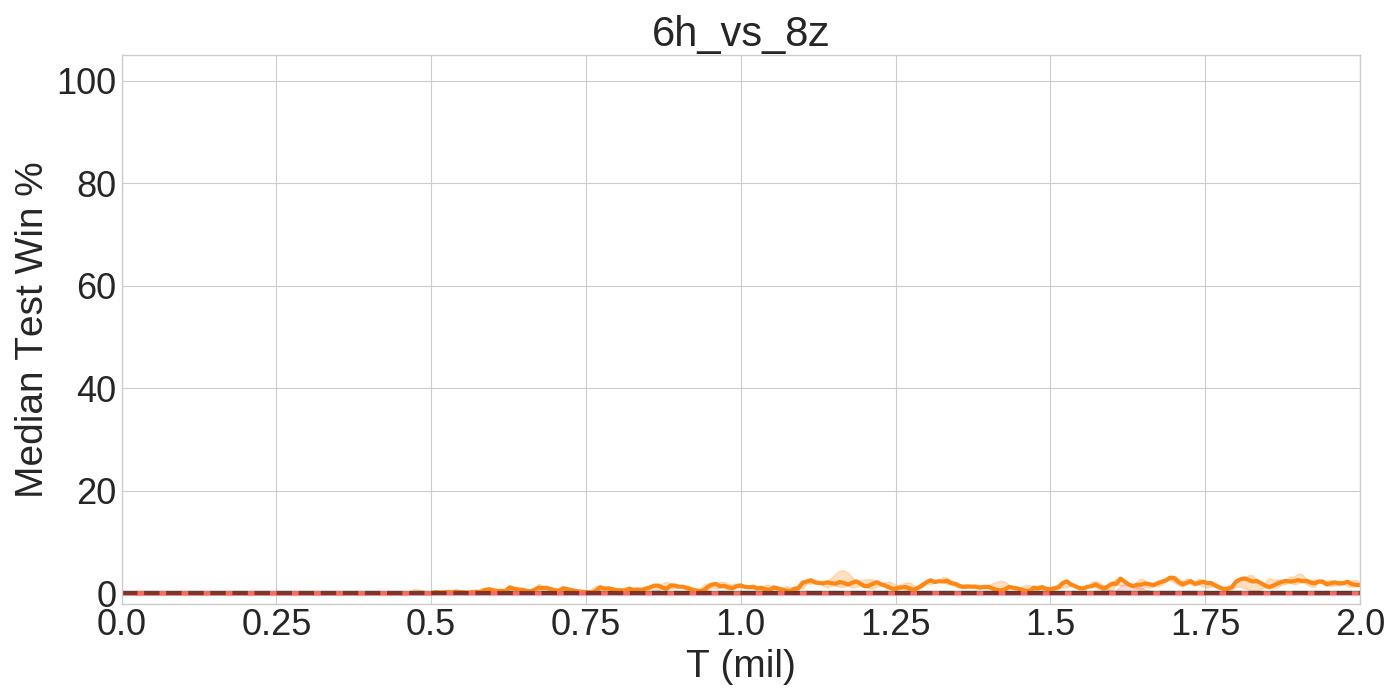}
	\includegraphics[width=0.45\textwidth]{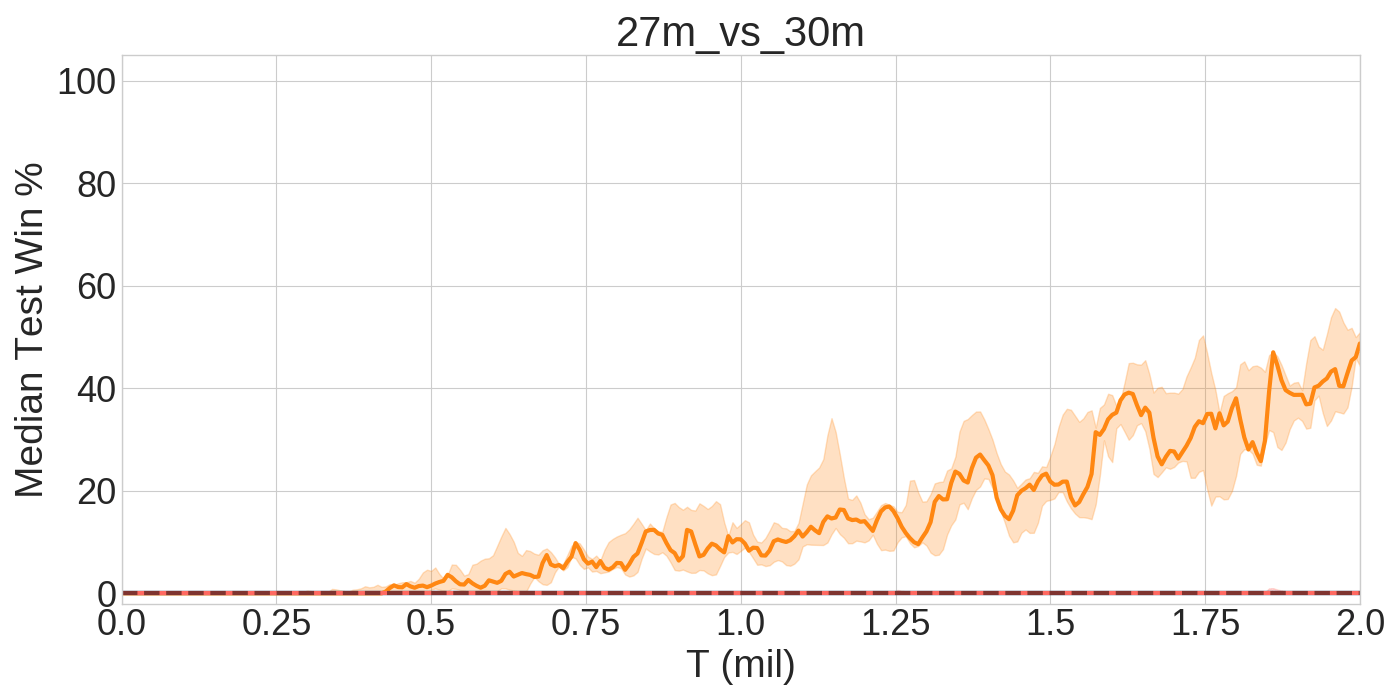}
	\includegraphics[width=0.45\textwidth]{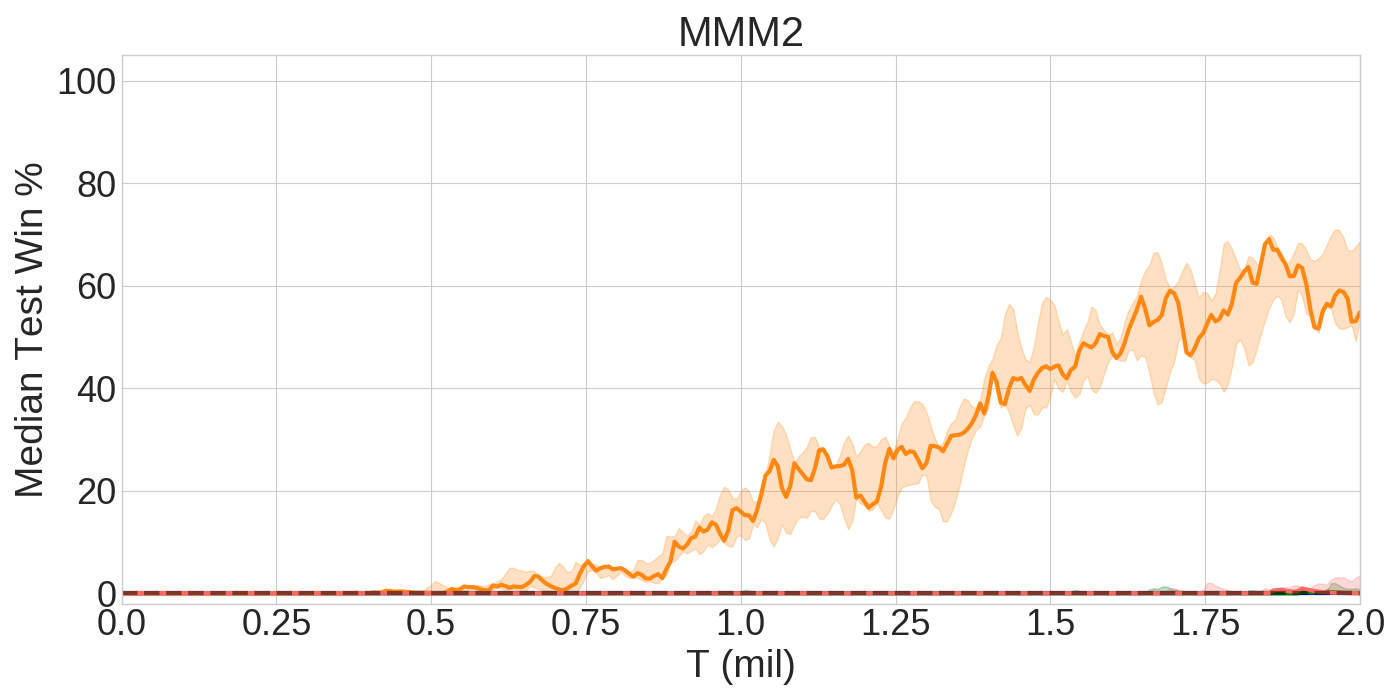}
	\includegraphics[width=0.45\textwidth]{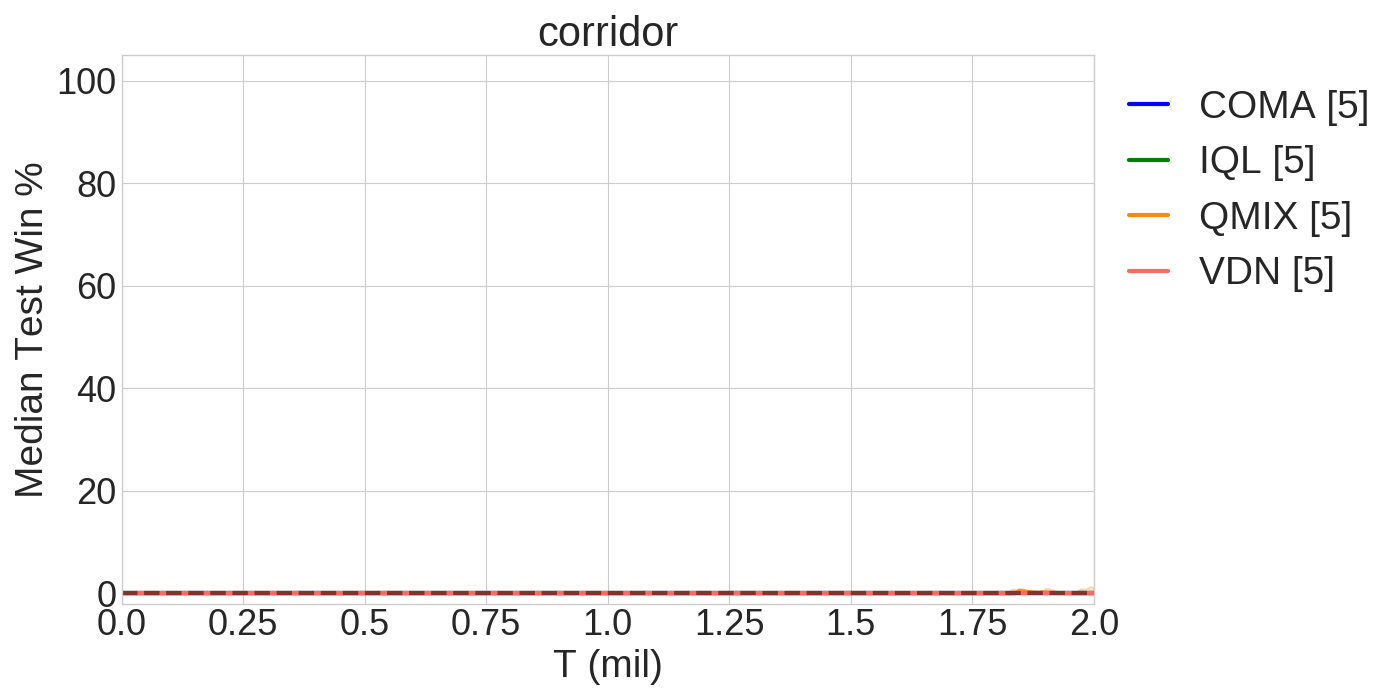}
	\caption{Super Hard scenarios. The heuristic AI's performance shown as a dotted black line.}
	\label{fig:super_hard_map_results}
\end{figure*}

The scenarios shown in Figure \ref{fig:super_hard_map_results} are categorised as \textit{Super Hard} because of the poor performance of all algorithms, with
only QMIX  making meaningful progress on two of the five.
We hypothesise that exploration is a bottleneck in many of these scenarios, providing a nice testbed for future research in this domain.

\section{Analysis}
\label{sec:analysis}

In this section, we aim to understand the reasons why QMIX outperforms VDN in 
our experiments. In particular, we consider two possible causes: 1) the 
inclusion of the central state in the mixing network and 2) the nonlinearity of 
the mixing network.

We find that inclusion of the central state is extremely important for 
performance but that it alone is not sufficient to explain QMIX's increased 
performance.
We further find that the mixing network learns nonlinear functions when 
required, but achieves good performance on the SMAC benchmark even when the 
learned solutions are mostly linear combinations of the agent's utilities.
Studying this phenomenon further, we conclude that it is the flexibility of the 
mixing network's parameterisation, but not its nonlinearity, that accounts for 
QMIX's superior performance in these tasks.

Our ablation experiments focus on three scenarios, one for each of the three difficulty categories: \texttt{3s5z}, \texttt{2c\_vs\_64zg}, and \texttt{MMM2}.
We chose these three scenarios because they exhibit a large gap in performance between QMIX and VDN, in order to better understand the contribution of the components in QMIX towards that increased performance.

\subsection{Role of the Central State}

In order to disentangle the effects of the various components of the QMIX architecture, we consider the following two ablations:

\begin{itemize}
    \item \textbf{VDN-S:} The output of VDN is augmented by a state-dependent bias.
    \item \textbf{QMIX-NS:} The weights of the mixing network are not a function of the state. The state-dependent bias remains.
\end{itemize}

\begin{figure*}[h!]
    \centering
    \includegraphics[width=0.45\textwidth]{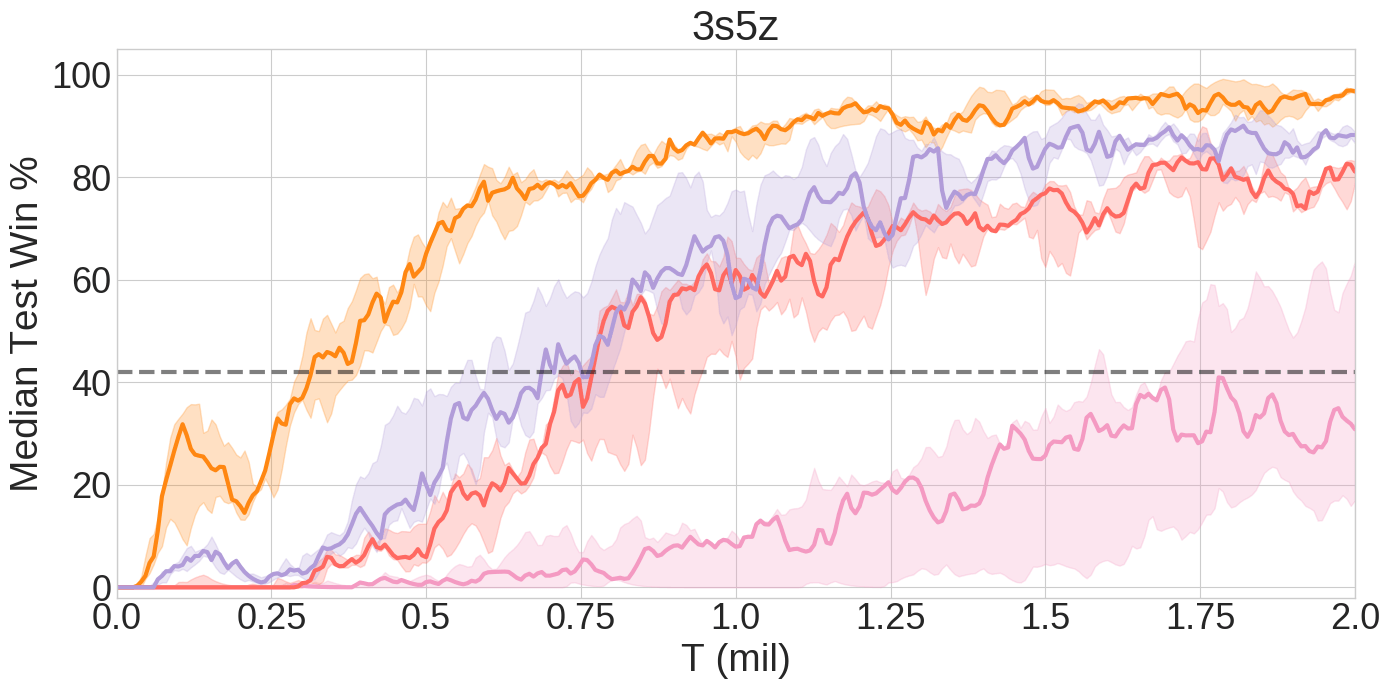}
    \includegraphics[width=0.45\textwidth]{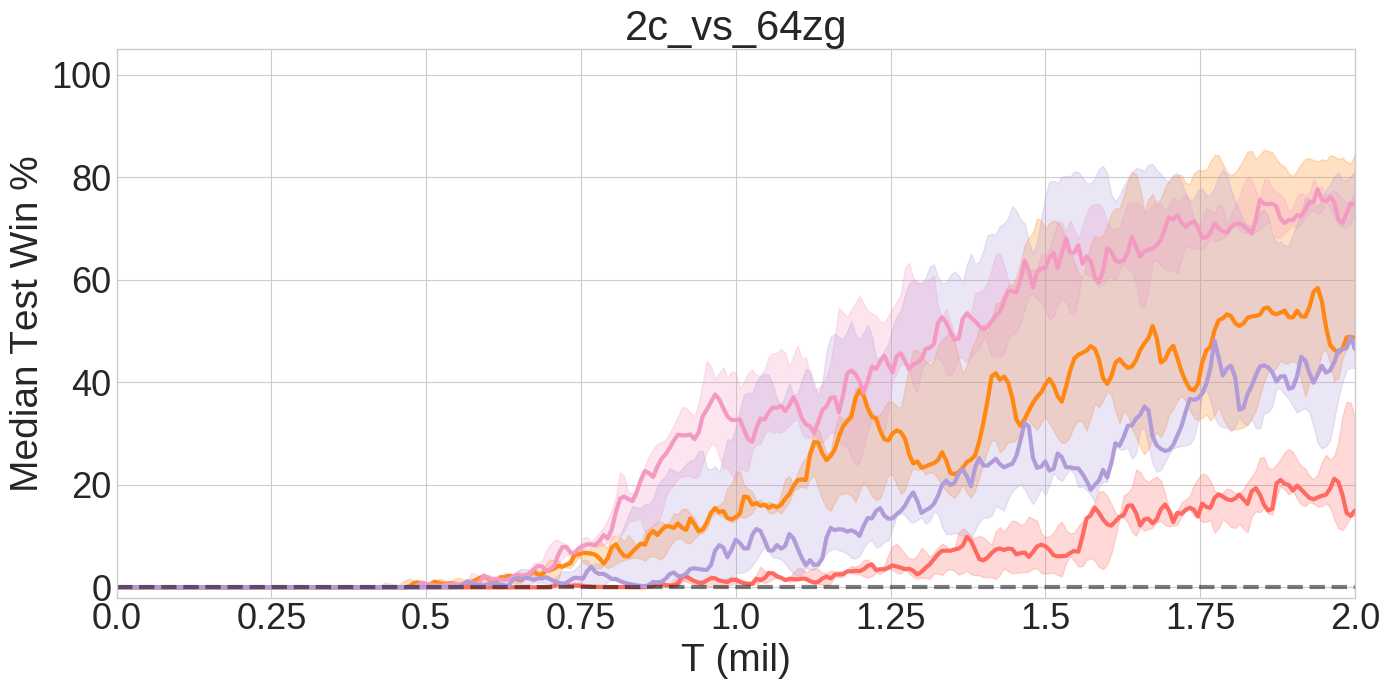}
    \includegraphics[width=0.45\textwidth]{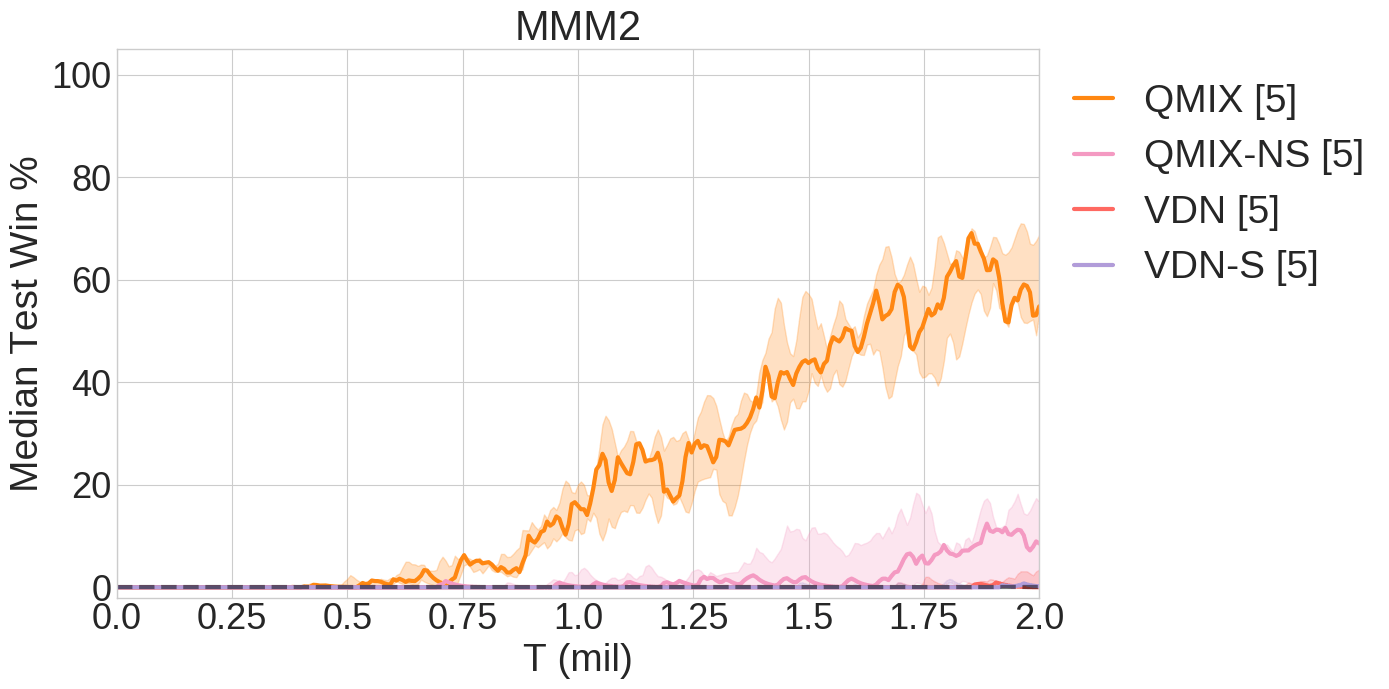}
    \caption{Ablations for state experiments.}
    \label{fig:state_ablations}
\end{figure*}

Figure \ref{fig:state_ablations} shows that VDN-S performs better than VDN, but still worse than QMIX, indicating that inclusion of the state accounts for some but not all of QMIX's performance. 
On the easy \textit{3s5z} scenario, QMIX-NS performs very poorly, which is surprising since it is a strict generalisation of VDN-S. 
However, on the other two scenarios it outperforms VDN-S, indicating that the extra flexibility afforded by a learnt mixing over a fixed summation can be beneficial.

\subsection{Role of Nonlinear Mixing}

The results in Figure \ref{fig:state_ablations} show that both a state dependent bias on VDN (VDN-S) and a learnt state-independent mixing (QMIX-NS) do not match the performance of QMIX. 
Thus, we consider another ablation of QMIX:

\begin{itemize}
    \item \textbf{QMIX-Lin:} We remove the final layer and nonlinearity of the mixing network, making it a linear network. The state-dependant bias remains.
\end{itemize}

\begin{figure*}[h!]
    \centering
    \includegraphics[width=0.45\textwidth]{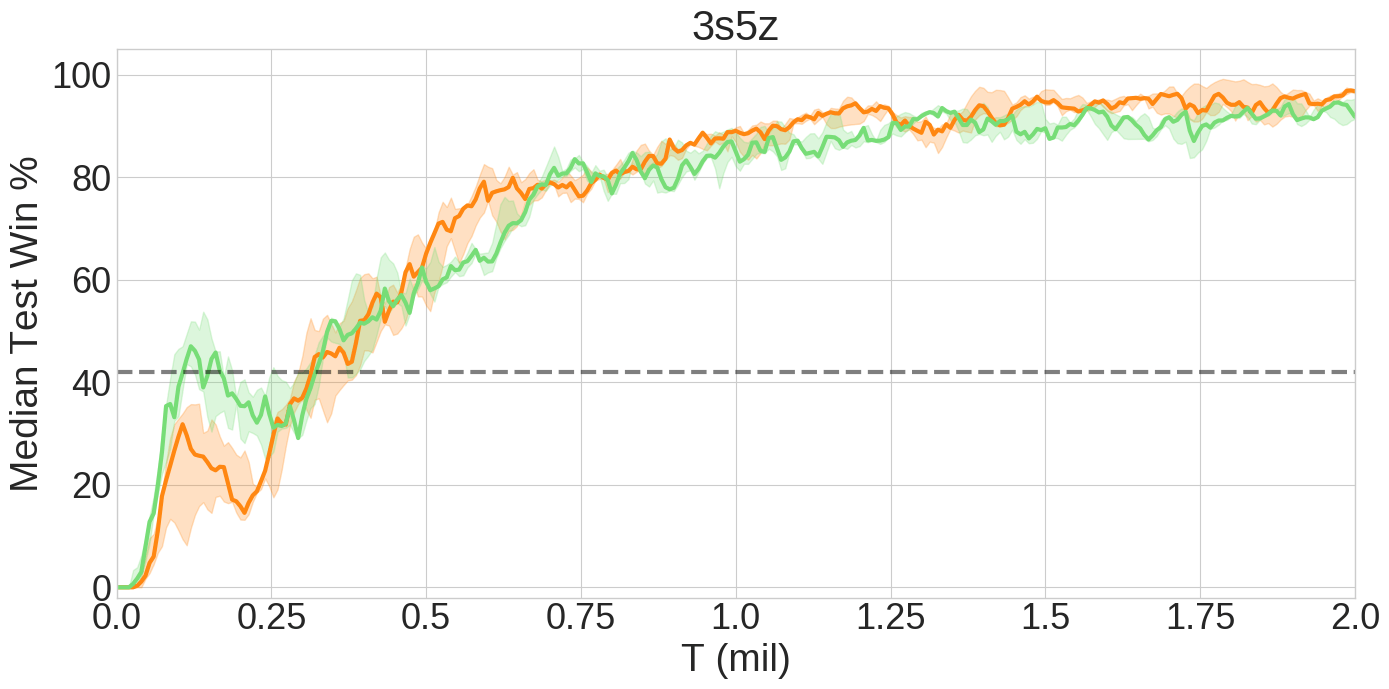}
    \includegraphics[width=0.45\textwidth]{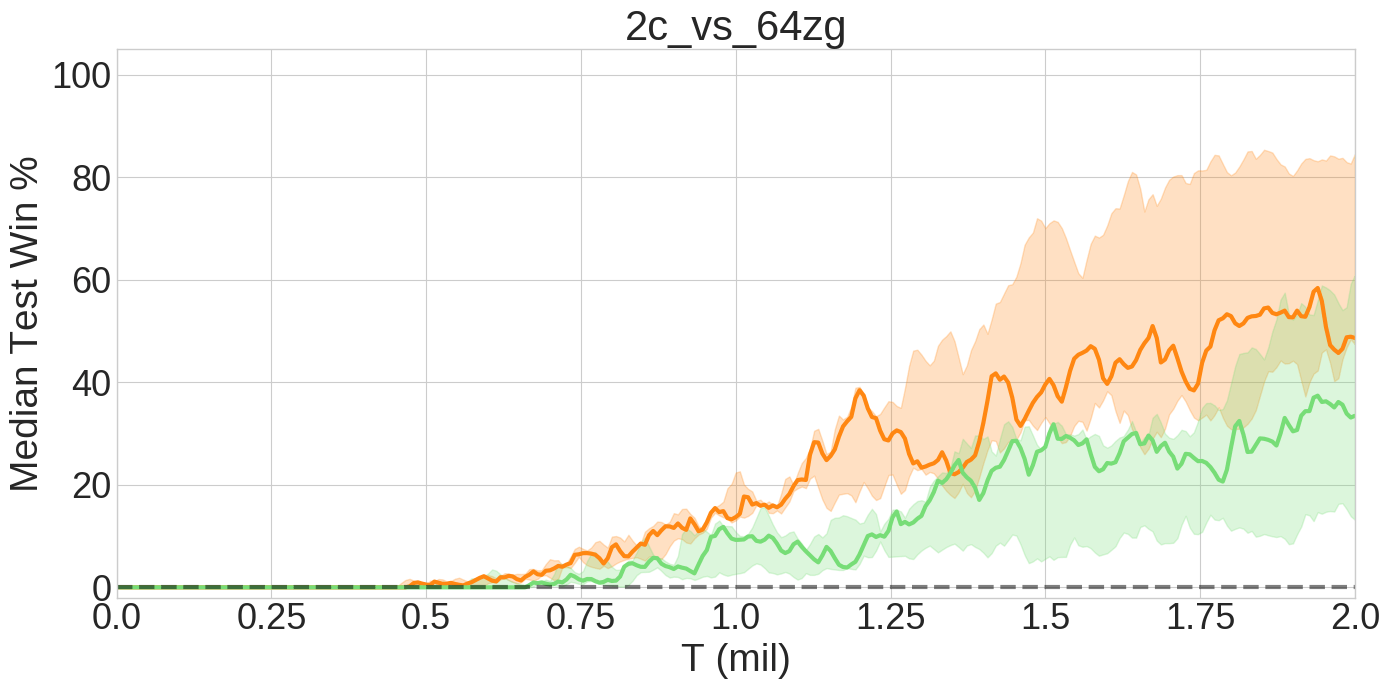}
    \includegraphics[width=0.45\textwidth]{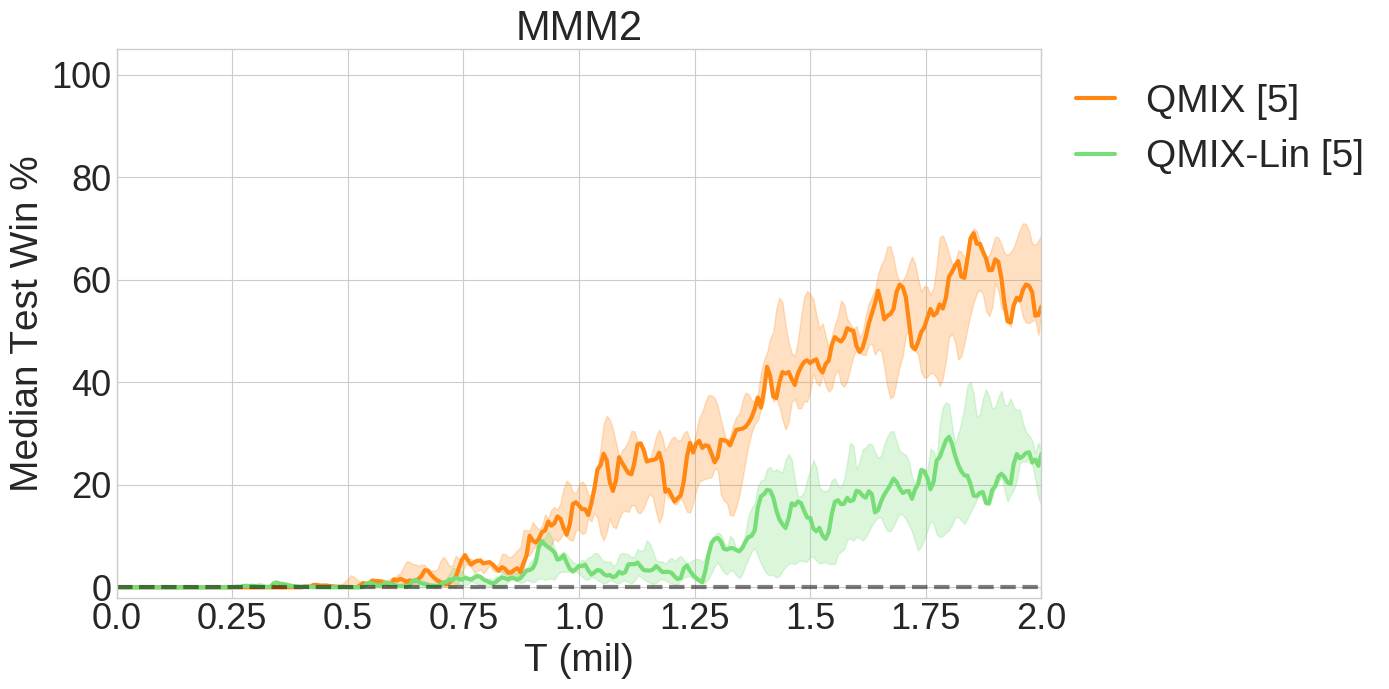}
    \caption{Ablations for mixing network linearity experiments.}
    \label{fig:state_extra_ablations}
\end{figure*}

Figure \ref{fig:state_extra_ablations} shows that whilst QMIX-Lin matches QMIX 
on the easy \textit{3s5z} scenario, it underperforms QMIX on the other two 
scenarios we considered.
This suggests that the ability to perform a nonlinear mixing, through the inclusion of a nonlinearity and additional hidden layer in the mixing network, plays an important role QMIX's performance.
To confirm this, we examine the function the mixing network learns.

\begin{figure*}[h!]
    \centering
    \includegraphics[width=0.45\textwidth]{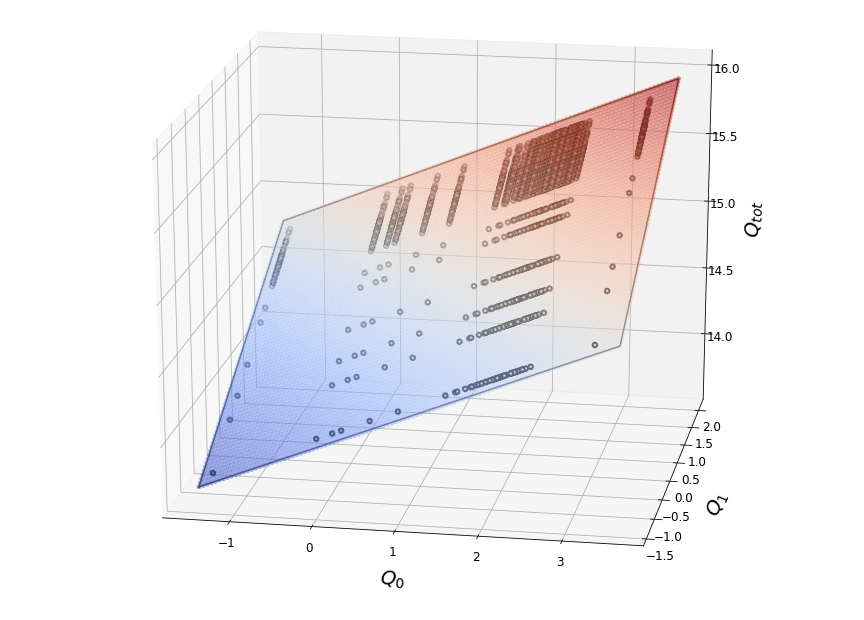}
    \includegraphics[width=0.45\textwidth]{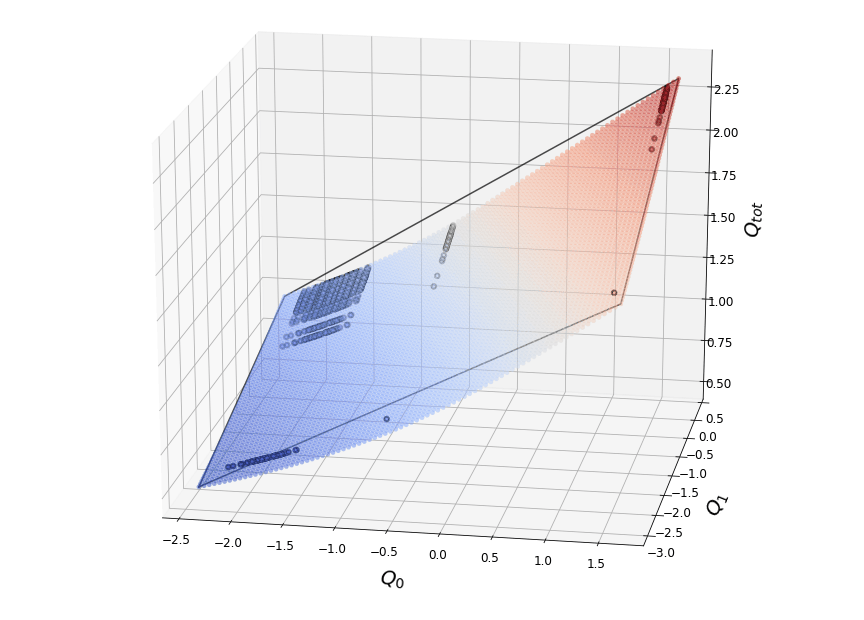}
    \caption{The learnt mixing of QMIX on \textit{2c\_vs\_64zg} at the end of training for timesteps 0 (left) and 50 (right). Circles indicate the $Q_{tot}$-values for the discrete joint-action space.}
    \label{fig:2c_mixing_elu}
\end{figure*}

Figure \ref{fig:2c_mixing_elu} shows the learnt mixing of the network at the end of training for the \textit{2c\_vs\_64zg} scenario with two agents trained using QMIX.
We generated a sample trajectory by greedily selecting actions and plotted the mixing function at different timesteps.
The learnt function is approximately linear for the vast majority of the timesteps 
in the agent's $Q$-values, similarly to the first timestep of the episode shown in Figure \ref{fig:2c_mixing_elu} (left).
The only timesteps in which we observe a noticeable
nonlinearity are at the end of the episode (53 timesteps total). 
When trained for longer we still observe that the the mixing function remains linear for the vast majority of timesteps. 

\begin{figure*}[h!]
    \centering
    \includegraphics[width=0.35\textwidth]{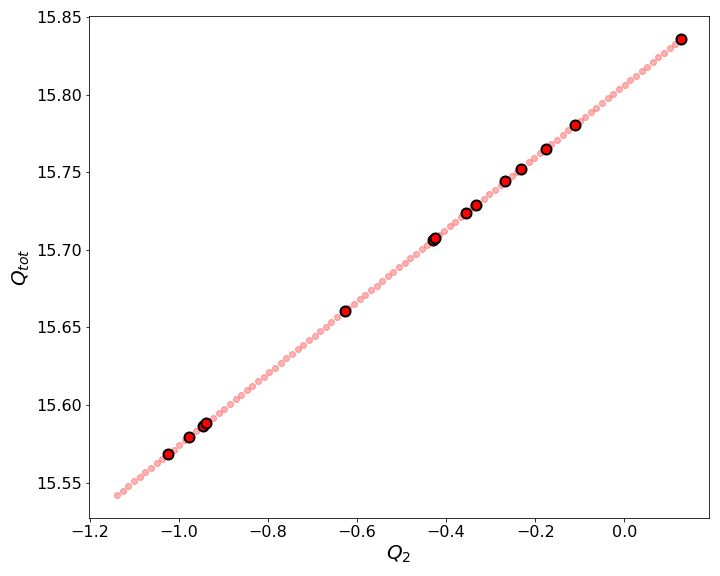}
    \includegraphics[width=0.35\textwidth]{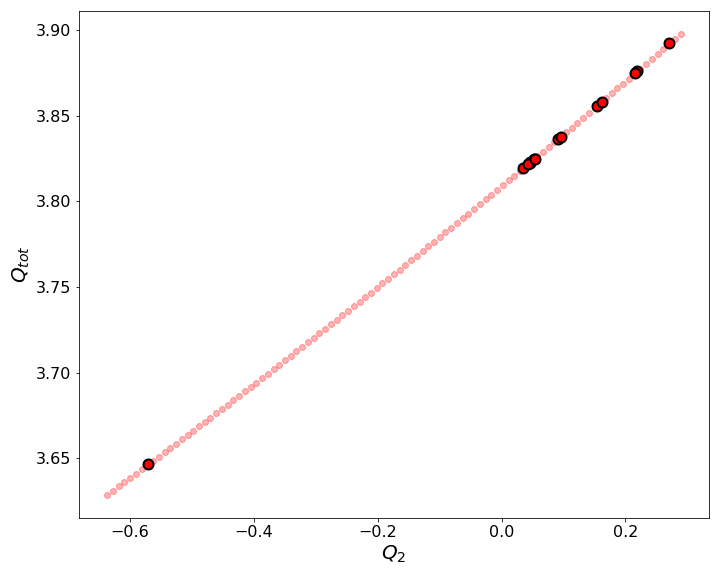}
    \includegraphics[width=0.35\textwidth]{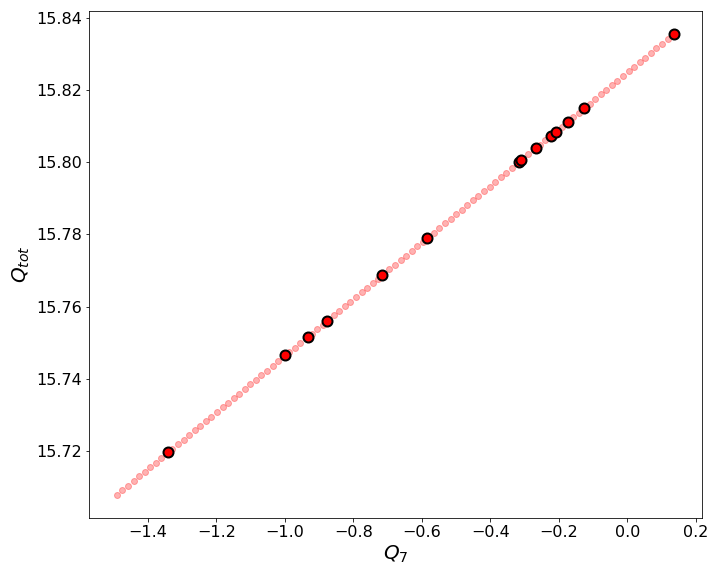}
    \includegraphics[width=0.35\textwidth]{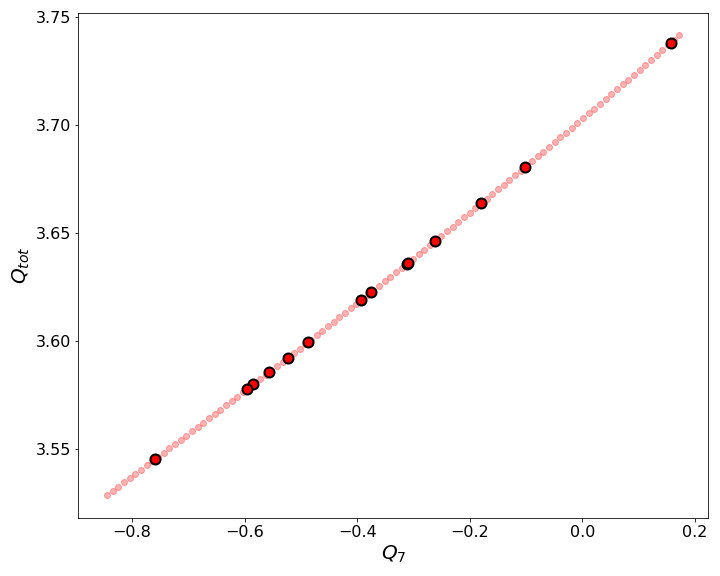}
    \caption{The learnt mixing of QMIX on \textit{3s5z} at the end of training for timesteps 0 (left) and 50 (right), for agents 2 (top) and 7 (bottom).}
    \label{fig:3s5z_mixing_elu}
\end{figure*}

Figure \ref{fig:3s5z_mixing_elu} shows the learnt mixing for \textit{3s5z} as a function of a single $Q_a$, keeping the $Q$-values for the other agent's fixed.
The mixing network clearly learns an approximately linear mixing, which is 
consistent across all timesteps and agents (only timesteps 0 and 50 for agents 
2 and 7 are shown since they are representative of all other timesteps and 
agents).

This is a surprising result. QMIX is certainly capable of learning nonlinear 
mixing functions in practice, and thereby outperforming the more restrictive 
VDN decomposition.
As an example, Figure \ref{fig:two_step_mixing} visualises the learnt mixing 
function for the didactic two-step game from section \ref{sec:two_step_game}, 
clearly showing how a nonlinear fit allows QMIX to represent the true optimal 
value function.

\begin{figure*}[h!]
	\centering
	\includegraphics[width=0.6\textwidth]{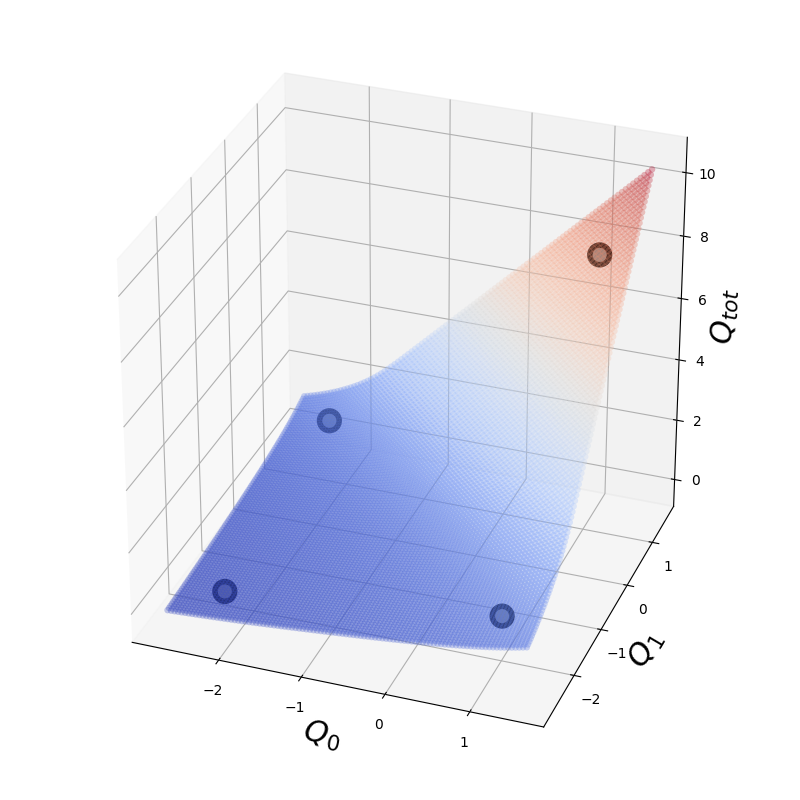}
	\caption{The learnt mixing of QMIX on the two-step game from Section 
	\ref{sec:two_step_game}.}
	\label{fig:two_step_mixing}
\end{figure*}

Yet, we observe mostly linear mixing functions in solutions to the SMAC 
benchmark tasks.
If QMIX is learning a mostly linear mixing, then why does it outperform QMIX-Lin?
One difference is that QMIX's mixing network is compromised of two 
fully-connected layers, rather than the single linear layer of QMIX-Lin.
If the nonlinearity is not essential, it may be that the extra layer in the 
mixing network is responsible for the performance increase.
To test this hypothesis, we consider another ablation:

\begin{itemize}
    \item \textbf{QMIX-2Lin:} QMIX, without a nonlinear activation function in the mixing network.
\end{itemize}

QMIX-2Lin's mixing network is compromised of two linear layers without any non-linearity between them. Thus, it can only learn a linear transformation of $Q_a$ into $Q_{tot}$. 
QMIX-Lin's mixing network is a single linear layer, thus both QMIX-Lin and QMIX-2Lin's mixing networks share the same representational capacity.
QMIX's mixing network, on the other hand, has an ELU non-linearity between its 
two weight layers and thus can represent non-linear transformations of $Q_a$ to 
$Q_{tot}$. 

We first consider a simple regression experiment in which the states, agent $Q$-values, and targets are randomly generated and fixed. 
Further details on the regression task are included in Appendix \ref{sec:regression_exp}.

\begin{figure*}[h!]
    \centering
    \includegraphics[width=0.6\textwidth]{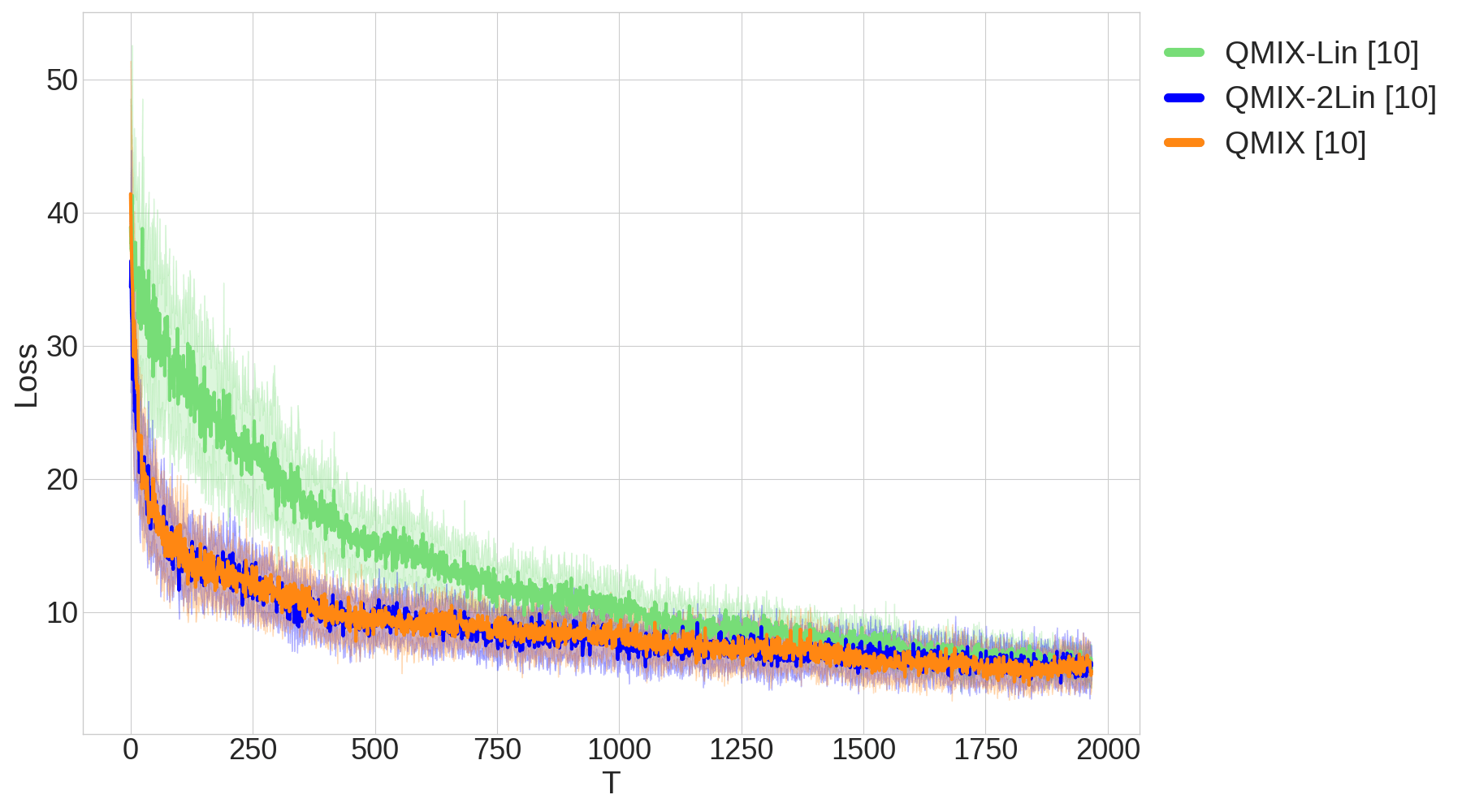}
    \caption{The loss on a regression task.}
    \label{fig:regression_loss}
\end{figure*}

We compare the loss of QMIX (with an ELU non-linearity), QMIX-Lin, and QMIX-2Lin in Figure \ref{fig:regression_loss}, which illustrates that QMIX and QMIX-2Lin perform similarly, and both perform a faster optimisation of the loss than QMIX-Lin.
Even though QMIX-Lin and QMIX-2Lin can both learn only linear mixing functions, their parametrisation of such a linear function are different.
Having an extra layer is a better parametrisation in this scenario because it allows for gradient descent to optimise the loss faster \citep{arora2018convergence}.

\begin{figure*}[h!]
    \centering
    \includegraphics[width=0.45\textwidth]{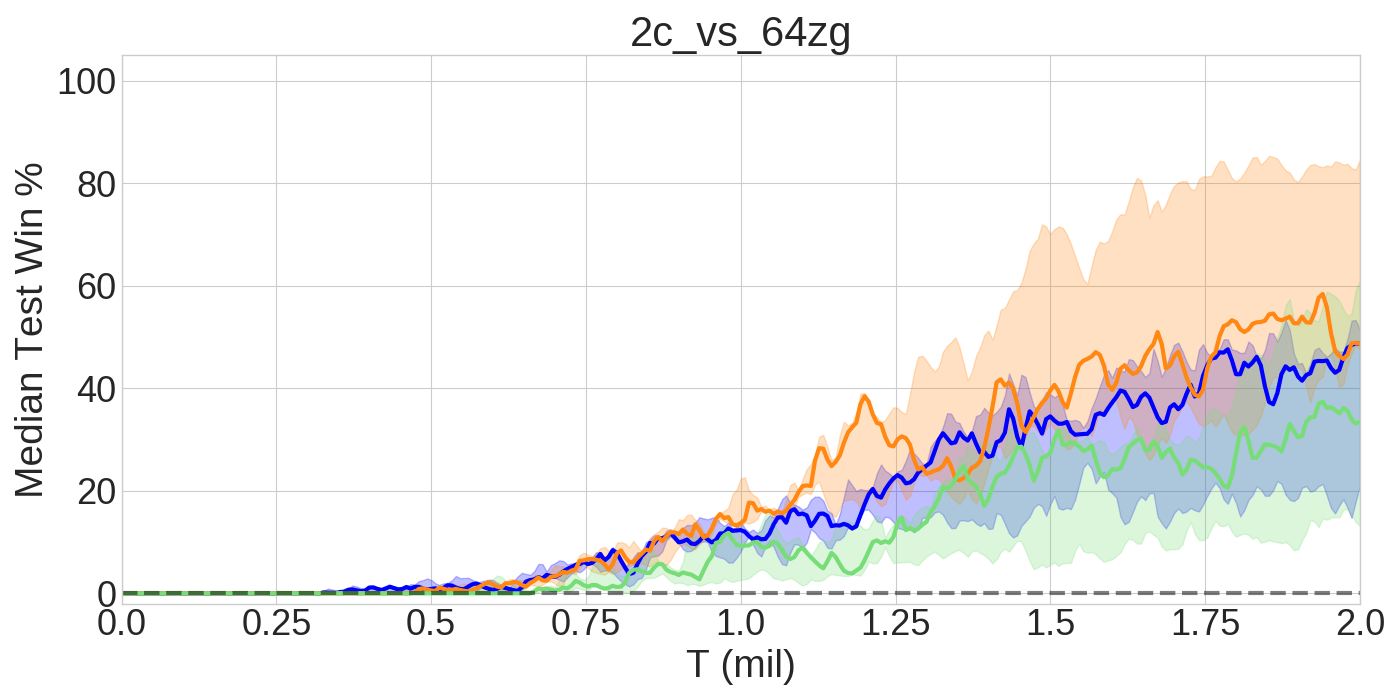}
    \includegraphics[width=0.45\textwidth]{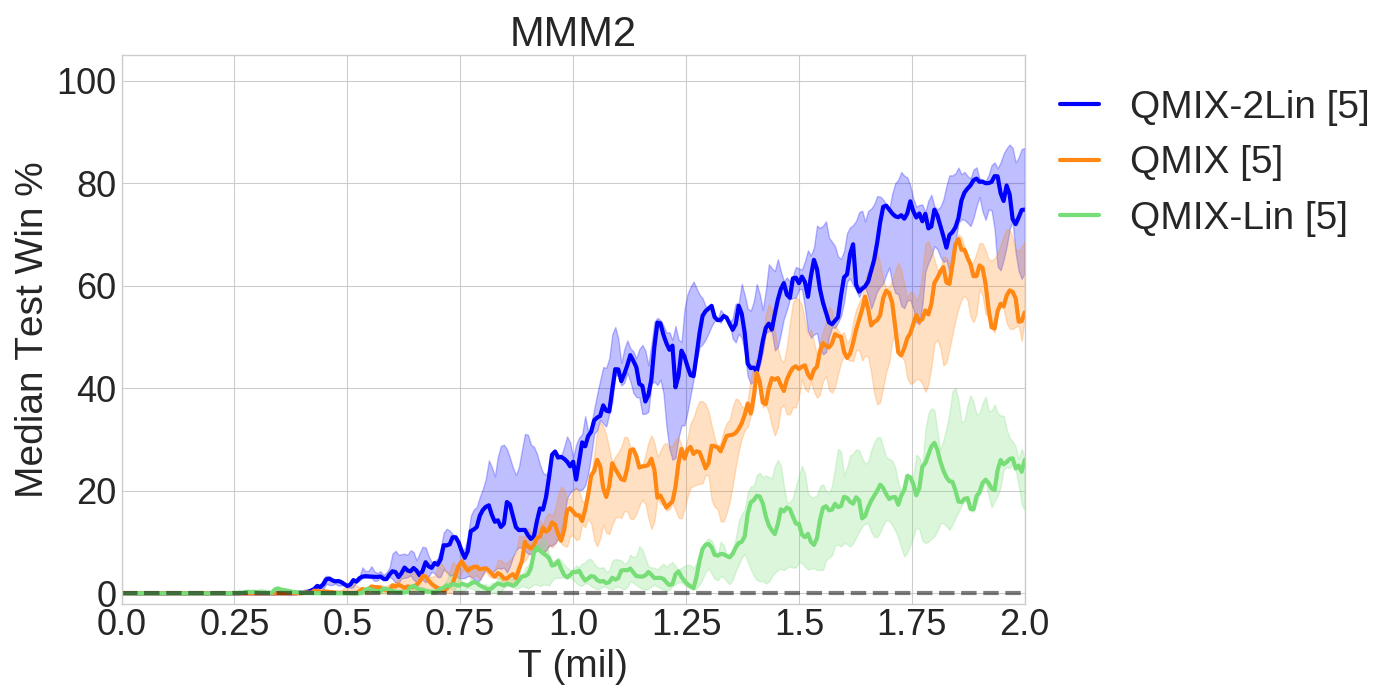}
    \caption{Ablations for linear experiments.}
    \label{fig:linear_ablations}
\end{figure*}

Figure \ref{fig:linear_ablations} confirms that the more flexible 
parameterisation of the mixing network is largely responsible for the increased 
performance. 
However, the results on \textit{2c\_vs\_64zg} show that QMIX performs slightly 
better than QMIX-2Lin, indicating that in some scenarios nonlinear mixing can 
still be beneficial. 
Figure \ref{fig:2c_mixing_elu} (right) shows an example of when such a nonlinearity is learnt on this task.

We also investigate changing the nonlinearity in the mixing network from ELU, which is largely linear, to Tanh.

\begin{figure*}[h!]
    \centering
    \includegraphics[width=0.45\textwidth]{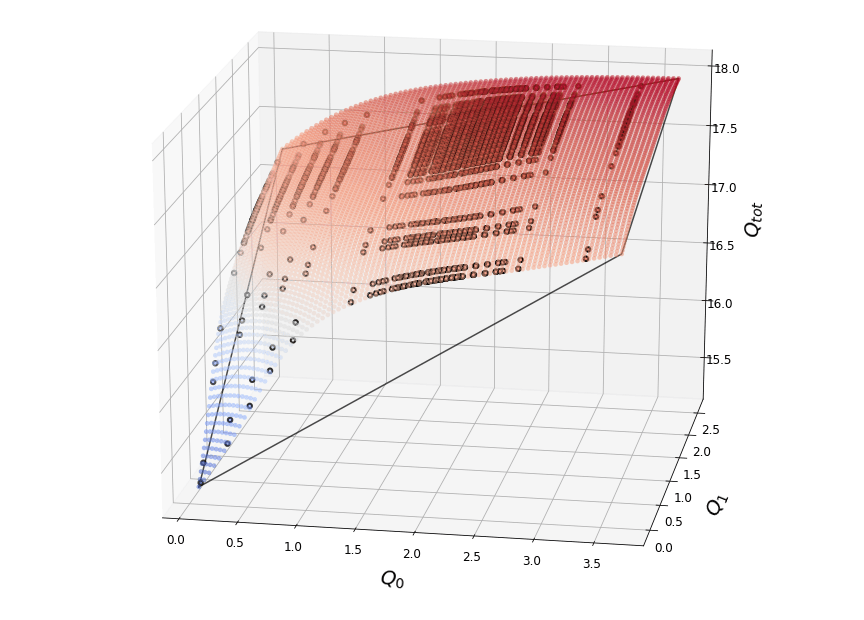}
    \includegraphics[width=0.45\textwidth]{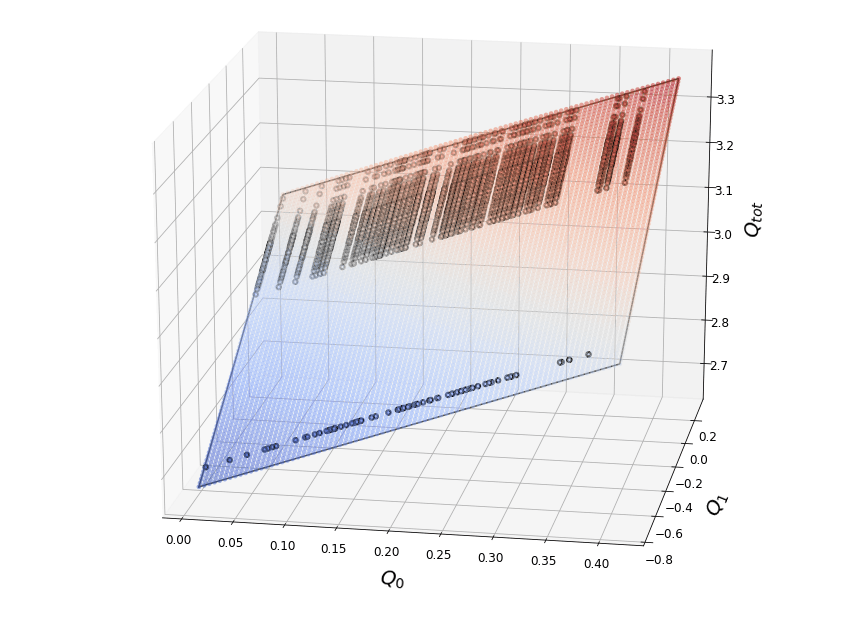}
    \caption{The learnt mixing of QMIX on \textit{2c\_vs\_64zg} using a tanh nonlinearity at the end of training for timesteps 0 (left) and 50 (right).}
    \label{fig:2c_mixing_tanh}
\end{figure*}

\begin{figure*}[h!]
    \centering
    \includegraphics[width=0.35\textwidth]{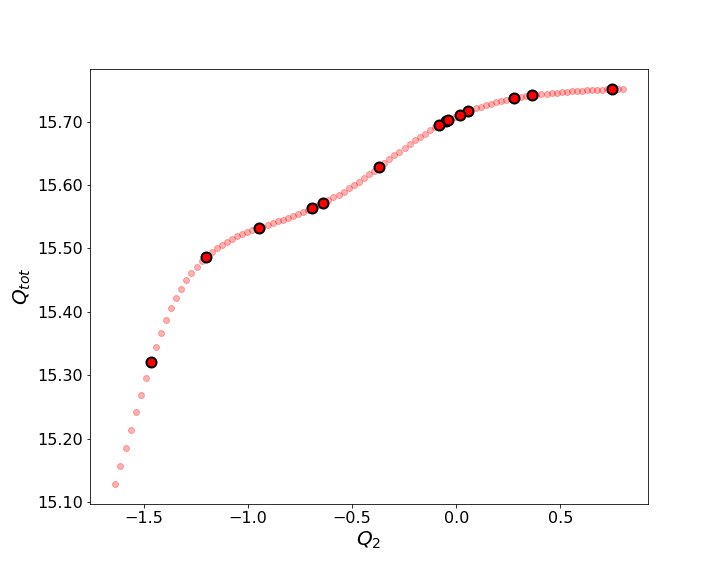}
    \includegraphics[width=0.35\textwidth]{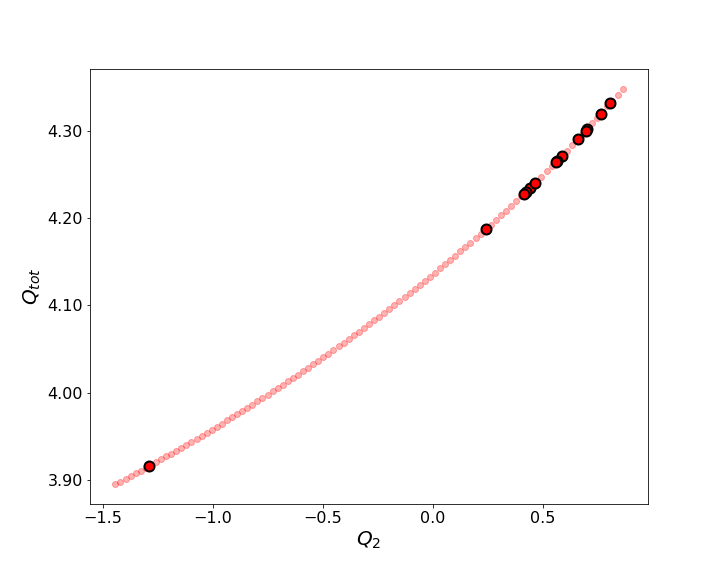}
    \includegraphics[width=0.35\textwidth]{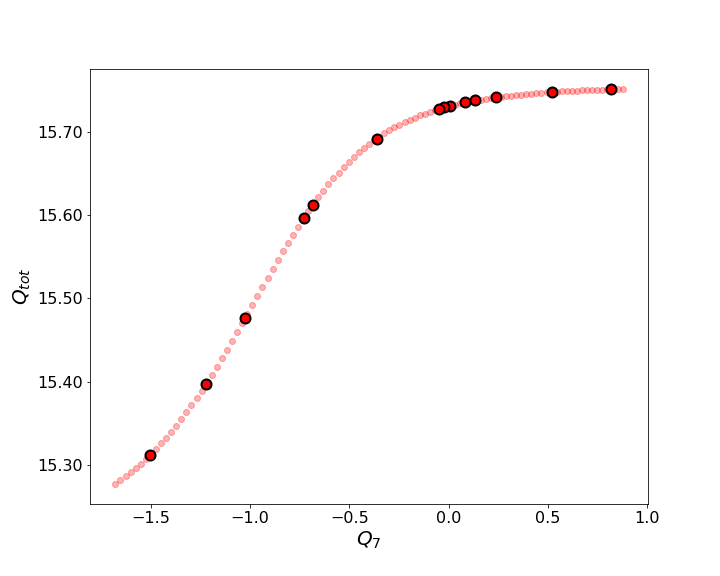}
    \includegraphics[width=0.35\textwidth]{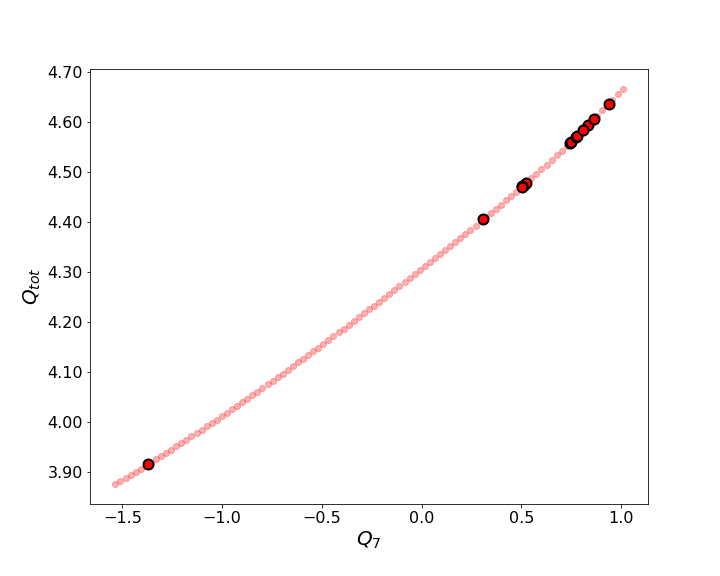}
    \caption{The learnt mixing of QMIX on \textit{3s5z}, using a tanh nonlinearity, at the end of training for timesteps 0 (left) and 50 (right), for agents 2 (top) and 7 (bottom).
    }
    \label{fig:3s5z_mixing_tanh}
\end{figure*}

\begin{figure*}[h!]
    \centering
    \includegraphics[width=0.4\textwidth]{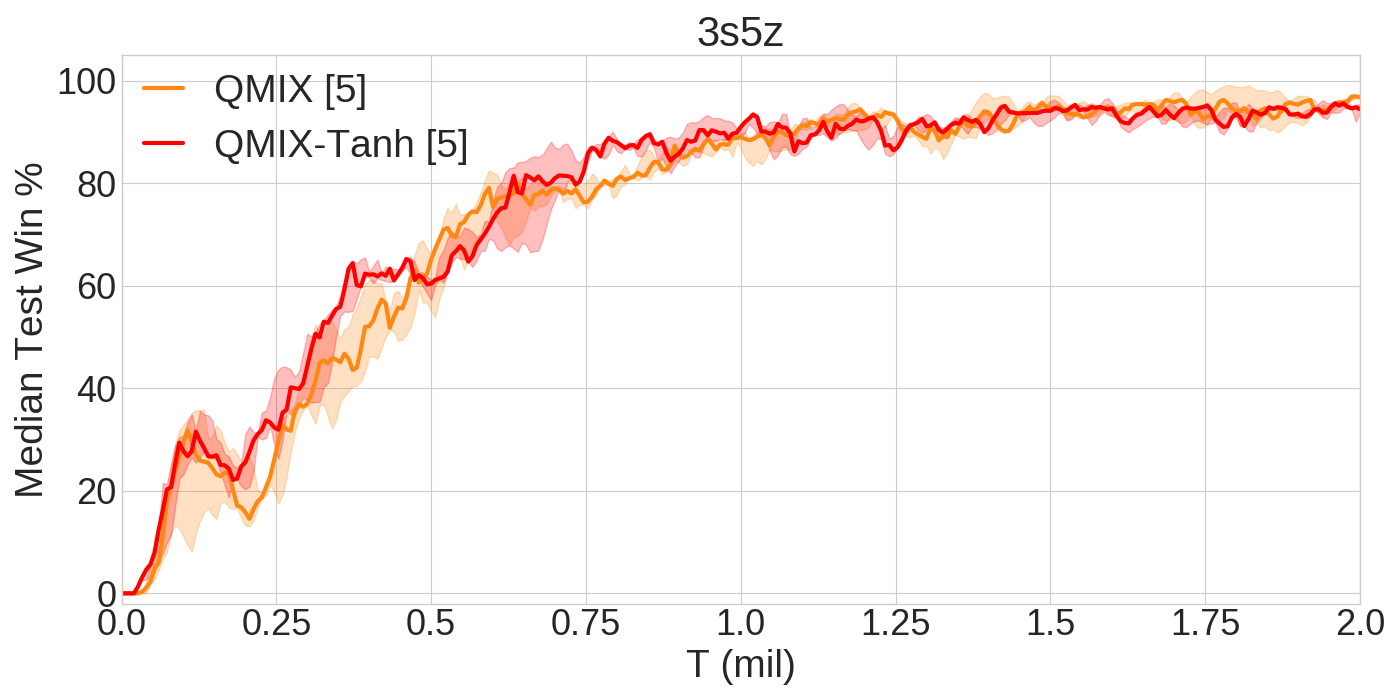}
    \includegraphics[width=0.4\textwidth]{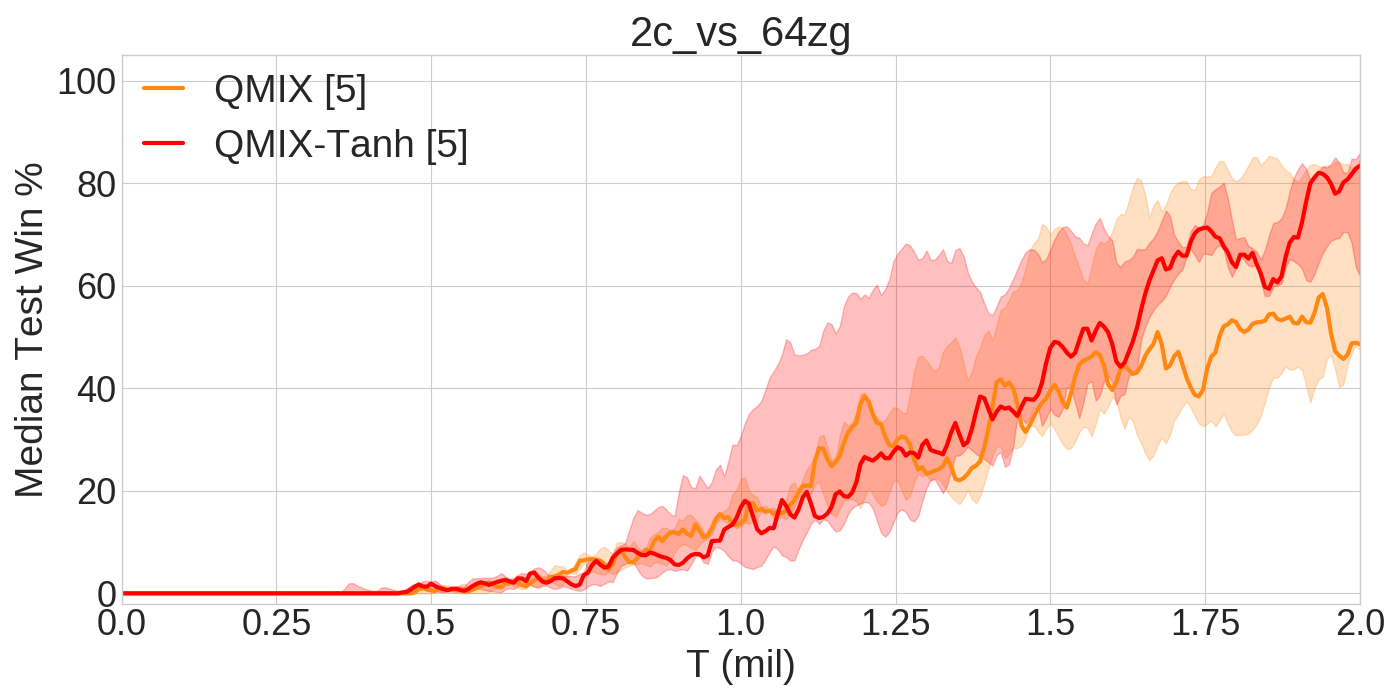}
    \caption{The Median test win \% comparing QMIX with an ELU nonlinearity (QMIX) and QMIX with a tanh nonlinearity in the mixing network (QMIX-Tanh).}
    \label{fig:elu_vs_tanh}
\end{figure*}

Figures \ref{fig:2c_mixing_tanh} and \ref{fig:3s5z_mixing_tanh} show the learnt mixings when using a tanh non-linearity.
The mixing network using tanh learns a more nonlinear mixing, especially for 
\textit{3s5z}, than with ELU.
Figure \ref{fig:elu_vs_tanh} shows how this translates to performance: there 
are perhaps some small gains in learning speed for some random seeds when using 
tanh.
However, the effect of an additional layer (in improving the optimisation 
dynamics) is much more significant than encouraging a more non-linear mixing 
through the choice of nonlinearity.

\subsection{Role of RNN in Agent Network}

\begin{figure*}[h!]
    \centering
    \includegraphics[width=0.45\textwidth]{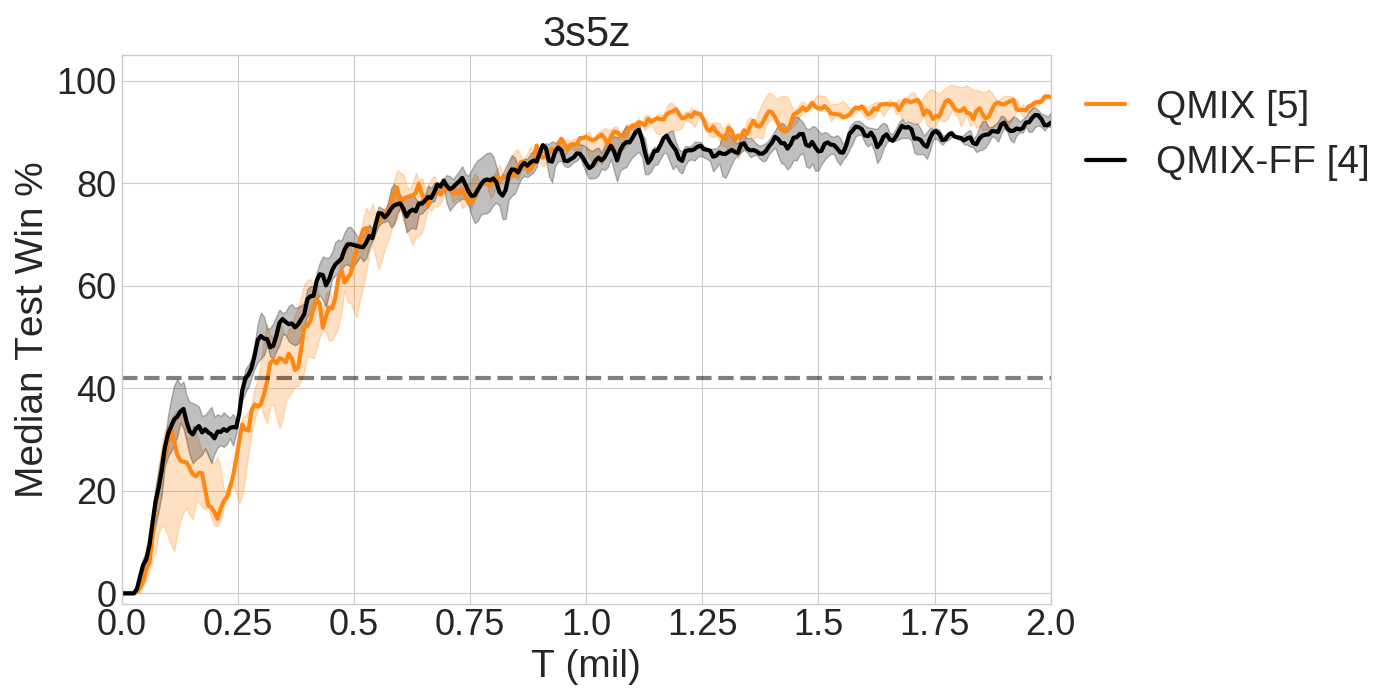}
    \includegraphics[width=0.45\textwidth]{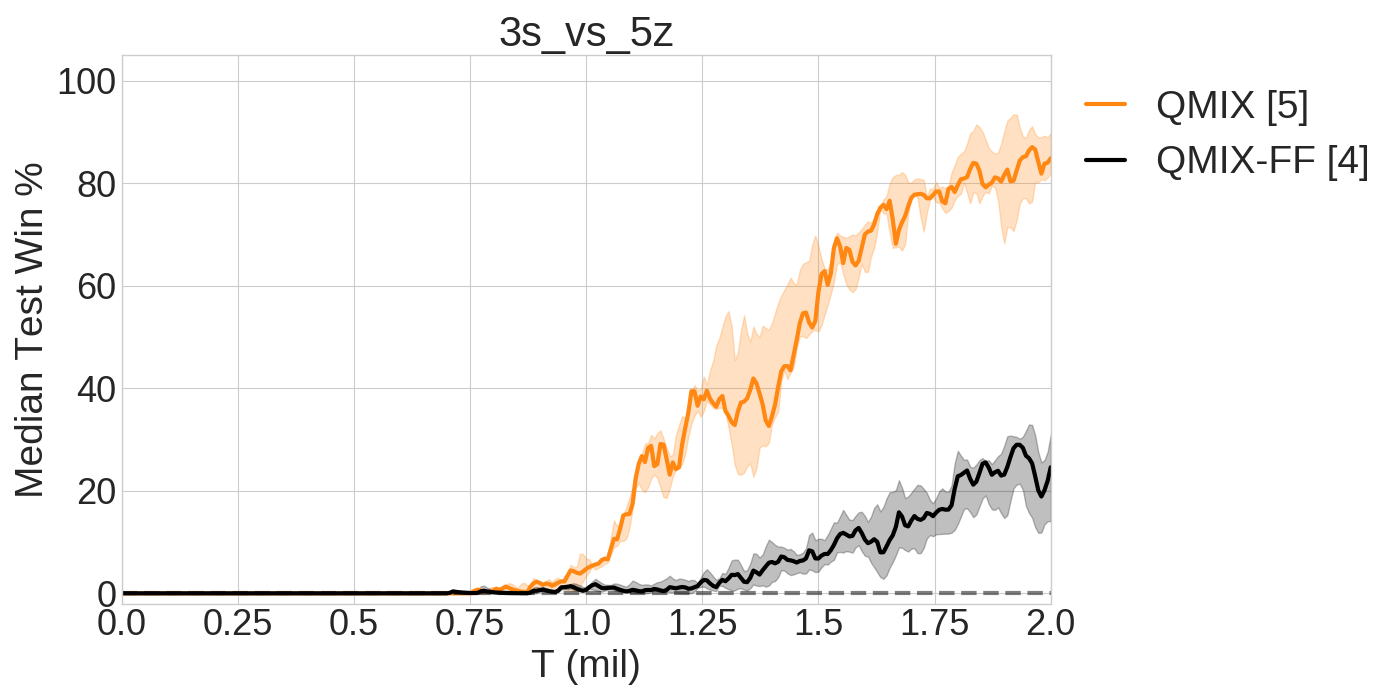}
    \caption{Comparing agent networks with and without RNNs (QMIX-FF) on two scenarios.}
    \label{fig:ff_vs_rnn}
\end{figure*}

Finally, we compare the necessity of using an agent's action-observation history by comparing the performance of an agent network with and without an RNN in Figure \ref{fig:ff_vs_rnn}.
On the easy scenario of \textit{3s5z} we can see that an RNN is not required to use action-observation information from previous timesteps, but on the harder scenario of \textit{3s\_vs\_5z} it is crucial to learning how to kite effectively. 

\subsection{Role of the COMA Critic Architecture}

In addition to the algorithmic differences between QMIX and COMA (QMIX is an off-policy $Q$-learning algorithm; COMA is an on-policy actor-critic algorithm), the architecture used to estimate the relevant $Q$-values differs.
COMA's critic is a feed-forward network, whereas QMIX's architecture for estimating $Q_{tot}$ contains RNNs (the agent networks).
The results of Figure \ref{fig:ff_vs_rnn} raise the question: To what extent do the architectural differences between QMIX and COMA affect final performance?
To test this, we change the architecture of the COMA critic to match that of QMIX more closely.
Specifically it now consists of per-agent components with the same architecture and inputs as the agent networks used by QMIX, which feed  $Q_a$ values for their chosen action into a mixing network. 
Since there is no requirement for monotonicity in this mixing network, we simplify its architecture significantly to a feed-forward network with ReLU non-linearities and two hidden layers of 128 units.
We test two variants of the mixing network, one with access to the state and one without.
The former takes the state as input in addition to the agent's chosen action $Q_a$-values.
The output of the mixing network is $|U|$ units, and we mask out the $Q_a$-value for the agent whose $Q$-values are being produced.

\begin{figure*}[h!]
    \centering
    \includegraphics[width=0.45\textwidth]{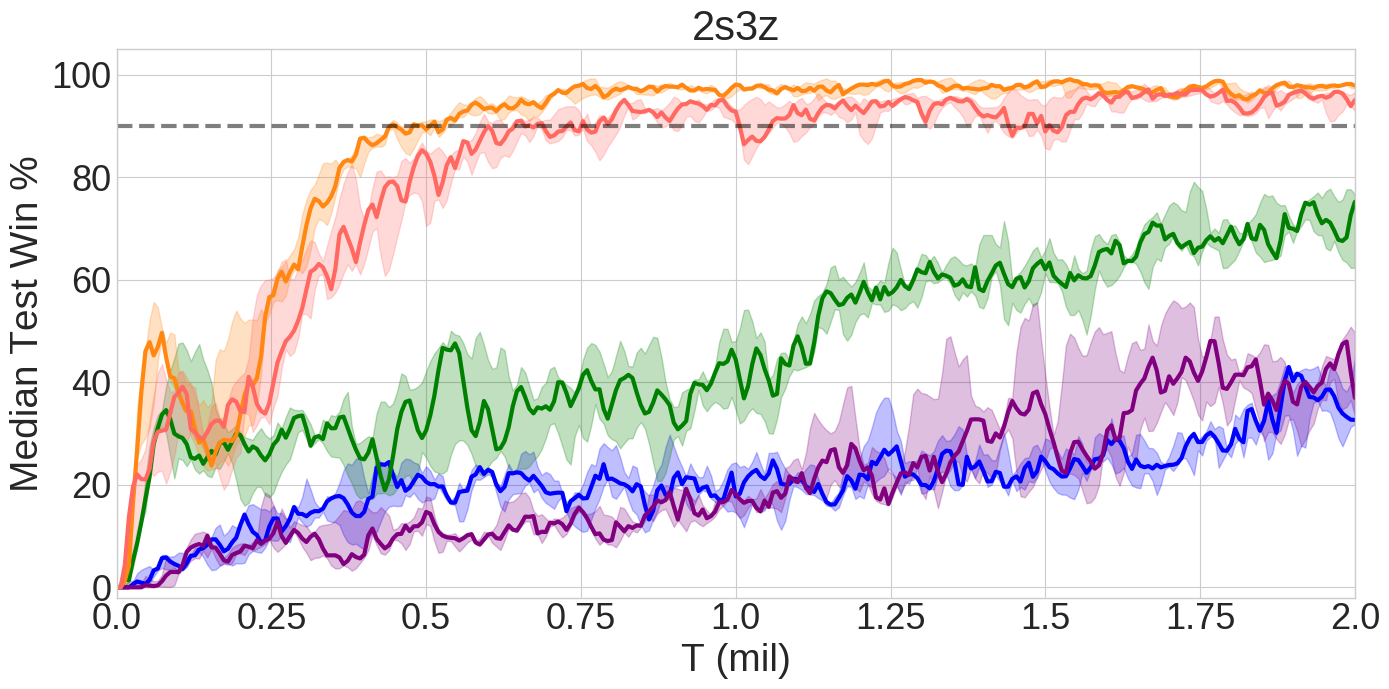}
    \includegraphics[width=0.45\textwidth]{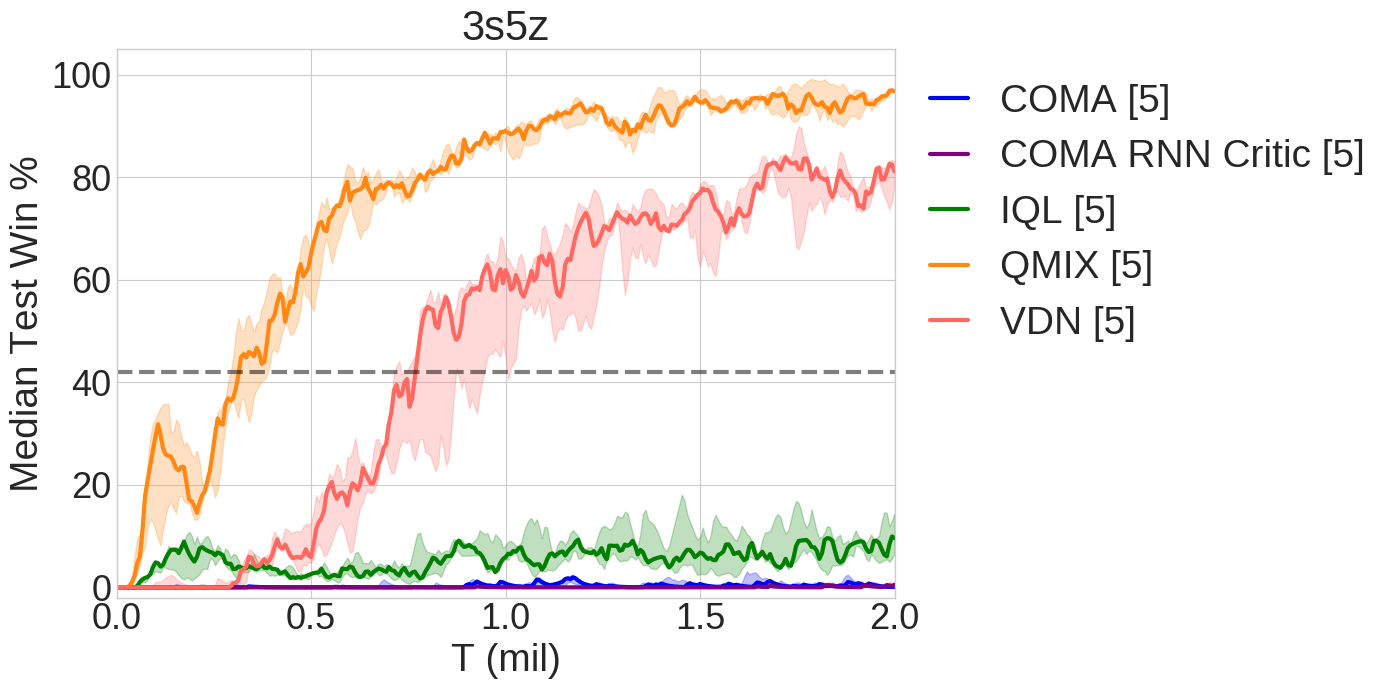}
    \caption{Comparing the different architectures for the COMA critic without the state, on two scenarios.}
    \label{fig:coma_rnn}
\end{figure*}

\begin{figure*}[h!]
    \centering
    \includegraphics[width=0.45\textwidth]{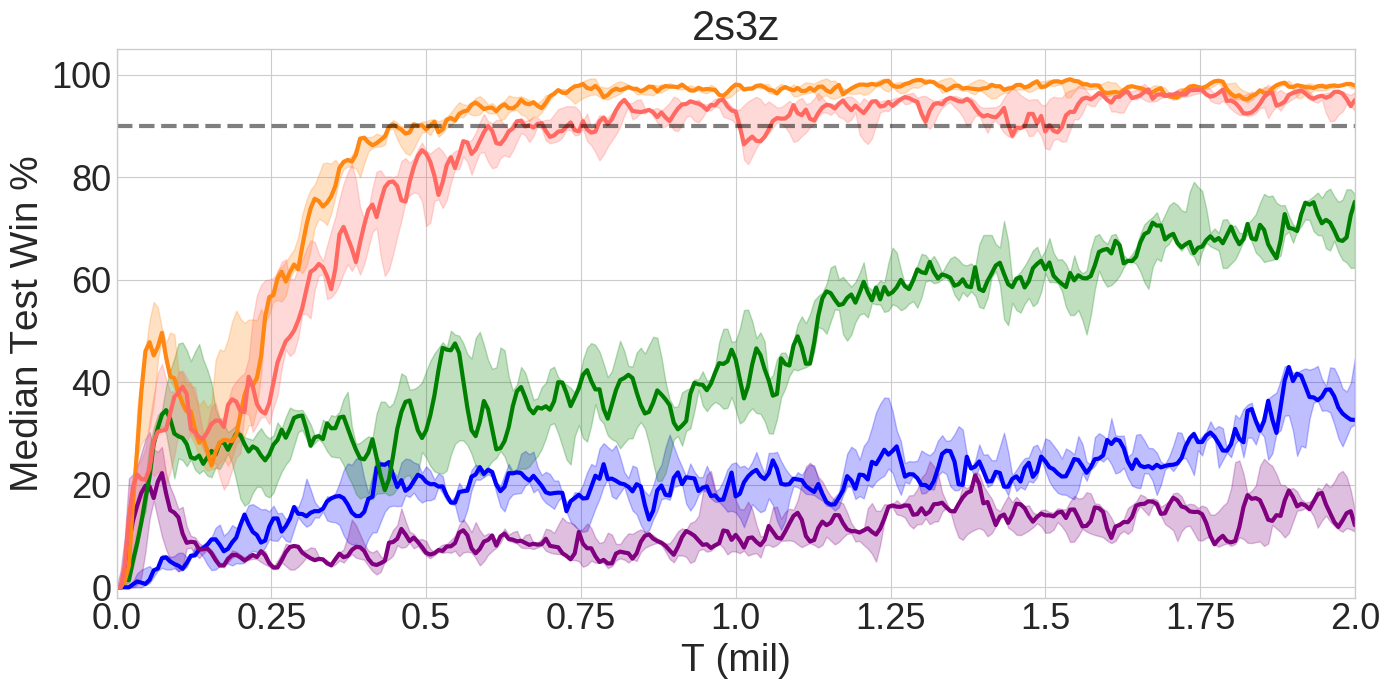}
    \includegraphics[width=0.45\textwidth]{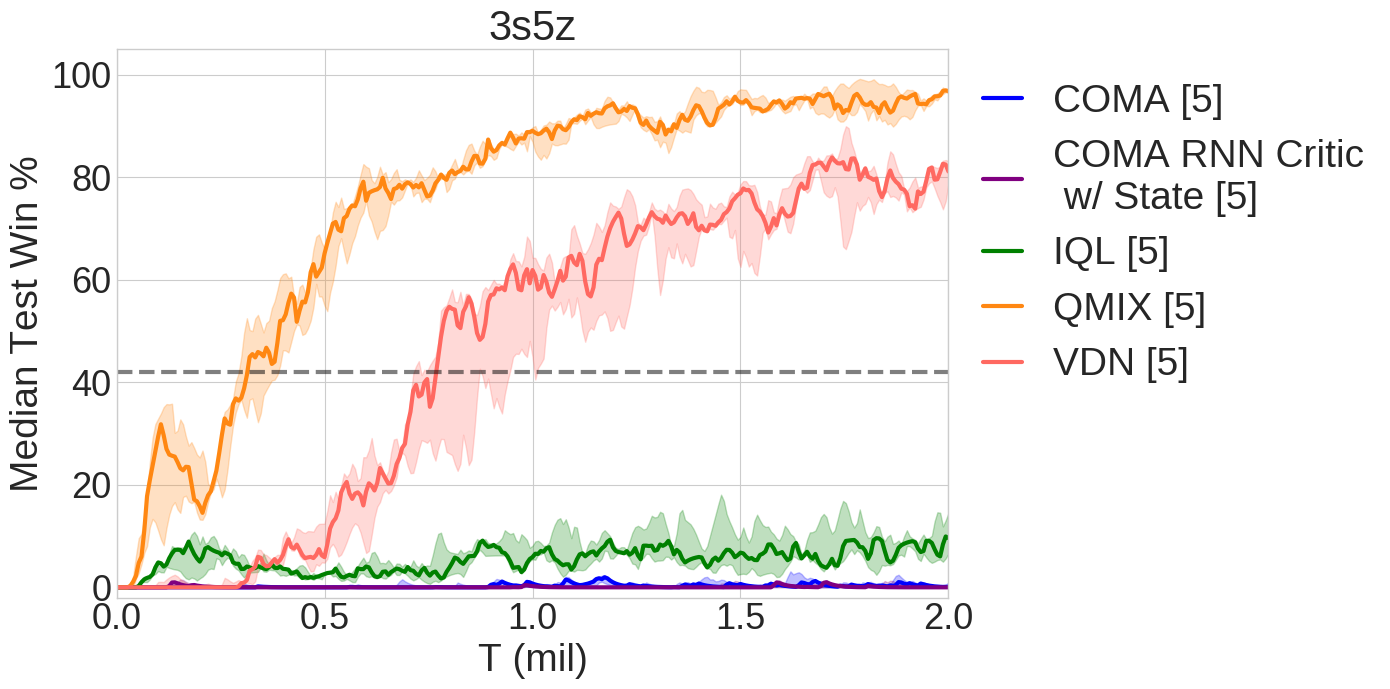}
    \caption{Comparing the different architectures for the COMA critic with the state on two scenarios.}
    \label{fig:coma_rnn_state}
\end{figure*}

Figures \ref{fig:coma_rnn} and \ref{fig:coma_rnn_state} show that the architectural change has little effect on the performance of COMA, with the addition of the state actually degrading performance in \textit{2s3z}. 

\subsection{Role of Network Size and Depth}

Another architectural consideration is the overall size of the network used to estimate $Q_{tot}$ for VDN and VDN-S.
Since QMIX has an additional mixing network, the overall architecture is deeper 
during training.
In order to understand the effect of this depth on performance, we consider the following variants of VDN and VDN-S.
\begin{itemize}
    \item \textbf{VDN Wider:} We increase the size of the agent's hidden layer from 64 to 128. 
    \item \textbf{VDN Deeper:} We add an additional layer to the agent network after the GRU, with a ReLU non-linearity. 
    \item \textbf{VDN-S Bigger:} We increase the size of the state-dependent bias to 2 hidden layers of 64 units.
\end{itemize}

In our other experiments, we have used the same agent networks for each method, 
to ensure a fair comparison of the performance of the decentralised policies.
By increasing the capacity of the agent networks for VDN Wider and VDN Deeper, 
we are now using larger agent networks than all other methods.
In this sense, the comparisons to these variants are no longer fair.
Figure \ref{fig:vdn_agent} shows the results for VDN Wider and VDN Deeper.
As expected, increasing the capacity of the agent networks increases 
performance. 
In particular, increasing the depth results in a large performance increase on 
the challenging MMM2. 
However, despite the larger agent networks, VDN Wider and VDN Deeper still do 
not match the performance of QMIX with smaller, less expressive agent networks.
Figure \ref{fig:vdn_state_bigger} shows the results for VDN-S Bigger, in which 
we increase the size of the state-dependent bias.  The results show that this has a minimal impact on performance over VDN-S. 

\begin{figure*}[h!]
    \centering
    \includegraphics[width=0.45\textwidth]{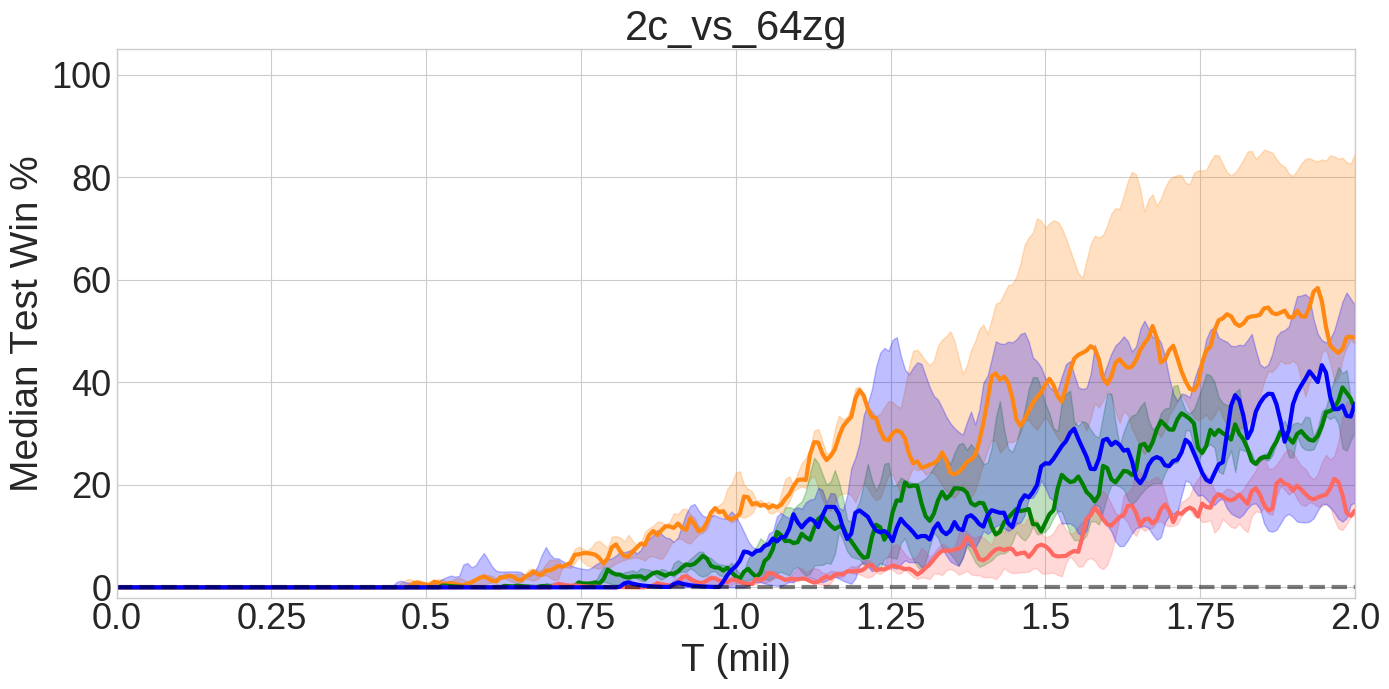}
    \includegraphics[width=0.45\textwidth]{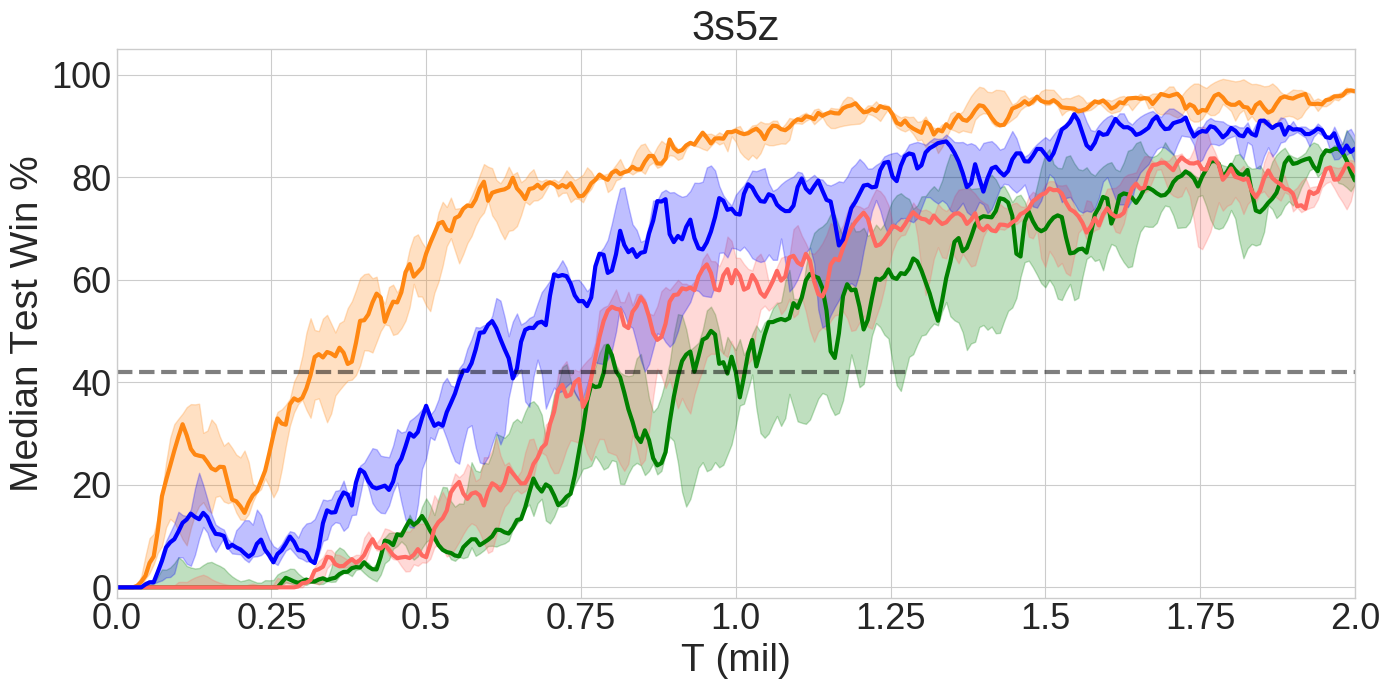}
    \includegraphics[width=0.45\textwidth]{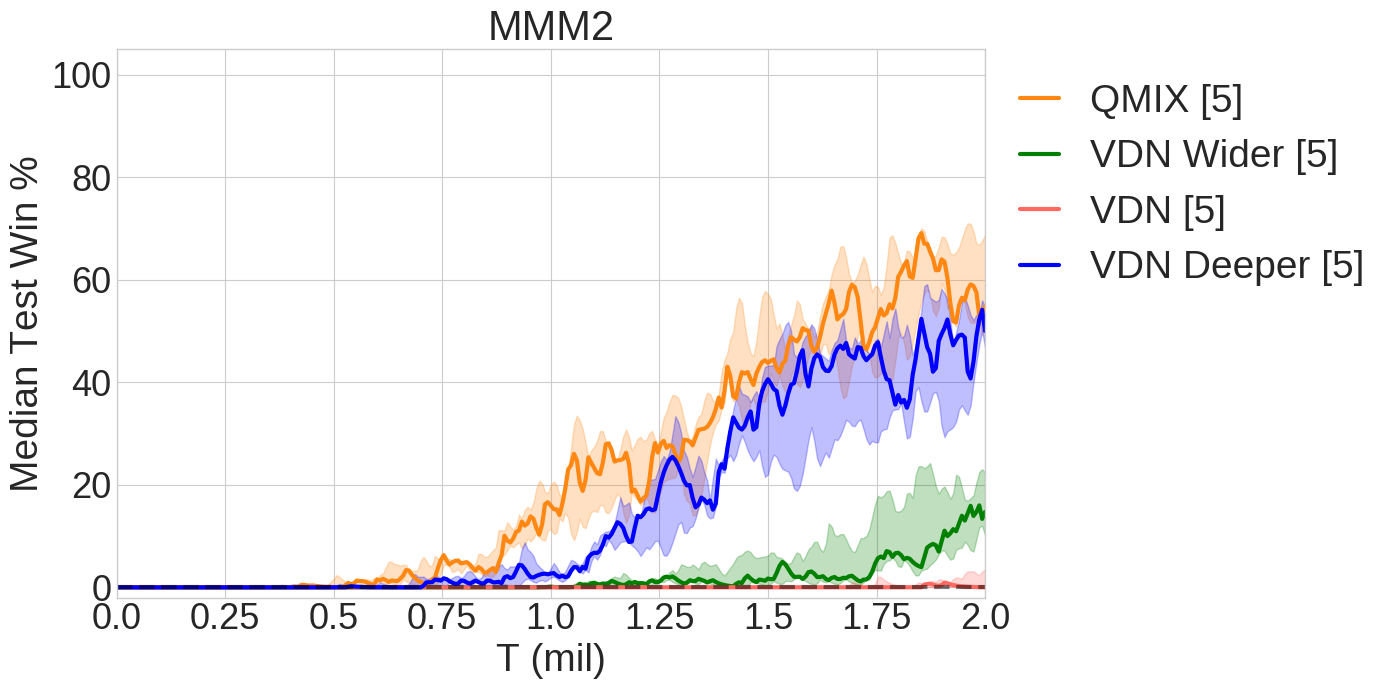}
    \caption{Comparing larger agent networks for VDN.}
    \label{fig:vdn_agent}
\end{figure*}

\begin{figure*}[h!]
    \centering
    \includegraphics[width=0.45\textwidth]{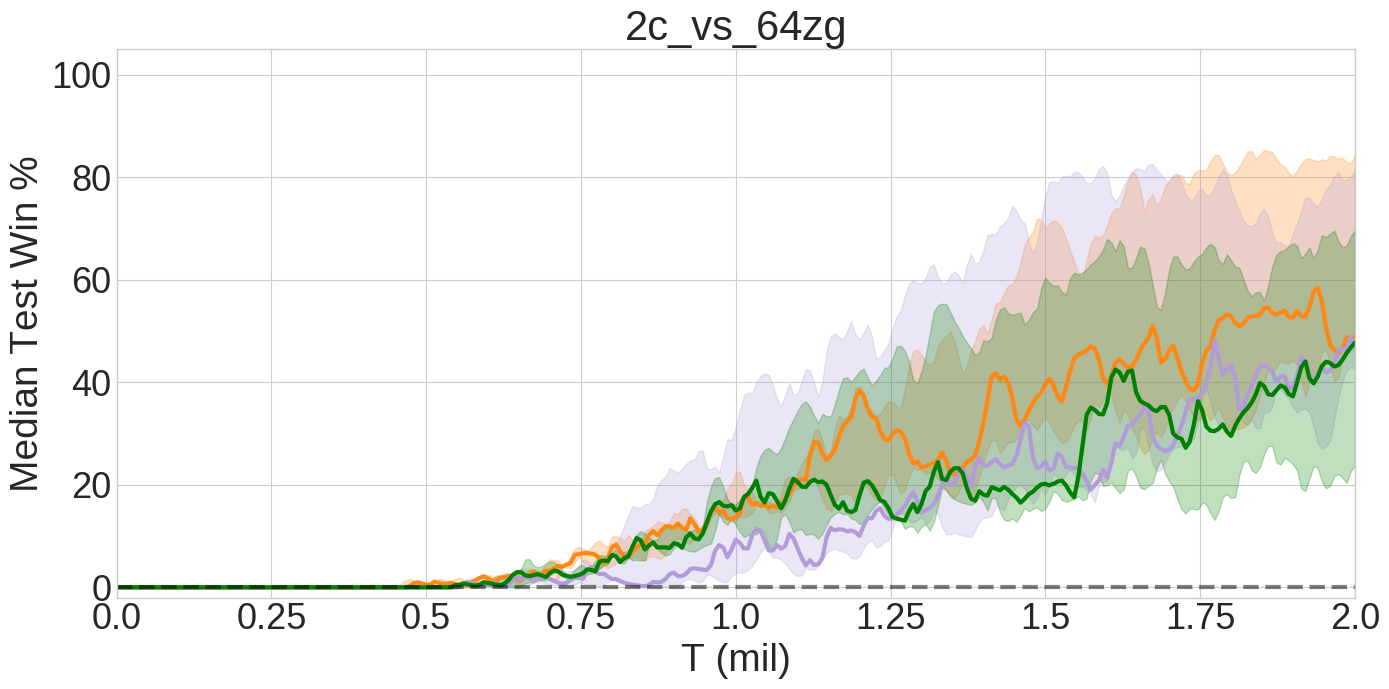}
    \includegraphics[width=0.45\textwidth]{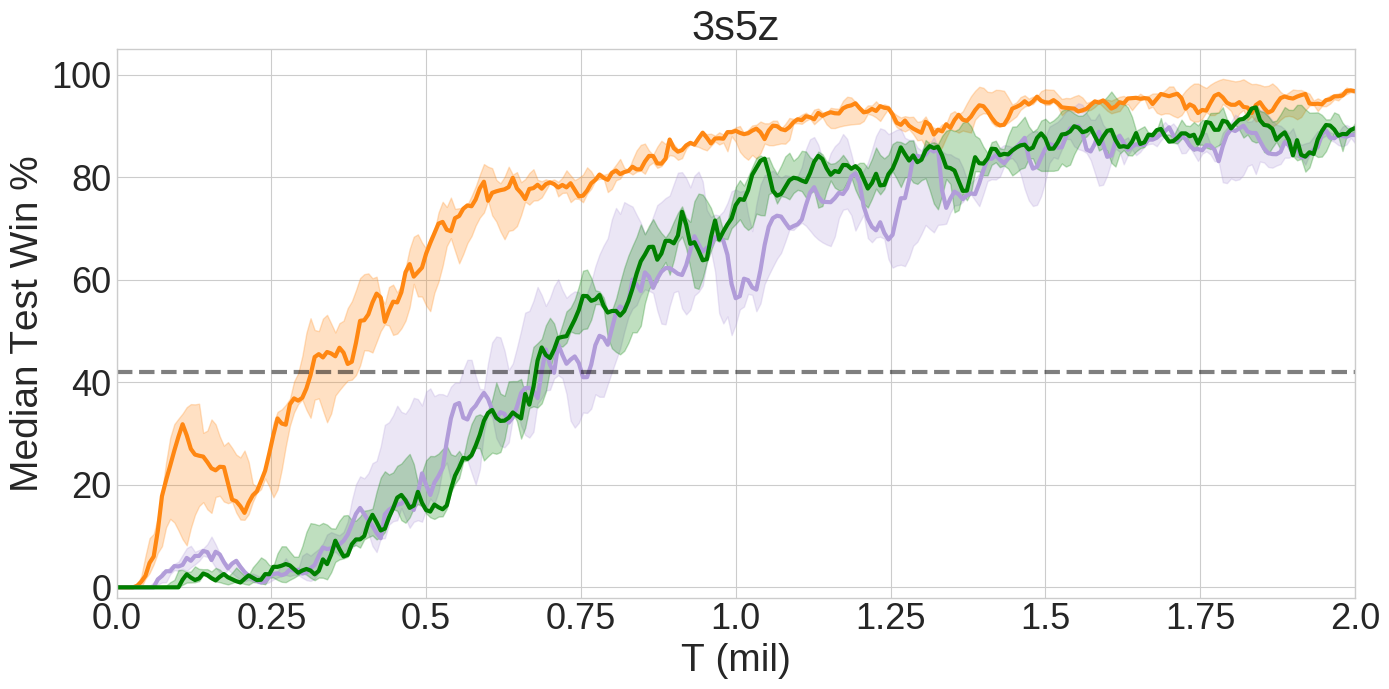}
    \includegraphics[width=0.45\textwidth]{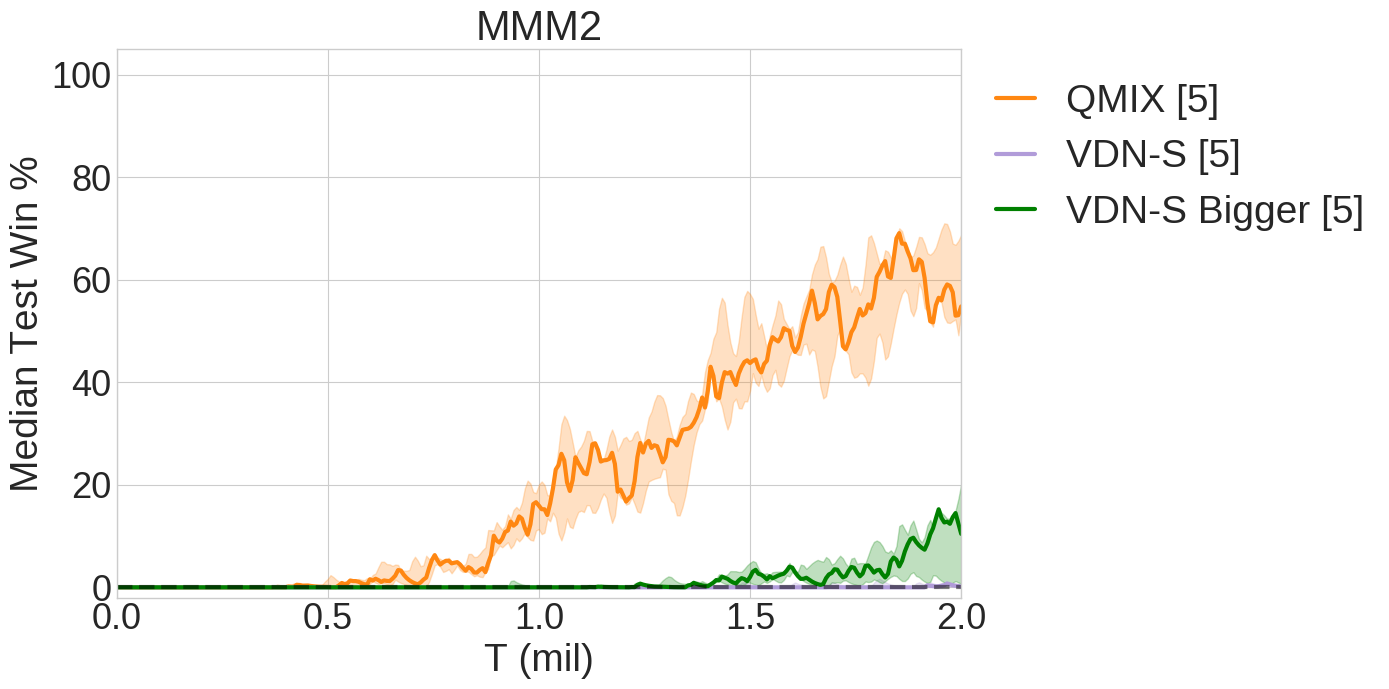}
    \caption{Comparing a larger state-dependent bias for VDN-S.}
    \label{fig:vdn_state_bigger}
\end{figure*}

\section{Conclusion}

This article presented QMIX, a deep multi-agent RL method that allows end-to-end learning of decentralised policies in a centralised setting and makes efficient use of extra state information. QMIX allows the learning of a rich joint action-value function, which admits tractable decompositions into per-agent action-value functions. This is achieved by imposing a monotonicity constraint on the mixing network.

To evaluate QMIX and other multi-agent RL algorithms, we introduced SMAC as a set of benchmark problems for cooperative multi-agent RL.
SMAC focuses on decentralised
micromanagement tasks and features 14 diverse combat scenarios which challenge
MARL methods to handle partial observability and high-dimensional inputs. Additionally, we
open-source PyMARL, a framework for cooperative deep MARL algorithms that features implementations of IQL, VDN, QMIX, QTRAN and COMA.

Our results on SMAC show that QMIX improves the final performance over other relevant deep multi-agent RL algorithms, most notably VDN. 
We include a detailed analysis through a series of ablations and visualisations for why QMIX outperforms VDN, showing that two components of QMIX are crucial to the performance the algorithm: a state-dependent bias and a learnt mixing of the agent's utilities into $Q_{tot}$ that is conditioned on the state.
Through our analysis, we found that the learnt mixing need only be linear to recover the majority of QMIX's performance, but that the parametrisation for this linear function is an important consideration.

In the future, we aim to extend SMAC with new challenging scenarios that
feature a more diverse set of units and require a higher level of coordination
amongst agents. Particularly, we plan to make use of the rich skill set of
StarCraft II units, and host scenarios that require the agents to utilise the
features of the terrain. With harder multi-agent coordination problems, we aim to explore the
gaps in existing multi-agent RL approaches and motivate further research in this domain.

There are several interesting avenues for future research:
\begin{itemize}
	\item Investigating the performance gap between value-based methods and policy-based methods. 
	In our experiments, QMIX (and VDN) significantly outperform COMA in all scenarios, both in terms of sample efficiency and final performance. The primary differences between the methods are the use of $Q$-learning vs the Policy Gradient theorem, use of off-policy data from the replay buffer and the architecture used to estimate the joint-action $Q$-Values.
	\item Expanding upon the joint-action $Q$-values representable by QMIX, to allow for the representation of coordinated behaviour. Importantly, this must scale gracefully to the deep RL setting.
	\item Better algorithms for multi-agent exploration that go beyond $\epsilon$-greedy. There have been many advances with regards to exploration in single-agent RL, dealing with both sparse and dense rewards environments. Taking inspiration from these, or developing entirely new approaches that improve exploration in a cooperative multi-agent RL setting is an important problem. SMAC's \textit{super-hard} scenarios provide a nice benchmark for evaluating these advances. 
\end{itemize}

\acks{
This project has received funding from the European Research Council (ERC) under the European Union's Horizon 2020 research and innovation programme (grant agreement number 637713). 
It was also supported by the Oxford-Google DeepMind Graduate Scholarship, the UK EPSRC CDT in Autonomous Intelligent Machines and Systems, Chevening Scholarship, Luys Scholarship and an EPSRC grant (EP/M508111/1, EP/N509711/1). This work is linked to and partly funded by the project Free the Drones (FreeD) under the Innovation Fund Denmark and Microsoft. The experiments were made possible by a generous equipment grant from NVIDIA.

We would like to thank Frans Oliehoek and Wendelin Boehmer for helpful comments and discussion. 
We also thank Oriol Vinyals, Kevin Calderone, and the rest of the SC2LE team at DeepMind and Blizzard for their work on the interface.

}

\bibliography{final}
\newpage

\appendix

\section{Monotonicity}
\label{appendix:monotonicity}
	\label{proof:mon}
	\begin{theorem} 
		If $\forall a \in A \equiv\{1, 2, ..., n\}$, $\frac{\partial Q_{tot}}{\partial Q_a}  \geq 0$ then
		\begin{equation}
		\nonumber
		\argmax_{\mathbf{u}}Q_{tot}(\boldsymbol{\tau}, \mathbf{u}) = 
		\begin{pmatrix}
		\argmax_{u^1}Q_1(\tau^1, u^1)   \\
		\vdots \\
		\argmax_{u^n}Q_n(\tau^n, u^n) \\
		\end{pmatrix}.
		\end{equation}
	\end{theorem}
	\begin{proof}
	Since $\frac{\partial Q_{tot}}{\partial Q_a}\geq 0$ for $\forall a\in A$, the following holds for any $(u^1, \dots, u^n)$ and the \textit{mixing network} function $Q_{tot}(\cdot)$ with $n$ arguments:
	\begin{align}
	\nonumber
	\begin{split}
	&Q_{tot}(Q_1(\tau^1, {u}^1), \dots, Q_a(\tau^a, {u}^a), \dots, Q_n(\tau^n, {u}^n)) \\
	&\leq Q_{tot}(\max_{u^1}Q_1(\tau^1, u^1), \dots, Q_a(\tau^a, u^a), \dots, Q_n(\tau^n, {u}^n))\\
	&\dots\\
	&\leq Q_{tot}(\max_{u^1}Q_1(\tau^1, u^1), \dots, \max_{u^a}Q_a(\tau^a, u^a), \dots, Q_n(\tau^n, {u}^n))\\
	&\dots\\
	&\leq Q_{tot}(\max_{u^1}Q_1(\tau^1, u^1), \dots, \max_{u^a}Q_a(\tau^a, u^a), \dots, \max_{u^n}Q_n(\tau^n, {u}^n)). \\
	\end{split}
	\end{align}
	Therefore, 
	the maximiser of the mixing network function is: $$(\max_{u^1}Q_1(\tau^1, u^1), \dots, \max_{u^a}Q_a(\tau^a, u^a), \dots, \max_{u^n}Q_n(\tau^n, {u}^n)).$$
	Thus,
	\begin{align}
	\nonumber
	\begin{split}
	\max_{\mathbf{u}}Q_{tot}(\boldsymbol{\tau}, \mathbf{u}) &\coloneqq \max_{\mathbf{u} = (u^1, \dots, u^n)} Q_{tot}(Q_1(\tau^1, {u}^1), \dots, Q_n(\tau^n, {u}^n))\\
	&= Q_{tot}(\max_{u^1}Q_1(\tau^1, u^1), \dots, \max_{u^n}Q_n(\tau^n, {u}^n)).\\
	\end{split}
	\end{align}
	Letting $\mathbf{u}_* = (u_*^1,\dots,u_*^n) = \begin{pmatrix}
		\argmax_{u^1}Q_1(\tau^1, u^1)   \\
		\vdots \\
		\argmax_{u^n}Q_n(\tau^n, u^n) \\
		\end{pmatrix}$,
	we have that:
	\begin{align}
	\nonumber
	\begin{split}
	Q_{tot}(Q_1(\tau^1, {u}_*^1), \dots, Q_n(\tau^n, {u}_*^n))
	&= Q_{tot}(\max_{u^1}Q_1(\tau^1, u^1), \dots, \max_{u^n}Q_n(\tau^n, {u}^n))\\
	&= \max_{\mathbf{u}}Q_{tot}(\boldsymbol{\tau}, \mathbf{u})
	\end{split}
	\end{align}
	Hence, $\mathbf{u}_* =  \argmax_{\mathbf{u}}Q_{tot}(\boldsymbol{\tau}, \mathbf{u})$, 
	which proves \eqref{eq:argmax_constist}.
	\end{proof}

\section{Two Step Game}
\label{sec:2step_arch}

\subsection{Architecture and Training}

The architecture of all agent networks is a DQN with a single hidden layer comprised of $64$ units with a ReLU nonlinearity. Each agent performs independent $\epsilon$ greedy action selection, with $\epsilon=1$. We set $\gamma=0.99$. The replay buffer consists of the last 500 episodes, from which we uniformly sample a batch of size 32 for training. The target network is updated every 100 episodes. The learning rate for RMSprop is set to $5 \times 10^{-4}$. We train for $10k$ timesteps. The size of the mixing network is $8$ units. All agent networks share parameters, thus the agent id is concatenated onto each agent's observations. We do not pass the last action taken to the agent as input. Each agent receives the full state as input.

Each state is one-hot encoded. The starting state for the first timestep is State $1$. If Agent $1$ takes Action A, it transitions to State $2$ (whose payoff matrix is all $7$s). If agent $1$ takes Action B in the first timestep, it transitions to State $3$. 

\subsection{Learned Value Functions}

The learned value functions for the different methods on the Two Step Game are shown in Tables \ref{qmix_2step_game_all} and \ref{tab:iql_qvals}. 

\begin{table}[h]
    \setlength{\extrarowheight}{3pt}{}
    \centering
    (a)
    \begin{tabular}{c|*{2}{>{\centering\arraybackslash}p{.05\linewidth}|}}
        \multicolumn{1}{c}{} & \multicolumn{2}{c}{State $1$} \\
        \multicolumn{1}{c}{} & \multicolumn{1}{c}{\bb{$A$}}  & \multicolumn{1}{c}{\bb{$B$}} \\ \cline{2-3}
        \cc{$A$} & 6.94 & 6.94 \\ \cline{2-3}
        \cc{$B$} & 6.35 & 6.36  \\\cline{2-3}
    \end{tabular}~
    \begin{tabular}{|*{2}{>{\centering\arraybackslash}p{.05\linewidth}|}}
        \multicolumn{2}{c}{State $2$A} \\
        \multicolumn{1}{c}{\bb{$A$}}  & \multicolumn{1}{c}{\bb{$B$}} \\ \cline{1-2}
        6.99 & 7.02 \\\cline{1-2}
        6.99 & 7.02  \\\cline{1-2}
    \end{tabular}~
    \begin{tabular}{|*{2}{>{\centering\arraybackslash}p{.05\linewidth}|}}
        \multicolumn{2}{c}{State $2$B} \\
        \multicolumn{1}{c}{\bb{$A$}}  & \multicolumn{1}{c}{\bb{$B$}} \\\cline{1-2}
        \text{-1.87} & 2.31 \\\cline{1-2}
        2.33 & 6.51  \\\cline{1-2}
    \end{tabular}\\\bigskip

    (b)
    \begin{tabular}{c|*{2}{>{\centering\arraybackslash}p{.05\linewidth}|}}
        \multicolumn{1}{c}{} & \multicolumn{1}{c}{\bb{$A$}}  & \multicolumn{1}{c}{\bb{$B$}} \\ \cline{2-3}
        \cc{$A$} & 6.93 & 6.93  \\ \cline{2-3}
        \cc{$B$} & 7.92 & 7.92  \\\cline{2-3}
    \end{tabular}~
    \begin{tabular}{|*{2}{>{\centering\arraybackslash}p{.05\linewidth}|}}
        \multicolumn{1}{c}{\bb{$A$}}  & \multicolumn{1}{c}{\bb{$B$}} \\ \cline{1-2}
        7.00 & 7.00 \\ \cline{1-2}
        7.00 & 7.00  \\\cline{1-2}
    \end{tabular}~
    \begin{tabular}{|*{2}{>{\centering\arraybackslash}p{.05\linewidth}|}}
        \multicolumn{1}{c}{\bb{$A$}}  & \multicolumn{1}{c}{\bb{$B$}} \\\cline{1-2}
        0.00 & 1.00 \\\cline{1-2}
        1.00 & 8.00 \\\cline{1-2}
    \end{tabular}\\\bigskip

    (c)
    \begin{tabular}{c|*{2}{>{\centering\arraybackslash}p{.05\linewidth}|}}
        \multicolumn{1}{c}{} & \multicolumn{1}{c}{\bb{$A$}}  & \multicolumn{1}{c}{\bb{$B$}} \\ \cline{2-3}
        \cc{$A$} & 6.94 & 6.93  \\ \cline{2-3}
        \cc{$B$} & 7.93 & 7.92  \\\cline{2-3}
    \end{tabular}~
    \begin{tabular}{|*{2}{>{\centering\arraybackslash}p{.05\linewidth}|}}
        \multicolumn{1}{c}{\bb{$A$}}  & \multicolumn{1}{c}{\bb{$B$}} \\ \cline{1-2}
        7.03 & 7.02 \\ \cline{1-2}
        7.02 & 7.01  \\\cline{1-2}
    \end{tabular}~
    \begin{tabular}{|*{2}{>{\centering\arraybackslash}p{.05\linewidth}|}}
        \multicolumn{1}{c}{\bb{$A$}}  & \multicolumn{1}{c}{\bb{$B$}} \\\cline{1-2}
        0.00 & 1.01 \\\cline{1-2}
        1.01 & 8.02 \\\cline{1-2}
    \end{tabular}\\\bigskip

    (d)
    \begin{tabular}{c|*{2}{>{\centering\arraybackslash}p{.05\linewidth}|}}
        \multicolumn{1}{c}{} & \multicolumn{1}{c}{\bb{$A$}}  & \multicolumn{1}{c}{\bb{$B$}} \\ \cline{2-3}
        \cc{$A$} & 6.98 & 6.97  \\ \cline{2-3}
        \cc{$B$} & 6.37 & 6.36  \\\cline{2-3}
    \end{tabular}~
    \begin{tabular}{|*{2}{>{\centering\arraybackslash}p{.05\linewidth}|}}
        \multicolumn{1}{c}{\bb{$A$}}  & \multicolumn{1}{c}{\bb{$B$}} \\ \cline{1-2}
        7.01 & 7.02 \\ \cline{1-2}
        7.02 & 7.04  \\\cline{1-2}
    \end{tabular}~
    \begin{tabular}{|*{2}{>{\centering\arraybackslash}p{.05\linewidth}|}}
        \multicolumn{1}{c}{\bb{$A$}}  & \multicolumn{1}{c}{\bb{$B$}} \\\cline{1-2}
        \text{-1.39} & 2.57 \\\cline{1-2}
        2.67 & 6.58 \\\cline{1-2}
    \end{tabular}\\\bigskip

    (e)
    \begin{tabular}{c|*{2}{>{\centering\arraybackslash}p{.05\linewidth}|}}
        \multicolumn{1}{c}{} & \multicolumn{1}{c}{\bb{$A$}}  & \multicolumn{1}{c}{\bb{$B$}} \\ \cline{2-3}
        \cc{$A$} & 6.95 & 6.99  \\ \cline{2-3}
        \cc{$B$} & 6.18 & 6.22  \\\cline{2-3}
    \end{tabular}~
    \begin{tabular}{|*{2}{>{\centering\arraybackslash}p{.05\linewidth}|}}
        \multicolumn{1}{c}{\bb{$A$}}  & \multicolumn{1}{c}{\bb{$B$}} \\ \cline{1-2}
        6.99 & 7.06 \\ \cline{1-2}
        7.01 & 7.09  \\\cline{1-2}
    \end{tabular}~
    \begin{tabular}{|*{2}{>{\centering\arraybackslash}p{.05\linewidth}|}}
        \multicolumn{1}{c}{\bb{$A$}}  & \multicolumn{1}{c}{\bb{$B$}} \\\cline{1-2}
        \text{-1.21} & 2.73 \\\cline{1-2}
        2.46 & 6.40 \\\cline{1-2}
    \end{tabular}\\\bigskip

    \caption{$Q_{tot}$ on the 2 step game for (a) VDN, (b) QMIX, (c) QMIX-NS, (d) QMIX-Lin and (e) VDN-S}
    \label{qmix_2step_game_all}
\end{table}

\begin{table}[h]
    \setlength{\extrarowheight}{3pt}
    \centering

    \begin{tabular}{c|*{2}{>{\centering\arraybackslash}p{.05\linewidth}|}}
        \multicolumn{1}{c}{} & \multicolumn{1}{c}{$A$}  & \multicolumn{1}{c}{$B$} \\ \cline{2-3}
        Agent $1$ & 6.96 & 4.47  \\ \cline{2-3}
    \end{tabular}~
    \begin{tabular}{|*{2}{>{\centering\arraybackslash}p{.05\linewidth}|}}
        \multicolumn{1}{c}{$A$}  & \multicolumn{1}{c}{$B$} \\ \cline{1-2}
        6.98 & 7.00 \\ \cline{1-2}
    \end{tabular}~
    \begin{tabular}{|*{2}{>{\centering\arraybackslash}p{.05\linewidth}|}}
        \multicolumn{1}{c}{$A$}  & \multicolumn{1}{c}{$B$} \\\cline{1-2}
        0.50 & 4.50 \\\cline{1-2}
    \end{tabular}\\\bigskip

    \begin{tabular}{c|*{2}{>{\centering\arraybackslash}p{.05\linewidth}|}}
        \multicolumn{1}{c}{} & \multicolumn{1}{c}{$A$}  & \multicolumn{1}{c}{$B$} \\ \cline{2-3}
        Agent $2$ & 5.70 & 5.78  \\\cline{2-3}
    \end{tabular}~
    \begin{tabular}{|*{2}{>{\centering\arraybackslash}p{.05\linewidth}|}}
        \multicolumn{1}{c}{$A$}  & \multicolumn{1}{c}{$B$} \\ \cline{1-2}
        7.00 & 7.02  \\\cline{1-2}
    \end{tabular}~
    \begin{tabular}{|*{2}{>{\centering\arraybackslash}p{.05\linewidth}|}}
        \multicolumn{1}{c}{$A$}  & \multicolumn{1}{c}{$B$} \\\cline{1-2}
        0.50 & 4.47 \\\cline{1-2}
    \end{tabular}

    \caption{$Q_{a}$ for IQL on the 2 step game}
    \label{tab:iql_qvals}
\end{table}

\subsection{Results}

Figure \ref{fig:two_step_loss} shows the loss for the different methods. Table \ref{table:final_test_two_step} shows the final testing reward for each method.

\begin{figure*}[htb!]
    \centering
    \includegraphics[width=0.39\textwidth]{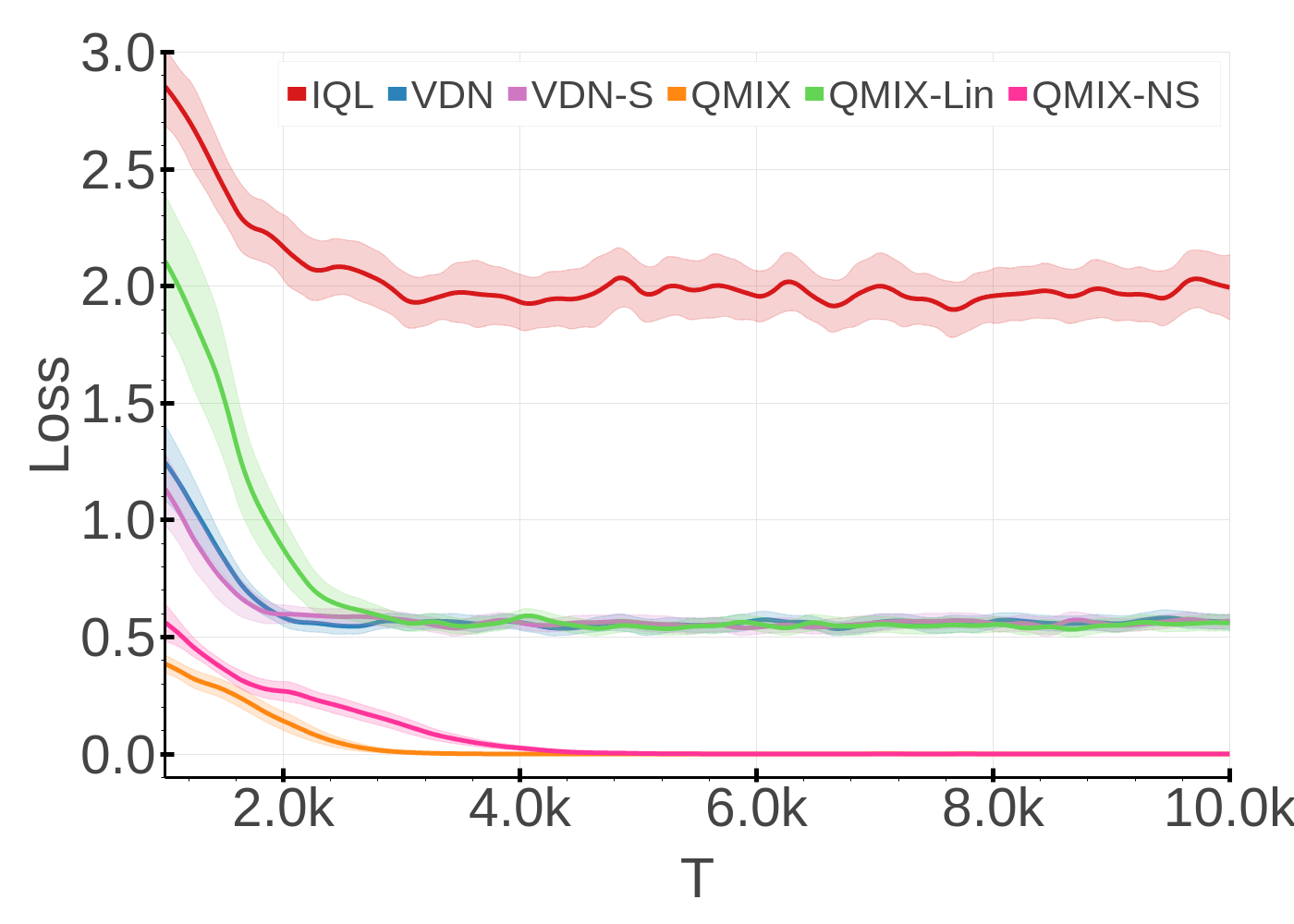}
    \caption{Loss for all six methods on the Two Step Game. The mean and 95\% confidence interval is shown across 30 independent runs.}
    \label{fig:two_step_loss}
\end{figure*}

\begin{table}[h]
	\setlength{\extrarowheight}{3pt}
	\centering
    \begin{center}
        \begin{tabular}{|c| c | c | c | c | c |}
        \hline
        \textbf{IQL} & \textbf{VDN} & \textbf{VDN-S} & \textbf{QMIX} & \textbf{QMIX-Lin} & \textbf{QMIX-NS} \\
        \hline
        7 & 7 & 7 & 8 & 7 & 8 \\
        \hline
        \end{tabular}
    \end{center}
    \caption{The final Test Reward acheived.}
    \label{table:final_test_two_step}
\end{table}

\section{Random Matrix Games} 
\label{sec:rnd_matrix_setup}

\subsection{Architecture and Training}

The architecture of all agent networks is a DQN with 2 hidden layers comprised of $64$ units with a ReLU nonlinearity. Each agent performs independent $\epsilon$ greedy action selection, with $\epsilon=1$. 
The replay buffer consists of the last 500 episodes, from which we uniformly sample a batch of size 32 for training. The target network is updated every 100 episodes. The learning rate for RMSprop is set to $5 \times 10^{-4}$. We train for $20k$ timesteps. 
The mixing network is identical to the SMAC experiments. 
All agent networks share parameters, thus the agent id is concatenated onto each agent's observations. We do not pass the last action taken to the agent as input. Each agent receives the full state (a vector of $1$s with dimension 5) as input.
\section{SMAC}\label{appendix:SMAC}

\subsection{Scenarios}\label{appendix:SMAC_scenarios}

Perhaps the simplest scenarios are \textbf{symmetric} battle scenarios, where the two armies are composed of the same units.
Such challenges are particularly interesting when some of the units are extremely effective against others (this is known as \textit{countering}), for example, by dealing bonus damage to a particular armour type.
In such a setting, allied agents must deduce this property of the game and design an intelligent strategy to protect teammates vulnerable to certain enemy attacks.

SMAC also includes more challenging scenarios, for example, in which the enemy army outnumbers the allied army by one or more units. In such \textbf{asymmetric} scenarios it is essential to consider the health of enemy units in order to effectively target the desired opponent.

Lastly, SMAC offers a set of interesting \textbf{micro-trick} challenges that require a higher-level of cooperation and a specific micro trick to defeat the enemy. 
An example of such scenario is the \texttt{corridor} scenario (Figure~\ref{fig:SC2maps_2}d). Here, six friendly Zealots face 24 enemy Zerglings, which requires agents to make effective use of the terrain features. Specifically, agents should collectively wall off the choke point (the narrow region of the map) to block enemy attacks from different directions. 
The \texttt{3s\_vs\_5z} scenario features three allied Stalkers against five enemy Zealots (Figure~\ref{fig:starcraft_screenshots}a). Since Zealots counter Stalkers, the only winning strategy for the allied units is to kite the enemy around the map and kill them one after another.
Some of the micro-trick challenges are inspired by \textit{StarCraft Master} challenge missions released by Blizzard \cite{SC2Master}.

\subsection{Environment Setting}\label{appendix:SMAC_evn_setting}

SMAC makes use of the StarCraft II Learning Environment (\textit{SC2LE}) \cite{vinyals_starcraft_2017} to communicate with the StarCraft II engine. SC2LE provides full control of the game by making it possible to send commands and receive observations from the game. However, SMAC is conceptually different from the RL environment of SC2LE. The goal of SC2LE is to learn to play the full game of StarCraft II. This is a competitive task where a centralised RL agent receives RGB pixels as input and performs both macro and micro with the player-level control similar to human players. SMAC, on the other hand, represents a set of cooperative multi-agent micro challenges where each learning agent controls a single military unit.

SMAC uses the \textit{raw API} of SC2LE. Raw API observations do not have any graphical component and include information about the units on the map such as health, location coordinates, etc. The raw API also allows sending action commands to individual units using their unit IDs. This setting differs from how humans play the actual game, but is convenient for designing decentralised multi-agent learning tasks.

Furthermore, to encourage agents to explore interesting micro strategies themselves, we limit the influence of the StarCraft AI on our agents. 
In the game of StarCraft II, whenever an idle unit is under attack, it automatically starts a reply attack towards the attacking enemy units without being explicitly ordered. We disable such automatic replies towards the enemy attacks or enemy units that are located closely by creating new units that are the exact copies of existing ones with two attributes modified: \textit{Combat: Default Acquire Level} is set to \textit{Passive} (default \textit{Offensive}) and \textit{Behaviour: Response} is set to \textit{No Response} (default \textit{Acquire}). These fields are only modified for allied units; enemy units are unchanged.

The sight and shooting range values might differ from the built-in \textit{sight} or \textit{range} attribute of some StarCraft II units. Our goal is not to master the original full StarCraft II game, but rather to benchmark MARL methods for decentralised control.

The game AI is set to level 7, \textit{very difficult}. Our experiments, however, suggest that this setting does significantly impact the unit micromanagement of the built-in heuristics.

\subsection{Architecture and Training}
\label{sec:smac_setup}

The architecture of all agent networks is a DRQN with a recurrent layer 
comprised of a GRU with a 64-dimensional hidden state, with a fully-connected 
layer before and after.
Exploration is performed during training using independent $\epsilon$-greedy action selection, where each agent $a$ performs $\epsilon$-greedy action selection over its own $Q_a$. 
Throughout the training, we anneal $\epsilon$ linearly from $1.0$ to $0.05$ over $50k$ time steps and keep it constant for the rest of the learning. 
We set $\gamma = 0.99$ for all experiments.
The replay buffer contains the most recent $5000$ episodes.  
We sample batches of 32 episodes uniformly from the replay buffer, and train on fully unrolled episodes, performing a single gradient descent step after every episode.
The target networks are updated after every $200$ training episodes.
The Double $Q$-Learning update rule from \citep{van2016deep} is used for all $Q$-Learning variants (IQL, VDN, QMIX and QTRAN).

To speed up the learning, we share the parameters of the agent networks across all agents. 
Because of this, a one-hot encoding of the \texttt{agent\_id} is concatenated onto each agent's observations. 
All neural networks are trained using RMSprop\footnote{We set $\alpha = 0.99$ and do not use weight decay or momentum.} with learning rate $5 \times 10^{-4}$. 

The mixing network consists of a single hidden layer of $32$ units, utilising an ELU non-linearity. 
The hypernetworks consist of a feedforward network with a single hidden layer of $64$ units with a ReLU non-linearity. 
The output of the hypernetwork is passed through an absolute function (to acheive non-negativity) and then resized into a matrix of appropriate size.

The architecture of the COMA critic is a feedforward fully-connected neural network with the first 2 layers having 128 units, followed by a final layer of $|U|$ units. We set $\lambda=0.8$. 
We utilise the same $\epsilon$-floor scheme as in \citep{foerster_counterfactual_2017} for the agents’ policies, linearlly annealing $\epsilon$ from 0.5 to 0.01 over 100k timesteps.
For COMA we roll-out 8 episodes and train on those episodes.
The critic is first updated, performing a gradient descent step for each timestep in the episode, starting with the final timestep.
Then the agent policies are updated by a single gradient descent step on the data from all 8 episodes.

The architecture of the centralised $Q$ for QTRAN is similar to the one used in \citep{son_qtran:_2019}.
The agent's hidden states ($64$ units) are concatenated with their chosen action ($|U|$ units) and passed through a feedforward network with a single hidden layer and a ReLU non-linearity to produce an agent-action embedding ($64 + |U|$ units). The network is shared across all agents. The embeddings are summed across all agents.
The concatentation of the state and the sum of the embeddings is then passed into the $Q$ network.
The $Q$ network consists of 2 hidden layers with ReLU non-linearities with $64$ units each.
The $V$ network takes the state as input and consists of 2 hidden layers with ReLU non-linearities with $64$ units each.
We set $\lambda_{opt}=1$ and $\lambda_{nopt\_min}=0.1$.

We also compared a COMA style architecture in which the input to $Q$ is the state and the joint-actions encoded as one-hot vectors.
For both architectural variants we also tested having 3 hidden layers.
For both network sizes and architectural variants we performed a hyperparameter search over $\lambda_{opt}=1$ and $\lambda_{nopt\_min} \in \{0.1,1,10\}$ on all three of the scenarios we tested QTRAN on and picked the best performer out of all configurations.

VDN-S uses a state-dependent bias that has the same architecture as the final bias in QMIX's mixing network, a single hidden layer of 64 units with a ReLU non-linearity.
VDN-S Bigger uses 2 hidden layers of 64 units instead.

\subsection{Table of Results}
\label{sec:table_results}

Table \ref{tab:hai} shows the final median performance (maximum median across the testing intervals within the last $250k$ of training) of the algorithms tested.
The mean test win \%, across 1000 episodes, for the heuristic-based ai is also shown.
\begin{table}[h]
    \setlength{\extrarowheight}{3pt}
    \centering
    \begin{center}
        \begin{tabular}{| c | c | c | c | c | c |}
        \hline
        ~&\textbf{IQL}&\textbf{COMA}&\textbf{VDN}&\textbf{QMIX}&\textbf{Heuristic}\\
        \hline \hline
\textbf{2s\_vs\_1sc} &
100&
 98&
100&
100&
0\\
\hline

\textbf{2s3z} &
 75&
 43&
 97&
 99&
90\\
\hline

\textbf{3s5z} &
 10&
  1&
 84&
 97&
42\\
\hline

\textbf{1c3s5z} &
 21&
 31&
 91&
 97&
81\\
\hline

\textbf{10m\_vs\_11m} &
 34&
  7&
 97&
 97&
12\\
\hline
\hline

\textbf{2c\_vs\_64zg} &
  7&
  0&
 21&
 58&
0\\
\hline

\textbf{bane\_vs\_bane} &
 99&
 64&
 94&
 85&
43\\
\hline

\textbf{5m\_vs\_6m} &
 49&
  1&
 70&
 70&
0\\
\hline

\textbf{3s\_vs\_5z} &
 45&
  0&
 91&
 87&
0\\
\hline
\hline

\textbf{3s5z\_vs\_3s6z} &
  0&
  0&
  2&
  2&
0\\
\hline

\textbf{6h\_vs\_8z} &
  0&
  0&
  0&
  3&
0\\
\hline

\textbf{27m\_vs\_30m} &
  0&
  0&
  0&
 49&
0\\
\hline

\textbf{MMM2} &
  0&
  0&
  1&
 69&
0\\
\hline

\textbf{corridor} &
  0&
  0&
  0&
  1&
0\\
\hline
        \end{tabular}
        \caption{The Test Win Rate \% of the various algorithms.}
    \end{center}
    \label{tab:hai}
\end{table}

\section{Regression Experiment}
\label{sec:regression_exp}

For the regression experiment we simulate a task with $4$ agents.
There are 10 states, of dimensionality 10, in which each dimension is uniformly sampled from $[-1,1]$.
For each of the 10 states, the agent $Q$-values are uniformly sampled from $[-1,1]$.
The $Q_{tot}$ targets for each state are uniformly sampled from $[0,10]$.
The agent $Q$-values, targets and states are kept fixed, only the parameters of the mixing network may change.

We use a mixing embed dimension of 32, and train for a total of 2000 timesteps.
At each timestep we uniformly sample a state, add it to the replay buffer of size 200, and execute a training step.
Each training step consists of sampling a minibatch of 32 from the replay buffer and minimising the $L_2$ loss between the $Q_{tot}$ targets and the network's outputs.

\end{document}